\documentclass[Afour,sageh,times]{sagej}

\usepackage{array}
\usepackage{textcomp}
\usepackage{stfloats}
\usepackage{url}
\usepackage{verbatim}
\usepackage{graphicx}
\usepackage[utf8]{inputenc}
\usepackage[english]{babel}
\usepackage{setspace}
\usepackage{multirow}
\usepackage[table,xcdraw]{xcolor}
\usepackage{epsfig}
\usepackage{bm}
\graphicspath{{images/}}
\usepackage[mathscr]{euscript}
\usepackage{enumitem}   
\usepackage{multicol}
\usepackage{multirow}
\usepackage{algorithm}
\usepackage{algorithmic}
\usepackage{amsthm} 
\usepackage{amsmath, amsfonts}
\usepackage{amssymb}
\usepackage[percent]{overpic}
\usepackage{siunitx}
\usepackage{soul}
\usepackage{sidecap}
\usepackage{caption}
\usepackage{subfigure}
\usepackage{float}
\usepackage{overpic}
\usepackage{booktabs}
\usepackage{comment}
\usepackage{mathtools}

\usepackage{amsmath}
\usepackage{amssymb}
\usepackage{xparse}

\newcommand{\vect}[1]{\boldsymbol{#1}}
\newcommand{\mat}[1]{\boldsymbol{#1}}

\newcommand{\diff}[2]{ \displaystyle \frac{\partial #1}{\partial #2}}

\newcommand{\diffs}[3]{\frac{\partial^2 #1}{
		\ifx#2#3 
		\partial #2^2
		\else
		\partial #2 \partial #3
		\fi
}}

\newcommand{\norm}[1]{{\left \| {#1} \right \|}}

\newcommand{\grad}[2]{\boldsymbol{\nabla}_{#2}{#1}}

\newcommand{\hess}[2]{\boldsymbol{\nabla}_{#2}^{2}{#1}}
\newcommand{\jac}[2]{\Jm_{#2}{#1}}

\newcommand{\R}{\mathbb{R}}

\newcommand\abs[1]{\left \vert#1\right \vert}

\newenvironment{carray}
{ \left( \begin{array}}
	{ \end{array} \right) }

\newcommand{\zerov}{\vect{0}}

\newcommand{\av}{\vect{a}}

\newcommand{\dv}{\vect{d}}

\newcommand{\fv}{\vect{f}}
\newcommand{\gv}{\vect{g}}

\newcommand{\hv}{\vect{h}}
\newcommand{\kv}{\vect{k}}
\newcommand{\lv}{\vect{l}}

\newcommand{\pv}{\vect{p}}

\newcommand{\qv}{{\vect{q}}}
\newcommand{\dqv}{\dot{\vect{q}}}
\newcommand{\ddqv}{\ddot{\vect{q}}}

\newcommand{\sv}{\vect{s}}

\newcommand{\uv}{\vect{u}}
\newcommand{\vv}{\vect{v}}

\newcommand{\xv}{\vect{x}}

\newcommand{\yv}{\vect{y}}

\newcommand{\zv}{\vect{z}}

\newcommand{\alphav}{\vect{\alpha}}

\newcommand{\tauv}{\vect{\tau}}

\newcommand{\thetav}{\vect{\theta}}

\newcommand{\dthetav}{\dot{\vect{\theta}}}

\newcommand{\ddthetav}{\ddot{\vect{\theta}}}

\newcommand{\tthetav}{\tilde{\thetav}}

\newcommand{\nuv}{\vect{\nu}}

\newcommand{\dV}{\dot{V}}

\newcommand{\Fv}{\vect{F}}

\newcommand{\IIm}{\mat{I}}

\newcommand{\Am}{\mat{A}}

\newcommand{\Cm}{\mat{C}}

\newcommand{\Dm}{\mat{D}}

\newcommand{\Jm}{\mat{J}}

\newcommand{\Km}{\mat{K}}

\newcommand{\Mm}{\mat{M}}

\newcommand{\Pm}{\mat{P}}
\newcommand{\Qm}{\mat{Q}}
\newcommand{\Rm}{\mat{R}}

\newcommand{\Tm}{\mat{T}}

\newcommand{\pvone}[1]{ \nabla_{#1}\, }
\newcommand{\pvtwo}[2]{ \nabla_{#2}\,#1 }
\newcommand{\parv}[2]{\expandafter\ifx\expandafter\relax
	\detokenize{#1}\relax\pvone{#2}\else\pvtwo{#1}{#2}\fi}

\NewDocumentCommand{\Rotm}{ O{} O{} O{} }{{}^{#2}\Rm_{#1}#3}
\NewDocumentCommand{\Transm}{ O{} O{} O{} }{{}^{#2}\Tm_{#1}#3}

\newcommand{\drm}{\mathrm{d}}

\newcommand{\rb}[1]{\left( #1 \right)}

\newenvironment{sequation*}
    {\begin{equation*}\small
    }
    { 
    \end{equation*}
    }

\newtheorem{property}{Property}
\newtheorem{theorem}{Theorem}
\newtheorem{assumption}{Assumption}
\newtheorem{problem}{Problem}
\newtheorem{lemma}{Lemma}
\newtheorem{definition}{Definition}
\newtheorem*{remark}{Remark}
\newtheorem{class}{Class}

\newtheorem*{example}{Example}
\newtheorem{corollary}{Corollary}

\newcommand{\setS}{\mathcal{S}(\beta, \alpha_{1}, \alpha_{2}, \alpha_{3})}

\newcommand\BibTeX{{\rmfamily B\kern-.05em \textsc{i\kern-.025em b}\kern-.08em
T\kern-.1667em\lower.7ex\hbox{E}\kern-.125emX}}

\def\volumeyear{2026}
\begin{document}


\runninghead{Pustina et al.}

\title{Collocated Shape Regulation for Soft Robots}



\author{
Pietro Pustina\affilnum{1,2},
Ebrahim Shahabi\affilnum{2,*},
Daniel Feliu-Talegon\affilnum{2,*},
Alessandro De Luca\affilnum{1} and
Cosimo Della Santina\affilnum{2,3}
}

\affiliation{
\affilnum{1}Department of Computer, Control and Management Engineering,
Sapienza University of Rome, Rome, Italy\\
\affilnum{2}Department of Cognitive Robotics,
Delft University of Technology, 2628 CN Delft, The Netherlands\\
\affilnum{3}Institute of Robotics and Mechatronics,
German Aerospace Center (DLR), 82234 Oberpfaffenhofen, Germany\\
\affilnum{*}These authors contributed equally to this work.
}

\corrauth{
Cosimo Della Santina,
Department of Cognitive Robotics,
Delft University of Technology,
2628 CN Delft, The Netherlands
}

\email{C.DellaSantina@tudelft.nl}



\begin{abstract}
Controlling the shape of a continuum soft robot typically requires an accurate dynamic model and actuation of all degrees of freedom. We show that regulating only the actuated coordinates—through collocated shape control—achieves provably stable convergence of those coordinates and, under an explicit compatibility condition, of the entire robot shape. While collocated control is a cornerstone of high-performance motion control in rigid robotics, extending this formulation to continuum soft robots has remained challenging due to the complexity of their dynamics. We present the first general framework for collocated control of continuum soft robots and derive a unified family of controllers, including PD, PID, P-satI-D, and their counterparts with compensation and cancellation components. The framework unifies existing approaches while introducing new controller designs. In particular, we develop three classes of PD- and PID-like regulators with local, semi-global, and global stability guarantees, and provide rigorous convergence analyses for each. Extensive experimental validation demonstrates the effectiveness of the proposed methods across different model discretizations and controller parameters. The resulting framework provides practical design guidelines for selecting and implementing controllers with known stability guarantees, without requiring a complete dynamic model of the robot.

\end{abstract}


\maketitle

\section{Introduction}
Continuum soft robots represent one of the most significant advancements in robotics over the past decade. Due to their deformability and adaptability, these systems have a broad range of possible applications, such as minimally invasive surgery~\cite{runciman2019soft}, exploration of hazardous or confined spaces~\cite{aracri2021soft}, and advanced human-robot interaction~\cite{jorgensen2022soft}. Despite their potential, the practical deployment of soft robotic systems in the real world faces several challenges that must be addressed. 

From a control perspective, soft robots demand novel control architectures that can exploit their compliance in a purposeful and aware manner. Inspiration for achieving this can be drawn from other complex robotic systems, such as humanoids and quadrupeds, where the execution of dynamic tasks relies on low-level configuration space controllers~\cite{albu2007dlr, bledt2018cheetah}. Similarly, in soft robotics, robust and reliable shape regulators can provide the foundation for achieving high-level behaviors. Traditionally, shape control has been addressed using learning-based methods, which remain widely used~\cite{kim2021review}. These methods relieve from deriving the complex dynamics that describe soft robots. However, they lack performance guarantees and fail to offer insights into system behavior. These limitations, together with the development of relatively accurate and computationally manageable reduced-order descriptions~\cite{boyer2020dynamics, sadati2021tmtdyn}, has sparked research into model-based techniques. A comprehensive review of these methods is provided in~\cite{dellasantina2023survey}. Initially, the possibility of employing model-based control laws for soft robots was demonstrated under the fully-actuated hypothesis~\cite{falkenhahn2015model, della2020model}, i.e., the number of Degrees of Freedom (DoF) equals the number of system inputs. This assumption proved useful for establishing that model-based techniques
could indeed be applied to soft robotics. Nevertheless, fully capturing the benefits of such approaches requires accounting for the inherent underactuation of these systems. Indeed, body deformability introduces additional mechanical coordinates that are not directly actuated, making underactuation an intrinsic consequence of compliance. Neglecting this aspect may result in inaccurate stability analyses and degraded control performance. Similar phenomena have long been studied in rigid robotic systems and become particularly relevant in the presence of joint or link flexibility, which introduces additional unactuated deformation coordinates. Moreover, when actuation and the controlled output are non-collocated, as often occurs in flexible robotic systems, non-minimum-phase behavior and unstable zero dynamics may arise, further complicating the control problem~\cite{de2021flexible}.

Driven by the above considerations, in recent years, there has been an increasing interest in designing provably-stable shape regulators for underactuated soft robots models~\cite{della2019control, franco2021energy, borja2022energy, pustina2022feedback, pustina2023psaitd, soleti2023energy, caradonna2024model, patterson2024modeling}. However, these efforts have primarily focused on specific sub-classes of robots, thus limiting the applicability of the proposed solutions. In~\cite{della2019control}, it is presented a PD-like controller for planar slender arms modeled as Cosserat rods with a polynomial approximation of the curvature strain. Shape regulation is achieved through feedback on the constant curvature (CC) mode. Building on the energy-shaping framework,~\cite{franco2021energy} proposes a control law for a 2D rigid-link model with linear elasticity and no gravitational loads. This control law compensates for constant unknown external disturbances and is validated experimentally on a pneumatic manipulator. Further advancements have been made in the design of globally convergent control strategies for planar robots. In~\cite{borja2022energy}, energy-based regulators are developed for robots where unactuated dynamics are dominated by elasticity, while~\cite{pustina2022feedback} addresses robots with elastically decoupled actuated and unactuated DoF. These results are extended in~\cite{pustina2023psaitd} by incorporating a saturated integral term to compensate for model mismatches and constant disturbances. Under suitable assumptions on the dynamic model,~\cite{soleti2023energy} proposes a class of energy-shaping regulators guaranteeing local asymptotic stability. The method is validated on an elastic structure actuated by soft dielectric elastomer actuators. Previous works implementing collocated control in soft robots have also shown promising results, such as~\cite{bhatti2026experimental,perfetta2026reconciling}, although these studies focus on specific classes of regulators or particular class of soft robots.

The works described above propose various solutions for shape regulation, each based on specific assumptions. Collectively, these assumptions contribute to a fragmented understanding of shape regulation in soft robotics. Moreover, many of these controllers are validated solely through simulation, perpetuating the misconception that underactuation cannot be effectively addressed in practical control design. The primary objective of this paper is to investigate, both theoretically and experimentally, how shape regulation can be achieved using different control strategies and to compare their performance and stability properties. Furthermore, we aim to demonstrate that this challenge can be addressed independently of the kinematic discretization and its order. This is made possible by leveraging the collocated form, introduced in our previous work~\cite{pustina2024input}. The collocated form represents soft robot dynamics in a manner where only some DoF are directly influenced by system inputs. This approach enables the development of relatively simple and robust solutions, adhering to the principle that control strategies heavily reliant on model-dependent terms are more susceptible to failure due to the uncertainties in soft robot models. This paper also presents the first systematic experimental validation of the collocated form across multiple regulators and dynamic models.

The main contributions of the paper are summarized as follows.
\begin{enumerate}[label=(\roman*)]
    \item We formalize shape regulation via collocated control and demonstrate that this challenge can be addressed within a unified framework. We propose three classes of PD and PID-like regulators with increasing stability properties: local, semi-global, and global stability. 
    \item The majority of the proposed regulators are experimentally validated on a novel soft robotic platform, and their performance is systematically analyzed and compared. 
    \item Seven different dynamic models are employed for each controller, with their performance systematically evaluated and compared to assess their impact on shape regulation.  
\end{enumerate}
The efficient real-time implementation of the proposed controllers relies on the inverse dynamics algorithm introduced in~\cite{pustina2025recursive}.

The remainder of the paper is organized as follows. Section~\ref{sec:preliminaries} presents the preliminaries, including the system model, its collocated form, and the assumptions underlying the proposed controller design. Section~\ref{sec:definition_cc} defines collocated control in soft continuum robots and introduces the shape regulation problem via collocated control, which is briefly illustrated in Fig.~\ref{fig:cover}. Section~\ref{sec:shape regulators} presents the main theoretical contributions, namely, the derivation of several provably stable shape regulators. Section~\ref{sec:experiments} details the experimental validation and performance evaluation of the proposed approaches, as illustrated in Fig.~\ref{fig:cover}. Finally, Section~\ref{sec:conclusions} concludes the paper.

Moreover, most of the mathematical proofs and extensive experimental results have been moved to the appendix to improve readability and maintain the flow of the manuscript.

\begin{figure*}
    \centering
    \includegraphics[width=1\linewidth]{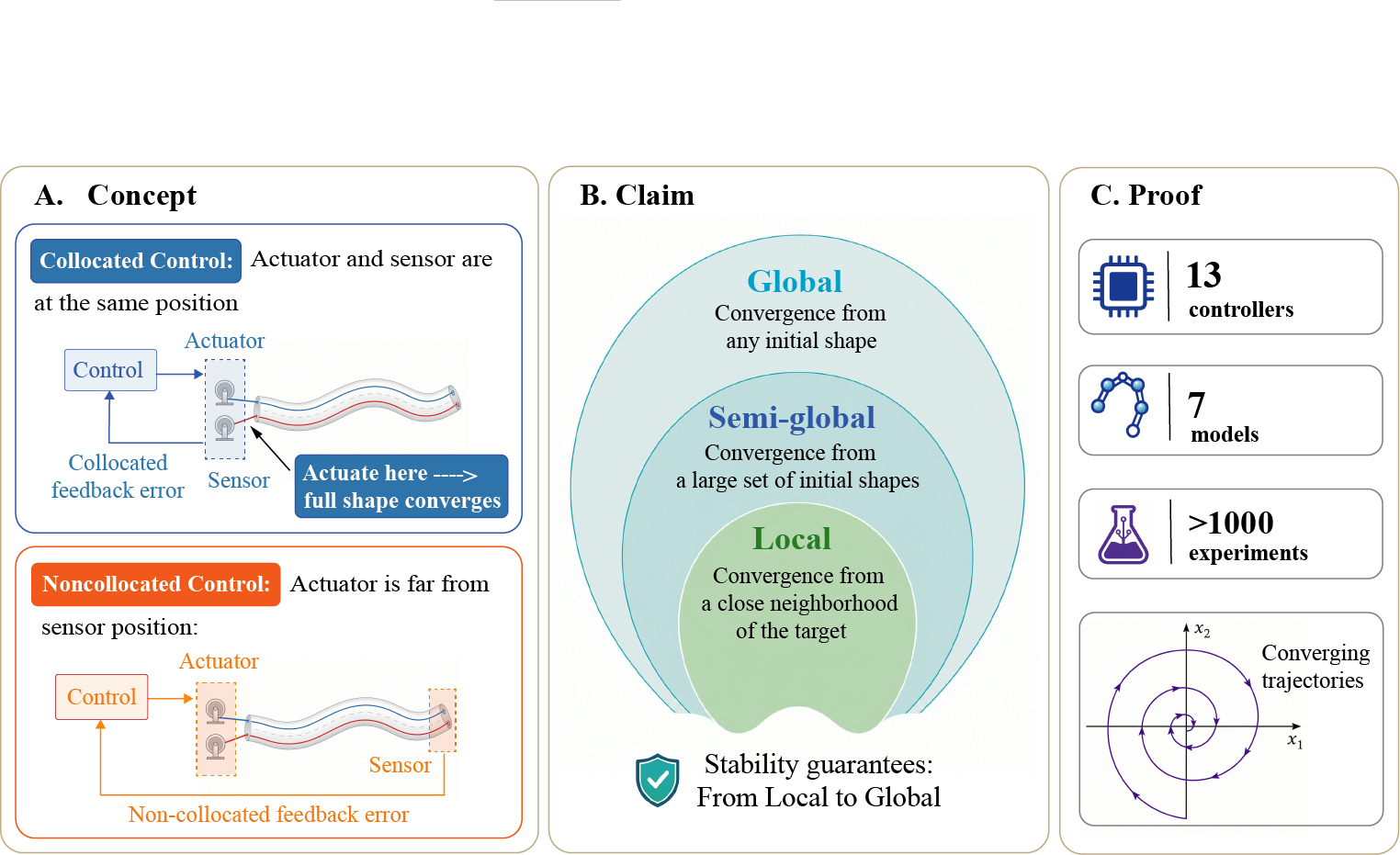}
    \caption{Drive a soft robot’s shape by actuating only where you can - $collocated \ control$ - with guarantees that scale from model-free practical regulation to model-based global stability guarantees. This approach paves the way for shape regulation in soft robots with stability guarantees, without requiring a full dynamic model.}
    \label{fig:cover}
\end{figure*}

\subsection{Notation}
Bold symbols are used to represent vectors and matrices. Function arguments are omitted where the context provides sufficient clarity. For ease of reference, Table~\ref{tab:notation} summarizes the notation adopted in this work.
\begin{table}[t!]
\centering
\caption{Nomenclature}
\begin{tabular}{p{0.35\columnwidth} p{0.55\columnwidth} }
\toprule
Symbol & Description\\
\toprule
$\R^{n}$ & Euclidean space of dimension $n$\\
$\R^{n \times m}$ & Space of $n \times m$ matrices over $\R$\\
$\R^{+}$ & Positive real numbers $n$\\
$\zerov_{n \times m}$ & Zero matrix of dimension $n \times m$\\
$\Pm > 0$ & Symmetric positive definite matrix\\
$\norm{\vv}$ & Euclidean norm of $\vv \in \R^{n}$\\
$\norm{\Am}$ & Matrix norm of $\Am \in \R^{n \times n}$ induced by vector norm\\
$\lambda_{\min}({\Am}) \,\, (\lambda_{\max}(\Am))$ & Smallest (largest) eigenvalue of $\Am \in \R^{n \times n}$\\
$\sigma_{\min}({\Am}) \,\, (\sigma_{\max}(\Am))$ & Smallest (largest) singular value of $\Am \in \R^{n \times m}$\\
$\grad{f(\xv)}{\xv} \in \R^{n}$ & Gradient of $f(\xv) : \R^{n} \rightarrow \R$\\
$\jac{\fv(\xv)}{\xv} \in \R^{m \times n}$ & Jacobian of $\fv(\xv) : \R^{n} \rightarrow \R^{m}$\\
$\hess{\fv(\xv)}{\xv} \in \R^{m \times n \times n}$ & Hessian of $\fv(\xv) : \R^{n} \rightarrow \R^{m}$\\
\bottomrule
\end{tabular}
\label{tab:notation}
\end{table}

\section{Preliminaries}\label{sec:preliminaries}
In this section, we begin by introducing the robot dynamics and the working assumptions. Subsequently, we apply the coordinate transformation introduced in~\cite{pustina2024input} to reformulate the equations of motion in their collocated form, wherein we define two interesting subclasses of soft robots for which global stability guarantees will be established.

\subsection{Dynamic model}\label{sec:dynamic model}
We consider continuum soft robots whose dynamics are described by the following reduced-order model
\begin{equation}\label{eq:dynamics}
    \Mm(\qv)\ddqv + \Cm(\qv, \dqv)\dqv + \gv(\qv) + \kv(\qv) + \Dm(\qv)\dqv = \Am(\qv)\uv,
\end{equation}
where $\qv \in \R^{n}$ is the vector of configuration variables, with $\dqv$ and $\ddqv$ representing its first and second time derivatives, respectively. Moreover, $\Mm(\qv) \in \R^{n \times n}$ is the mass matrix, and $\Cm(\qv, \dqv) \in \R^{n \times n}$ is the Coriolis matrix. The vectors $\gv(\qv) = \grad{\mathcal{U}_{g}(\qv)}{\qv}$ and $\kv(\qv) = \grad{\mathcal{U}_{e}(\qv)}{\qv}$ model the gravitational and elastic forces, derived from the potentials $\mathcal{U}_{g}(\qv) \in \R^{+}$ and $\mathcal{U}_{e}(\qv) \in \R^{+}$, respectively. The damping matrix is $\Dm(\qv) \in \R^{n \times n} > 0$, and $\uv \in \R^{m}$ are the system inputs, mapped in configuration space through the actuation matrix $\Am(\qv) \in \R^{n \times m}$. All above terms can be computed starting from the robot dynamic parameters, such as the mass distribution, and from its forward kinematics
\begin{equation}\label{eq:kinematics}
    \tilde{\xv} = \pv(\qv, \xv) \in \R^{3},
\end{equation}
where $\xv \in V$ represents the points in the region $V \subset \R^{3}$ occupied by the robot in its stress-free configuration $\qv = \zerov_{n}$. For a comprehensive introduction to this topic, we refer the reader to~\cite{dellasantina2023survey, pustinaphdthesis}. In the following derivations, the robot kinetic and potential energy are denoted $\mathcal{H}(\qv, \dqv) = \displaystyle\frac{1}{2}\dqv^{T}\Mm(\qv)\dqv$ and $\mathcal{U}(\qv) = \mathcal{U}_{g}(\qv) + \mathcal{U}_{e}(\qv)$, respectively. 
\begin{remark}
    To account for the inherently underactuated nature of soft robots, in this paper, the number of system inputs is fewer than the number of configuration variables, i.e., $m < n$.
\end{remark}

The control laws developed in this paper rely on four assumptions.  
\begin{enumerate}[label=(\roman*)]
    \item The robot kinematics and its derivatives with respect to the configuration variables up to second order are bounded.
    \item The elastic energy can be lower bounded in norm by a quadratic function of the configuration variables, its gradient can be upper bounded by a linear function of the configuration, and its Hessian is convex and upper bounded.
    \item The damping matrix is upper bounded.
    \item The dynamic model admits collocated form.
\end{enumerate}

The first three assumptions are formalized and justified below, while the fourth is addressed in Sec.~\ref{sec:collocated form}.


\begin{assumption}\label{assumption:kinematics}
    There exist constants $\lambda_{p}$, $\lambda_{\partial p}$, and $\lambda_{\partial^{2} p} \in \R^{+}$ such that, for all $\xv \in V$ and $\qv \in \R^{n}$, $\norm{\pv(\xv,\qv)} \leq \lambda_{p}$, $\norm{\jac{\pv}{\qv}} \leq \lambda_{\partial p}$, and $\norm{\hess{p_i}{\qv}} \leq \lambda_{\partial^{2}p}$ for each component $p_i$ of $\pv$, $i=1,\ldots,\dim(\pv)$.
\end{assumption}

In principle, a kinematic model that considers linear strains, namely elongation and shear, could violate the above. However, in practice, a soft robot can deform only up to a limit before failure. Furthermore, when $\jac{\pv}{\qv}$ or $\hess{\pv}{\qv}$ are unbounded, there are directions in configuration space, and eventually in Cartesian space, where the inertia grows unbounded, which is impossible.
In fact, only under Assumption~\ref{assumption:kinematics}, the dynamic model satisfies the following remarkable properties~\cite{dellasantina2023survey, pustinaphdthesis}.

\begin{property}\label{property:mass matrix 2}
     The mass matrix $\Mm(\qv)$ and the stiffness matrix $\Km$ are symmetric and positive definite for all $\qv \in \R^{n}$, $\bm{M}(\bm{q})=\bm{M}^{T}(\bm{q})\succ \bm{0}$ and $\bm{K}=\bm{K}^{T}\succ \bm{0}$. The damping matrix is positive semi-definite $\bm{D} \succeq \bm{0}$.
\end{property}

\begin{property}\label{property:mass matrix}
    The mass matrix $\Mm(\qv)$ is upper bounded, i.e., there exists a constant $\lambda_{M} \in \R^{+}$ such that, for all $\qv \in \R^{n}$, 
    $
        \norm{ \Mm(\qv) } \leq \lambda_{M}.
    $
\end{property}

\begin{property}\label{property:coriolis matrix}
    There exists a constant $\lambda_{C} \in \R^{+}$ such that, for all $\qv \in \R^{n}$ and $\dqv \in \R^{n}$, 
    $
        \norm{ \Cm(\qv, \dqv)} \leq \lambda_{C}\norm{\dqv}.
    $ In addition, $\dot{\bm{M}}(\bm{q})-2\dot{\bm{C}}(\bm{q}, \dot{\bm{q}})$ is skew symmetric or, equivalently,  $\dot{\bm{M}}(\bm{q}) = \bm{C}(\bm{q}, \dot{\bm{q}}) + \bm{C}^{T}(\bm{q}, \dot{\bm{q}})$.
\end{property}
\begin{property}\label{property:gravity}
    There exists a constant $\lambda_{\mathcal{U}_{g}} \in \R^{+}$ such that, for all $\qv \in \R^{n}$,  
    $
        0 \leq \norm{\mathcal{U}_{g}(\qv)} \leq \lambda_{\mathcal{U}_{g}}.
    $
    Furthermore, there exists constants $\lambda_{G}$ and $\lambda_{\partial G} \in \R^{+}$ such that, for all $\qv \in \R^{n}$, 
    $
        \norm{\gv(\qv)} \leq \lambda_{G},
    $
    and 
    $
        \norm{\jac{\gv}{\qv}} \leq \lambda_{\partial G}.
    $
    The latter inequality also implies that $\gv(\qv)$ is Lipschitz, namely $\norm{ \gv(\xv) - \gv(\yv)} \leq \lambda_{\partial G} \norm{\xv - \yv}$.
\end{property}
Moving to the elastic field, we ask the following.
\begin{assumption}\label{assumption:elastic energy}
    There exists constants $\lambda_{u_{e}}$, $\lambda_{K}$, $\lambda_{\partial k}$, $\lambda_{\partial K} \in \R^{+}$ such that, for all $\qv \in \R^{n}$, $\lambda_{u_{e}} \norm{\qv}^{2} \leq \mathcal{U}_{e}(\qv)$, $\norm{\kv(\qv)} \leq \lambda_{K} \norm{\qv}$ and $\lambda_{\partial k} \leq \jac{\kv}{\qv} \leq \lambda_{\partial K}$. The latter inequality also implies that $\kv(\qv)$ is Lipschitz.
\end{assumption} 
The constraints imposed on $\mathcal{U}_{e}$ and its derivatives ensure that the model reflects empirically observed properties of deformable systems. Specifically, in the absence of external forces, such as gravity, a continuum oscillates and converges\footnote{Neglecting hysteresis and other highly nonlinear effects.} to its stress-free configuration. Therefore, $\qv = \zerov_{n}$ must be a global minimizer of $\mathcal{U}_{e}$. This holds if and only if $\qv = \zerov_{n}$ is the unique solution to the statics $\kv(\qv) = \zerov_{n}$, and the Hessian of $\mathcal{U}_{e}$ is positive definite everywhere. Additionally, in the stress-free configuration, the robot should store no elastic energy, meaning $\mathcal{U}_{e}(\zerov) = 0$. The constraint on the boundedness of $\jac{\kv}{\qv}$ reflects that the arm cannot become arbitrarily stiff in a direction. Such conditions are all satisfied under Assumption~\ref{assumption:elastic energy}. In other words, the assumption guarantees, in a sufficient sense, that the model exhibits properties analogous to those observed in the physical system.

Finally, for the viscous term, the following is assumed, which requires there are no directions where the dissipation rate becomes arbitrarily large.
\begin{assumption}\label{assumption:damping}
    There exists a constant $\lambda_{D} \in \R^{+}$ such that, for all $\qv \in \R^{n}$, $\norm{ \Dm(\qv) } \leq \lambda_{D}$.
\end{assumption}

\subsection{Collocated Form}\label{sec:collocated form}
The actuation matrix $\Am(\qv)$ introduces significant complexity in the design of a controller for~\eqref{eq:dynamics} because it distributes actuator forces across all DoF. Consequently, determining whether a desired actuation force in the configuration space $\tauv_{d}$ can be achieved via a command $\uv_{d}$ satisfying $\Am(\qv)\uv_{d} = \tauv_{d}$ is not straightforward\footnote{In this paper, $\tauv_{d}$ represents the force required to achieve a desired shape.}. In fact, the linear system of equations $\Am(\qv)\uv_{d} = \tauv_{d}$ is generally overdetermined because $\Am(\qv)$ is a tall matrix. As shown in~\cite{pustina2024input}, this issue can be addressed by performing a change of coordinates that maps the inputs $\uv$ directly into the configuration space. Below, we briefly introduce such coordinate transformation and refer readers to~\cite{pustina2024input} for a more detailed explanation.

\begin{assumption}\label{assumption:collocated form}
The columns of $\Am(\qv)$ are exact differentials, i.e., there exists a function $\hv(\qv) : \R^{n} \rightarrow \R^{m}$ satisfying $\jac{\hv}{\qv} = \Am^{T}(\qv)$.    
\end{assumption}
Under Assumption~\ref{assumption:collocated form}, the change of coordinates
\begin{equation*}
    \thetav = \begin{carray}{c}
        \thetav_{a}\\
        \thetav_{u}
    \end{carray} = \begin{carray}{c}
        \hv(\qv)\\
        \lv(\qv)
    \end{carray},
\end{equation*}
with $\lv(\qv) : \R^{n} \rightarrow \R^{n-m}$ any complement to $\hv$, yields equations of motion in the form of~\eqref{eq:dynamics collocated splitted}. 
\begin{figure*}[t!]
\begin{equation}\label{eq:dynamics collocated splitted}\small
    \arraycolsep=1.4pt
    \begin{carray}{cc}
        \Mm_{{aa}} & \Mm_{{au}}\\
        \Mm_{{au}}^{T} & \Mm_{{uu}}
    \end{carray}\begin{carray}{c}
        \ddthetav_{a}\\
        \ddthetav_{u}
    \end{carray} + \begin{carray}{cc}
        \Cm_{{aa}} & \Cm_{{au}}\\
        \Cm_{{au}}& \Cm_{{uu}}
    \end{carray}\begin{carray}{c}
        \dthetav_{a}\\
        \dthetav_{u}
    \end{carray} + \begin{carray}{c}
        \gv_{{a}}\\
        \gv_{{u}}
    \end{carray} + 
    \begin{carray}{c}
        \kv_{{a}}\\
        \kv_{{u}}
    \end{carray} + \begin{carray}{cc}
        \Dm_{{aa}} & \Dm_{{au}}\\
        \Dm_{{au}} & \Dm_{{uu}}
    \end{carray}\begin{carray}{c}
        \dthetav_{a}\\
        \dthetav_{u}
    \end{carray} = \begin{carray}{c}
        \uv\\\zerov_{n-m}
    \end{carray}
\end{equation}
\end{figure*}
Eq.~\eqref{eq:dynamics collocated splitted} is called the collocated form because on its right-hand side, the control inputs directly affect some DoF. The configurations variables $\thetav_{a} = \hv(\qv)$ represent the coordinates on which the actuators do work and are thus termed actuation coordinates. From~\eqref{eq:dynamics collocated splitted}, it becomes evident that the control objective must be achieved through a generalized actuation force $\tauv_{d}$ structured as
\begin{equation*}
    \tauv_{d} = \begin{carray}{c}
        \uv_{d}\\ \zerov_{n-m}
    \end{carray}.
\end{equation*}
The previous equation can be seen as a constraint on controllers design, which originates from the underactuated nature of soft robots. It is also important to note that all the properties introduced in Sec.~\ref{sec:dynamic model}, such as the boundedness of the mass matrix and the Lipschitz continuity of potential forces, remain valid in the collocated form.

Assumption~\ref{assumption:collocated form} provides a necessary and sufficient condition for the existence of the collocated form under a change of coordinates. This condition can be relaxed by introducing also an input transformation~\cite[Lemma 4.3.2]{pustinaphdthesis}. In this paper, we focus on shape regulation for systems where the dynamics admit a collocated form, motivated by two key reasons. First, if the model cannot be represented as in~\eqref{eq:dynamics collocated splitted}, the controller must explicitly account for $\Am(\qv)$. Addressing this challenge remains an open problem. Second, many control-oriented models in soft robotics satisfy Assumption~\ref{assumption:collocated form}. For example, it has been demonstrated in~\cite{pustina2024input} that Cosserat rods driven by tendons and robots with fluidic actuators admit a collocated form. The existence of actuation coordinates $\hv(\qv)$ can also be motivated using power-based arguments. From~\eqref{eq:dynamics}, the power at the robot side is
\begin{equation*}
    P_{r} = \rb{\Am(\qv)\uv}^{T}\dqv = \uv^{T} \Am^{T}(\qv) \dqv.
\end{equation*}
Meanwhile, the power generated due to actuation can be expressed as
\begin{equation*}
    P_{a} = \uv^{T}\dot{\av},
\end{equation*}
for some variables $\av(t) \in \R^{m}$ on which the actuators perform work. For example, the torque generated by an electric motor performs work on its angular position, while the pressure in a pneumatic chamber performs work on its volume. If the system is subject only to scleronomous holonomic constraints, $\av$ depends solely on $\qv$ and not explicitly on $t$ and $\dqv$, i.e., $\av(\qv(t))$. Neglecting losses in power transfer, we have
\begin{equation*}
    P_{r} = \uv^{T} \Am^{T}(\qv) \dqv = \uv^{T} \dot{\av} = P_{a}.
\end{equation*}
Since the actuation inputs are controlled variables, the equality holds independently of $\uv$, implying
\begin{equation*}
    \dot{\av} = \Am^{T}(\qv) \dqv,
\end{equation*}
and
\begin{equation*}
    \thetav_{a} = \hv(\qv) = \av(\qv).
\end{equation*}
Sometimes $\thetav_{a}$ is measurable through proprioceptive sensors mounted on the actuators. For instance, this applies to tendon-driven Cosserat rods~\cite{pustina2024input}. This comes in handy because, as demonstrated in the following sections, it is possible to design provably stable shape regulators using $\thetav_{a}$ alone. However, the implementation of most controllers in this paper relies on full feedback from $\thetav$, or equivalently from $\qv$. We consider this dependence to be a technological limitation of the field, which could be addressed through advancements in sensing technologies or state observers. In this direction,~\cite{feliu2024dynamic} introduces an observer driven by actuation coordinates to estimate the state variables $\qv$ and $\dqv$ of Cosserat rods.

Introducing the change of coordinates $\thetav$ has one main drawback: it renders the generalized stiffness nonlinear, even in cases where $\kv(\qv) = \Km \qv$. However, a special case occurs when $\Am$ is constant. In this scenario, not only the actuation coordinates always exist as $\thetav_{a} = \Am^{T} \qv$ but also the elastic force remains linear. 

\subsection{Two interesting sub-classes of soft robots}

We conclude this section by defining two sub-classes of soft robots, introduced in~\cite{borja2022energy} and~\cite{pustina2022feedback}, for which control laws with global convergence properties can be designed.
\begin{class}\label{class:elastically decoupled}
    A soft robot is elastically decoupled if its generalized stiffness force takes the form
    \begin{equation*}
        \kv(\thetav) = \begin{carray}{c}
            \kv_{a}(\thetav_{a})\\
            \kv_{u}(\thetav_{u})
        \end{carray}.
    \end{equation*}
\end{class}

In other words, a robot is elastically decoupled when the elastic forces acting on the actuated and unactuated parts depend solely on their respective DoF. Whether a soft robot is elastically decoupled or can be approximated as such depends on at least two factors.
\begin{enumerate}[label=(\roman*)]
    \item The DoF of different bodies are always decoupled since internal deformation forces remain confined within each body.
    \item Distributing actuation forces across multiple bodies introduces coupling between actuated and unactuated DoF, as reflected in the stiffness matrix in the {collocated form}.
\end{enumerate}

\begin{class}\label{class:elastically dominated}
    A soft robot is elastically dominated if, for all $\thetav \in \R^{n}$, it holds
    \begin{equation*}
        \jac{\kv_{u}(\thetav)}{\thetav_{u}} + \jac{\gv_{u}(\thetav)}{\thetav_{u}} > 0.
    \end{equation*}
\end{class}
The above condition requires that the stiffness in the unactuated directions dominates gravitational effects, ensuring the potential energy is convex with respect to the unactuated DoF. Consequently, the actuation coordinates uniquely determine the equilibrium (see Section~\ref{section:shape regulation:PD regulator}). 


\subsection{Control Objective}\label{sec:control objective}

\begin{center}
\begin{minipage}{0.96\linewidth}
\textbf{Goal.}
Given a desired shape $\bm{q}_d$, define the collocated reference as 
$\bm{\theta}_{a_d}=\bm{h}(\bm{q}_d)$. Design a control input $\bm{u}$, in the form introduced in Definition \ref{Definition_2}, such that

\begin{equation}
\label{eq:control-objective}
\lim_{t\to\infty}\bm{\theta}_a(t)=\bm{\theta}_{a,d}.
\end{equation}

Furthermore: a) characterize theoretically the conditions under which
$\bm{\theta}_u$ also converges to a compatible equilibrium $\overline{\bm{\theta}}_u$, so that the full shape
satisfies $\bm{q}(t)\to \bm{q}_d$; and b) evaluate experimentally the controller performance, including the effects of discretization choice and gain tuning when applied to a real system.
\end{minipage} 
\end{center}

\section{Collocated shape regulation: definition and control goal}\label{sec:definition_cc}
As discussed in the Introduction, collocated control broadly refers to the regulation of the configuration variables that are directly paired with the actuation inputs in a power-conjugate sense, referred to as collocated variables or actuation coordinates (see Fig. \ref{fig:cover}). This property is particularly valuable in robotics because aligning sensing and actuation generally improves stability and robustness in the control of these variables, thereby mitigating instability arising from modeling errors, delays, and actuator–sensor mismatch. Such characteristics are essential for reliable task execution in applications such as manipulation, locomotion, surgical robotics, and cooperative or contact-rich tasks. In the context of continuum soft robotics, collocated control enables the development of a novel unified framework that shares similarities with conventional robotics. Figure ~\ref{fig:block_diagram} shows the block diagram of the proposed collocated shape-regulation framework.

\begin{figure*}
    \centering
    \includegraphics[width=1\linewidth]{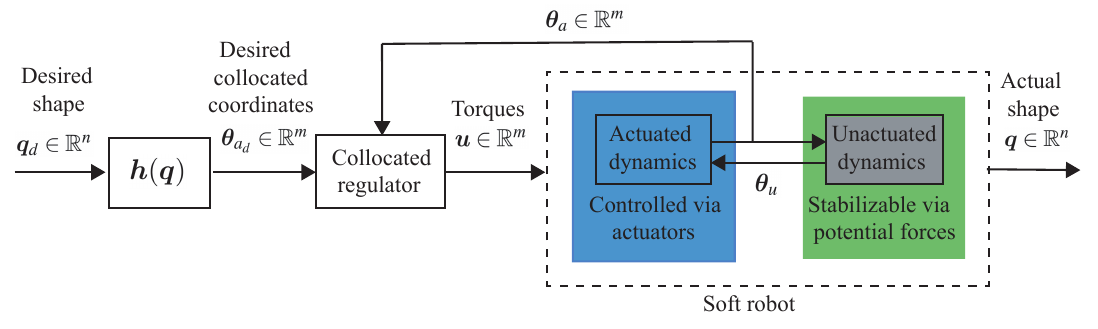}
    \caption{Block diagram of the collocated shape-regulation framework, incorporating the coordinate transformation, $\bm{h}(\bm{q})$, a collocated shape regulator, and a decomposition of the system dynamics into actuated and unactuated components.}
    \label{fig:block_diagram}
\end{figure*}

\begin{definition}
   Consider the coordinate partition
$\bm{\theta}=(\bm{\theta}_a, \ \bm{\theta}_u)$.
The coordinates $\bm{\theta}_a$ are said to be \emph{collocated} if they are power-conjugate to the inputs, i.e., the control input $\bm{u}$ enters only the generalized-force equations associated with $\bm{\theta}_a$ and is absent from the generalized-force equations associated with $\bm{\theta}_u$.
\end{definition} 

\textit{Caveat}: This definition concerns the distribution of generalized forces, not the resulting accelerations. Due to inertial and dynamic coupling, the motion of $\bm{\theta}_u$ may still depend on $\bm{u}$ even though no generalized force is applied directly to the $\bm{\theta}_u$ coordinates.

\begin{definition}\label{Definition_2}
    A controller is said to be \emph{collocated} if its regulated tracking error is defined exclusively in terms of
the collocated coordinates, their derivatives, and any associated controller states. Let $\tthetav_{a}$ denote the tracking error associated with the collocated coordinates. A collocated controller can generally be represented as

\begin{equation}
\label{ec:coll_control}
\uv = \mathcal{F}
(t,\tthetav_{a}, \dot{\tilde{\bm{\theta}}}_a, \bm{\theta}, \dot{\bm{\theta}}, \bm{z}) 
\end{equation}

 where $\bm{z}$ denotes auxiliary controller states, such as integrators, observers, or filters.
\end{definition}



\begin{remark} 
 The non-collocated coordinates $\bm{\theta}_u$ may appear in model-based controller terms, but they are not directly regulated. Their behavior is therefore governed by the internal dynamics of the resulting closed-loop system.
\end{remark}


\begin{definition}
    Consider a soft robot represented in collocated form. The family of collocated controllers considered in this work consists of feedback laws satisfying Definition \ref{Definition_2} and admitting the explicit representation

  {\small
\begin{equation}
\label{ec:controller_def}
\uv = \underbrace{\Fv_d(\thetav)}_{\substack{\text{decoupling}\\\text{mapping}}}
\underbrace{ ( \ddot{\bm{\theta}}_{a_d}
+ \bm{K}_D \dot{\tilde{\bm{\theta}}}_a
+ \bm{K}_P \tthetav_{a}
+ \bm{K}_I \bm{z}_{a}) }_
{\substack{\text{collocated}\\\text{feedback error}}}
+
\underbrace{\Fv_c(\thetav,\dthetav,\bm{\theta}_{d})}
_{\substack{\text{compensation}\\\text{cancellation terms}}}
\end{equation}
}
    
    where $\zv_a$ is a set of controller states associated with the error of the collocated variables\footnote{For example, $\zv$ may represent the state variables associated with an integral action.}, and $\thetav_{a_{d}}(t)$ is the desired value of the actuation coordinates. 
\end{definition}


With reference to the above definition, when $\thetav_{a_{d}}(t)$ is a function of time, the problem is referred to as trajectory tracking; when it is constant, i.e., $\thetav_{a_{d}}(t) = \thetav_{a_{d}}$, it is referred to as regulation. 
Collocated controllers provide a mean of controlling the shape of a continuum because, by definition, the configuration variables parametrize all its possible deformations. For fully actuated approximations ($n=m$), shape control can be solved through the actuation coordinates using controllers inspired by rigid manipulators~\cite{pustina2024input}. However, this becomes non-trivial for underactuated models ($n>m$), even for regulation tasks. As a matter of fact, shape regulation remains an open challenge in soft robotics. 

\begin{problem}[\textbf{Shape Regulation via Collocated Control}]
\label{problem:control problem}
Under Assumptions~\ref{assumption:kinematics}--\ref{assumption:collocated form}, find a feedback controller (\ref{ec:controller_def}), such that $\thetav_{a}$ converges asymptotically to an arbitrary desired reference $\thetav_{a_{d}}$, i.e.,

\begin{equation}\label{ec:control_objective}
        \lim_{t \rightarrow \infty} \tthetav_{a}(t) = \lim_{t \rightarrow \infty} \thetav_{a_{d}}(t) -\thetav_{a}(t) = \zerov_{m},
    \end{equation}

\end{problem}

It is important to note that the shape of a soft robot is uniquely determined by all configuration variables. Consequently, enforcing $\tthetav_{a} = \zerov_{m}$ does not necessarily guarantee that the manipulator reaches a specific desired shape.

\begin{example}
Consider a planar soft robot consisting of two segments, with the base oriented such that gravitational forces destabilize the straight configuration. Assume only the first segment is actuated, and the objective is to stabilize the straight configuration. Enforcing the straight configuration for the first segment alone does not guarantee that the passive segment also aligns, as the straight configuration is unstable for the passive segment.   
\end{example}
Such scenarios require sophisticated controllers~\cite{dellasantina2020soft}, analogous to those used for swing-up tasks in underactuated systems like the pendubot. These advanced strategies fall outside the scope of this work, which focuses on providing general and relatively simple control methodologies. Nevertheless, we will derive conditions under which shape regulation via collocated control can effectively steer the entire robot shape.

\begin{remark}
    Collocated control for soft robots enables a unified shape-regulation framework that is independent of the chosen kinematic discretization. This is because the collocated variables are directly paired with the actuation inputs and can therefore be regulated independently of the discretization. In contrast, the non-collocated variables, $\bm{\theta}_u$, depend on the discretization order and are not directly actuated; instead, they are stabilized indirectly through the potential forces of the system.
    
\end{remark}

\section{Shape Regulators}\label{sec:shape regulators}
In this section, we present the theoretical contributions of the paper: a series of PD and PID-like controllers designed to address shape regulation. The proposed results are presented incrementally, starting with regulators that guarantee local convergence to those with global convergence properties. Closed-loop stability is established under lower bounds on the proportional gain $\bm{K}_P$. Conservative choices for $\bm{K}_P$ are discussed in the stability proofs. A summary of the proposed controller families, their performance guarantees, and implementation costs is represented in Fig. \ref{fig:block_summary}.

\begin{figure}[ht]
    \centering
    \includegraphics[width=1\linewidth]{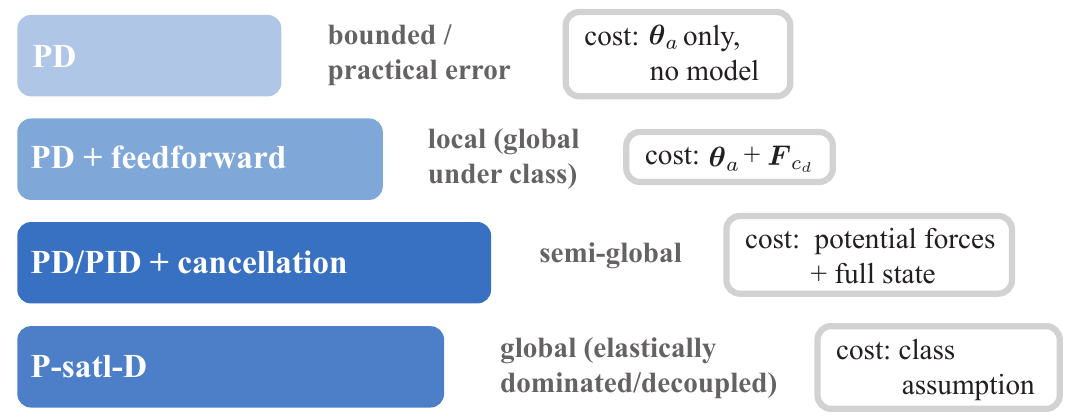}
    \caption{Summary of the proposed controller families, their performance guarantees, and implementation costs.}
    \label{fig:block_summary}
\end{figure}

\subsection{PD Regulator}\label{section:shape regulation:PD regulator}
The simplest feedback controller for shape regulation is a PD
\begin{equation}\label{eq:PD}
    \uv = \bm{K}_P\tthetav_{a} - \bm{K}_D\dthetav_{a},
\end{equation}
where $\bm{K}_P > 0$ and $\bm{K}_D \geq 0 \in \R^{m \times m}$. Note that derivative action is not strictly necessary due to the inherent dissipative properties of the system. This controller is highly practical and easy to implement because it only requires measuring $\bm{\theta}_a$.

Although a PD does not fully solve Problem~\ref{problem:control problem}, it remains a widely used control strategy in soft robotics~\cite{della2021model}. Moreover, it allows conclusions to be drawn that also apply to more advanced control laws. Stability can be analyzed using the energy of the closed-loop system
\begin{equation*}
    V = \mathcal{H} + \frac{1}{2}\tthetav_{a}^{T}\bm{K}_P\tthetav_{a}.
\end{equation*}
This function is non-negative and radially unbounded. After some computations and simplifications, its time derivative reads
\begin{equation*}
    \dV = - \dthetav^{T}\Dm\dthetav - \dthetav_{a}^{T}\bm{K}_D\dthetav_{a} \leq 0.
\end{equation*}
Thus, by applying LaSalle Invariance Principle~\cite{khalil2002nonlinear}, the trajectories are bounded and converge to the set $\{ ( \thetav, \dthetav ) = (\thetav_{eq}, \zerov_{n}) \}$, where $\thetav_{eq}$ satisfies
\begin{equation}\label{eq:PD:equilibrium}
    \begin{carray}{c}
        \gv_{a}(\thetav) + \kv_{a}(\thetav)\\
        \gv_{u}(\thetav) + \kv_{u}(\thetav)
    \end{carray} = \begin{carray}{c}
        \bm{K}_P\tthetav_{a}\\\zerov_{n-m}
    \end{carray}.
\end{equation}
Eq.~\eqref{eq:PD:equilibrium} is solved by $\thetav_{a_{d}}$ only if $\thetav_{a_{d}}$ is an open-loop equilibrium. In general, however, the actuated equilibrium $\thetav_{a_{eq}}$ differs from $\thetav_{a_{d}}$. This is because the controller does not embed any term that forces the first $m$ equations of~\eqref{eq:PD:equilibrium} to have $\thetav_{a_{d}}$ as solution. However, by increasing the proportional gain, $\thetav_{a_{eq}}$ can be made arbitrarily close to $\thetav_{a_{d}}$. Indeed, multiplying both sides of~\eqref{eq:PD:equilibrium} by
\begin{equation*}
    \begin{carray}{cc}
        \bm{K}_P^{-1} & \zerov_{m \times (n-m)}\\
        \zerov_{(n-m) \times m} & \IIm_{n-m}
    \end{carray},
\end{equation*}
and expanding into actuated and unactuated components, gives
\begin{equation*}
    \begin{carray}{c}
        \bm{K}_P^{-1}\rb{\gv_{a} + \kv_{a}}\\
        \gv_{u} + \kv_{u}
    \end{carray} = \begin{carray}{c}
        \thetav_{a_{d}} - \thetav_{a_{eq}}\\
        \zerov_{n-m}
    \end{carray}.
\end{equation*}
As $\bm{K}_P \rightarrow \infty$, $\bm{K}_P^{-1} \rightarrow \zerov_{m \times m}$, and $\thetav_{a_{eq}} \rightarrow \thetav_{a_{d}}$, because all model terms are bounded by assumption. It is worth observing that increasing $\bm{K}_P$ has three main drawbacks. 
\begin{enumerate}[label=(\roman*)]
    \item It modifies the natural stiffness and compliance of the robot~\cite{della2017controlling}.
    \item It can lead to actuator saturation during transients.
    \item It can cause instability due to measurement noise. 
\end{enumerate}
The considerations above may also apply to the other regulators discussed in this paper. However, for control laws other than PD, the lower bounds on $\bm{K}_P$ are finite and act as sufficient conditions. In contrast, for the PD, the requirement that $\bm{K}_P$ be large--theoretically infinite--is a necessary condition for convergence.

In addition to the above issue,~\eqref{eq:PD:equilibrium} has typically many different roots, implying that one cannot even control consistently a single equilibrium through a PD. To guarantee uniqueness of equilibria, a sufficient condition can be derived by considering the closed-loop potential energy
\begin{equation*}
    \mathcal{U}_{cl} = \mathcal{U} + \frac{1}{2}\tthetav_{a}^{T}\bm{K}_P\tthetav_{a}, 
\end{equation*}
whose extrema are the solutions to~\eqref{eq:PD:equilibrium} and check when $\mathcal{U}_{cl}$ is convex. Specifically, there exists a global minimum when
\begin{equation}\label{eq:PD with feedforward hessian}
\small
    \hess{\mathcal{U}_{cl}}{\thetav} = \jac{\gv}{\thetav} + \jac{\kv}{\thetav} +
    \begin{carray}{cc}
        \bm{K}_P & \zerov_{m \times (n-m)}\\
        \zerov_{(n-m) \times m} & \zerov_{(n-m) \times (n-m)}
    \end{carray} > 0.
\end{equation}
If $\lambda_{\partial K} > \lambda_{\partial G}$, then $\hess{\mathcal{U}_{cl}}{\thetav}$ is positive definite and $\mathcal{U}_{cl}$ convex. Notably, the above can be relaxed if the robot is elastically dominated.
\begin{lemma}\label{lemma:PD:uniqueness of equilibria}
    Suppose the robot is elastically dominated (Class~\ref{class:elastically dominated}). If $\bm{K}_P$ is large enough, then there exists a unique solution to~\eqref{eq:PD:equilibrium}.
\end{lemma}
\begin{proof}
According to the Schur criterion,~\eqref{eq:PD with feedforward hessian} is positive definite if and only if 
\begin{equation}\label{eq:PD with feedforward:schur 1}
    \jac{\gv_{u}}{\thetav_{u}} + \jac{\kv_{u}}{\thetav_{u}} > 0
\end{equation}
and
\begin{equation}\label{eq:PD with feedforward:schur 2}
\begin{split}
    &\bm{K}_P + \jac{\gv_{a}}{\thetav_{a}} + \jac{\kv_{a}}{\thetav_{a}}-\rb{ \jac{\gv_{a}}{\thetav_{u}} + \jac{\kv_{a}}{\thetav_{u}} }\\
    & \rb{\jac{\gv_{u}}{\thetav_{u}}+ \jac{\kv_{u}}{\thetav_{u}}}^{-1} \rb{ \jac{\gv_{a}}{\thetav_{u}} + \jac{\kv_{a}}{\thetav_{u}} }^{T} > 0.
\end{split}
\end{equation}
The first matrix inequality is satisfied if the robot is elastically decoupled. As for the second, it is always possible to choose $\bm{K}_P$ such that~\eqref{eq:PD with feedforward:schur 2} holds. For example, a conservative choice is
\begin{equation*}
\begin{split}
    \lambda_{\min}(\bm{K}_P) &> - \lambda_{\min}(\jac{\kv_{a}}{\thetav_{a}}) + \lambda_{\max}(\jac{\gv_{a}}{\thetav_{a}})\\
    &\quad\quad+ \frac{\sigma_{\max}^{2}(\jac{\gv_{a}}{\thetav_{u}} +\jac{\kv_{a}}{\thetav_{u}})}{\lambda_{\min}(\jac{\gv_{u}}{\thetav_{u}} + \jac{\kv_{u}}{\thetav_{u}})}. 
\end{split}
\end{equation*}
\end{proof}
When the robot is elastically dominated, i.e.,~\eqref{eq:PD with feedforward:schur 1} holds, the set of equilibria depends only on $\thetav_{a}$. Therefore, ensuring the convergence of $\thetav_{a}$ to $\thetav_{a_{d}}$ also guarantees the convergence of $\thetav_{u}$ to a unique point.

\begin{remark}
    The above uniqueness property holds not only for the PD but also for the other regulators discussed later, provided that $\bm{K}_P$ is sufficiently large. However, the lower bound on $\bm{K}_P$ depends on the specific control law and the cancellations it performs. Generally, the more cancellations, the smaller $\bm{K}_P$ needs to be to guarantee the convexity of $\hess{\mathcal{U}_{cl}}{\thetav}$.
\end{remark}

\subsection{PD with feedforward}
The simplest modification on the PD to ensure that $\thetav_{a_{d}}$ solves the statics involves adding the value of the generalized potential forces at $\thetav_{a_{d}}$ to the control action. A regulator of this type is called a PD controller with feedforward, and it has been extensively studied in the robotics literature~\cite{kuntze85, santibanez2001pd, deluca2000feedforward}. Its structure is
\begin{equation}\label{eq:PD with feedforward}
    \uv = \bm{K}_P\tthetav_{a} - \bm{K}_D\dthetav_{a} + \gv_{a}(\thetav_{d}) + \kv_{a}(\thetav_{d}),
\end{equation}
where $\thetav_{d} = (\thetav_{a_{d}}^{T} \,\, \thetav_{u_{d}}^{T})^{T}$ is a solution to the equilibrium equation
\begin{equation}\label{eq:PD with feedforward:equilibrium}
\begin{carray}{c}
    \gv_{a}(\thetav) + \kv_{a}(\thetav)\\
    \gv_{u}(\thetav) + \kv_{u}(\thetav)
\end{carray}
    = \begin{carray}{c}
        \bm{K}_P\tthetav_{a} + \gv_{a}(\thetav_{d}) + \kv_{a}(\thetav_{d})\\
        \zerov_{n-m}
    \end{carray}.
\end{equation}

Eq.~\eqref{eq:PD with feedforward} exhibits properties similar to the simple PD controller in terms of trajectories boundedness. However, it also guarantees local asymptotic stability of $\thetav_{d}$ and global asymptotic stability when the robot is elastically dominated (see Section~\ref{sec:shape_regulation:global controllers}). This controller only requires measuring $\bm{\theta}_a$, making it easy to implement without additional sensors or observers to estimate $\bm{\theta}_u$.

\begin{lemma}\label{lemma:PD with feedfoward stability}
    For all $\bm{K}_P > 0$, the trajectories of the closed-loop system~\eqref{eq:dynamics collocated splitted}--\eqref{eq:PD with feedforward} remain bounded and converge asymptotically to $(\thetav, \,\, \dthetav) = (\thetav_{eq}, \,\, \zerov_{n})$, where $\thetav_{eq}$ is a solution to~\eqref{eq:PD with feedforward:equilibrium}. In addition, for all $\thetav_{a_{d}} \in \R^{m}$, $\thetav_{d}$ is locally asymptotically stable.
\end{lemma}

\begin{proof}
    Define the Lyapunov-like function
    \begin{equation*}
        V = \mathcal{H} + \frac{1}{2}\tthetav_{a}^{T}\bm{K}_P\tthetav_{a} + \tthetav_{a}^{T}\rb{ \gv_{a}(\thetav_{d}) + \kv_{a}(\thetav_{d}) } + \nu,
    \end{equation*}
    where
    $$\nu = \displaystyle \frac{\norm{ \gv_{a}(\thetav_{d}) + \kv_{a}(\thetav_{d})}^{2}}{2 \lambda_{\min}(\bm{K}_P)}.$$
    This function is non-negative and radially unbounded. Moreover, after some calculations, the time derivative of $V$ reads
    \begin{equation*}
        \dV = - \dthetav^{T}\Dm\dthetav - \dthetav_{a}^{T}\bm{K}_D\dthetav_{a}.
    \end{equation*}
    By LaSalle Invariance Principle, the closed-loop trajectories are bounded and converge to $(\thetav, \,\, \dthetav) = (\thetav_{eq}, \,\, \zerov_{n})$, where $\thetav_{eq}$ satisfies the statics~\eqref{eq:PD with feedforward:equilibrium}. Since $\thetav_{d}$ solves~\eqref{eq:PD with feedforward:equilibrium}, it is a locally asymptotically stable equilibrium. 
\end{proof}

\subsection{Regulators with local asymptotic stability properties}\label{sec:regulators with local asymptotic stability}
The concept of adding the potential forces at $\thetav_{a_{d}}$ to ensure that $\thetav_{a_{d}}$ is an actuated equilibrium can be generalized by introducing a generic model-based term that guarantees such property, which is clearly necessary to address Problem~\ref{problem:control problem}. 

Consider the class of PD regulators defined by
\begin{equation}\label{eq:PD with actuation feedback:control law}
    \uv = \bm{K}_P\tthetav_{a} - \bm{K}_D\dthetav_{a} + \uv_{m}(\thetav, \thetav_{d}),
\end{equation}
where $\uv_{m}$ is a model-based term satisfying
\begin{equation}\label{eq:PD with actuation feedback:constraint}
    \uv_{m}(\thetav_{d}, \thetav_{d}) - \gv_{a}(\thetav_{d}) - \kv_{a}(\thetav_{d}) = \zerov_{m},
\end{equation}
so that $\thetav_{a_{d}}$ is an equilibrium of the closed-loop system. 

\begin{theorem}\label{theorem:PD with actuation feedback}
Suppose that
\begin{equation}\label{eq:PD with actuation feedback:assumption}
    \jac{\gv_{u}}{\thetav_{u}}(\thetav_{d}) + \jac{\kv_{u}}{\thetav_{u}}(\thetav_{d}) > 0. 
\end{equation}
If $\bm{K}_P$ is sufficiently large, then the equilibrium configuration $\thetav = \thetav_{d}$ is locally asymptotically stable.
\end{theorem}

\begin{proof}
   The result can be shown by performing a stability analysis on the linearized closed-loop equations~\eqref{eq:dynamics collocated splitted}--\eqref{eq:PD with actuation feedback:control law}. The details of the proof are provided in Appendix \ref{Sec:Appendix 1}.
\end{proof}

Condition~\eqref{eq:PD with actuation feedback:assumption} represents a local version of the requirement for elastically dominated robots. Although it is less restrictive than (\ref{eq:PD with feedforward:schur 1}), it provides only a sufficient condition for closed-loop stability under~\eqref{eq:PD with actuation feedback:control law}. Less restrictive conditions might apply. Additionally, Theorem~\ref{theorem:PD with actuation feedback} does not specify the region of convergence. In fact, stronger results will be derived for a subclass of~\eqref{eq:PD with actuation feedback:control law}, which allows the removal of~\eqref{eq:PD with actuation feedback:assumption}.

In principle, an infinite number of unactuated variables is needed to describe the true dynamics of a soft robot. Therefore, it is important to understand how the closed-loop system behaves when the controller uses only a subset of $\thetav_{u}$. To this end, we can conveniently partition $\thetav_{u}$ into modeled and unmodeled unactuated DoF
\begin{equation*}
    \thetav_{u} = \begin{carray}{c}
        \thetav_{u_{m}}\\
        \thetav_{u_{u}}
    \end{carray}.
\end{equation*}
When only $\thetav_{u_{m}}$ is used, the controller becomes
\begin{equation}\label{eq:PD with actuation feedback:control law:unactuated modeled}
    \uv = \bm{K}_P\tthetav_{a} - \bm{K}_D\dthetav_{a} + \uv_{m}(\thetav_{a}, \thetav_{u_{m}},  \zerov, \thetav_{a_{d}}, \thetav_{u_{m_{d}}}, \zerov).
\end{equation} 
Clearly, under~\eqref{eq:PD with actuation feedback:control law:unactuated modeled} condition~\eqref{eq:PD with actuation feedback:constraint} does not generally hold unless $\thetav_{u_{u_{d}}} = \zerov$. However, suppose that $\thetav_{eq}$ is an equilibrium of the closed-loop and that
\begin{equation*}
    \jac{\gv_{u}}{\thetav_{u}}(\thetav_{eq}) + \jac{\kv_{u}}{\thetav_{u}}(\thetav_{eq}) > 0.
\end{equation*}
Then, for a large enough $\bm{K}_P$, $\thetav_{eq}$ is locally asymptotically stable because similar arguments used in Theorem~\ref{theorem:PD with actuation feedback} apply.

From these considerations, two conclusions can be drawn.
\begin{enumerate}[label=(\roman*)]
    \item The stiffer the robot, the smaller $\bm{K}_P$ can be to guarantee stability.
    \item Elastically dominated soft robots are inherently more robust to unmodeled dynamics, as any feedback controller of the form~\eqref{eq:PD with actuation feedback:control law} will not destabilize the system even with unmodeled DoF, provided that $\bm{K}_P$ is sufficiently large.
\end{enumerate}

\subsection{Regulators with semi-global stability properties}\label{sec:semi-global stability}
By adding constraints on the structure of $\uv_{m}$, one can define a class of regulators that exhibits stronger stability traits. The control action still features a PD and a model-based component but it also includes the possibility of adding a saturated integral term to enhance robustness against constant disturbances. We show that the region of convergence can be made arbitrarily large by appropriately tuning the controller gains and hyperparameters, thereby ensuring semi-global stability.

To achieve this, introduce the following class of saturation functions, inspired by~\cite{kelly1998global}.

\begin{definition}\label{def:saturation}
Let $\setS$ be the set of $C^{1}$ monotonically increasing functions $\sv(\xv) = \left( s(x_{1}) \,\, s(x_{2}) \cdots s(x_{h}) \right)^{T} : \mathbb{R}^{h} \rightarrow \mathbb{R}^{h}$ such that, for all $y \in \mathbb{R}$,
\begin{enumerate}[label=(\roman*)]\label{enumerate:sat}
    \item $\small s(y) = \sigma(y)y$ holds for a function $\sigma : \mathbb{R} \rightarrow \mathbb{R}$ defined everywhere,
    \item \vspace{2pt} $\small \abs{s(y)} \geq \begin{cases}
        \alpha_{1} \abs{y};\,\,\, \abs{y} \leq \beta\\
        \,\,\,\, \alpha_{1}\beta \,;\,\,\, \abs{y} > \beta
    \end{cases}$,
    \item \vspace{2pt} $\small \abs{s(y)} \leq \begin{cases}
        \abs{y}; \,\,\, \abs{y} \leq \beta\\
        \,\,\, \beta \,;\,\,\, \abs{y} > \beta
    \end{cases}$,
    \item \vspace{3pt} $\small \alpha_{2} \geq \displaystyle \frac{d s(y)}{d y} \geq 0$,
    \item \vspace{2pt} $\small \abs{ 1 - \displaystyle \frac{d s(y)}{d y}} \leq \alpha_{3} \abs{ s(y) }$.
\end{enumerate}
\end{definition}

\begin{example}  
The saturation function
$
\sv_{t}(\xv) = \left( s_{t}(x_{1}) \cdots s_{t}(x_{h}) \right)$, where $s_{t}(y) = \text{tanh}(y)$, belongs to $\mathcal{S}(1, \text{tanh}(1), 1, 1)$.
Another possible choice for $\sv$ is 
$\small
\sv_{p}(\xv) = \left( s_{p}(x_{1}) \cdots s_{p}(x_{h})\right)$, where $s_{p}(y) =  \frac{y}{\left(1 + |y|^{p}\right)^{\frac{1}{p}}}$ and $p \in \mathbb{Z}_+$, which belongs to $\small \mathcal{S}\left(1, {2^{-\frac{1}{p}}}, 1, 1\right)$.
\end{example}

The set $\mathcal{S}$ satisfies several useful properties. Specifically, for all $\xv \in \mathbb{R}^{h}$ and $\Qm \in \mathbb{R}^{h \times h}$ symmetric, we have the following.

Consider now a regulator of the form
\begin{equation}\label{eq:regulator class}
    \uv = \bm{K}_P\tthetav_{a} - \bm{K}_D\dthetav_{a} + \bm{K}_I\int_{0}^{t}\sv(\tthetav_{a}(\tau))\drm\tau + \uv_{m}(\thetav, \thetav_{a_{d}}),
\end{equation}
where $\bm{K}_P > 0$, $\bm{K}_D \geq 0$, $\bm{K}_I \geq 0$, and $\uv_{m}$ is a model-based term such that
\begin{enumerate}[label=(\roman*)]
    \item $\norm{\uv_{m}} \geq k_{u_{1}}\norm{\thetav} + k_{u_{2}} \norm{\thetav_{a_{d}}}$,
    \item $\norm{ \uv_{m} - \gv_{a}(\thetav) - \kv_{a}(\thetav) } \leq k_{u_{3}}\norm{\tthetav_{a}}$,
    \item $\norm{ \jac{ \uv_{m} }{\thetav} } \leq k_{u_{4}}$,
\end{enumerate}
for some constants $k_{u_{1}}, k_{u_{2}}, k_{u_{3}}$, and $k_{u_{4}} \in \R^{+}$. Possible choices for $\uv_{m}$ are provided in Table~\ref{tab:model based regulator terms}. These controllers require measurements of both $\bm{\theta}_a$ and $\bm{\theta}_u$, necessitating additional sensors or observers to estimate the robot’s full state.

\begin{table*}[t!]
\centering
\caption{Possible model-based terms $\uv_{m}$ for~\eqref{eq:regulator class} guaranteeing stability according to Theorem~\ref{theorem:regulator}.}
\begin{tabular}{ll}
\hline
Name                                          & Expression of $\uv_{m}$        \\ \hline
Cancellation                           & $\gv_{a}(\thetav_{a}, \thetav_{u}) + \kv_{a}(\thetav_{a}, \thetav_{u})$ \\
Gravity cancellation and collocated elasticity compensation & $\gv_{a}(\thetav_{a}, \thetav_{u}) + \kv_{a}(\thetav_{a_{d}}, \thetav_{u})$ \\
Collocated gravity compensation and elasticity cancellation & $\gv_{a}(\thetav_{a_{d}}, \thetav_{u}) + \kv_{a}(\thetav_{a}, \thetav_{u})$ \\
Collocated gravity and elasticity compensation & $\gv_{a}(\thetav_{a_{d}}, \thetav_{u}) + \kv_{a}(\thetav_{a_{d}}, \thetav_{u})$ \\
\hline
\end{tabular}%
\label{tab:model based regulator terms}
\end{table*}

\begin{theorem}\label{theorem:regulator}
    For $\bm{K}_P$ sufficiently large, there exists a set $\Omega_{0}$ of initial conditions such that the trajectories of the closed-loop system~\eqref{eq:dynamics collocated splitted}--\eqref{eq:regulator class} remain bounded and converge to the set
    \begin{equation}\label{eq:regulator:stable set}
        \{ (\thetav, \,\, \dthetav) \vert \thetav_{a} = \thetav_{a_{d}}, \thetav_{u} = \thetav_{u_{d}} ,  \dthetav = \zerov_{n}\},
    \end{equation}
    where $\thetav_{u_{d}}$ is a solution to
    \begin{equation*}
        \gv_{u}(\thetav_{a_{d}}, \thetav_{u}) + \kv_{u}(\thetav_{a_{d}}, \thetav_{u}) = \zerov_{n-m}. 
    \end{equation*}
    Moreover, $\Omega_{0}$ can be made arbitrarily large by increasing $\bm{K}_P$ and $\beta$.
\end{theorem}

\begin{proof}
The proof of this theorem is provided in Appendix \ref{Sec:Appendix 2}. 
\end{proof}
The previous theorem and its proof allows to draw the following conclusions.

The saturated integral can be removed without affecting theoretical convergence. However, at the cost of requiring larger proportional gains, its inclusion enhances robustness to constant disturbances. In fact, by assuming \(\kv_I > 0\), it is possible to add a disturbance \(\dv\) with \(\frac{\drm \dv}{\drm t} = \zerov_m\) to the actuated dynamics and perform a similar proof by defining the integrator state variables as
\[
    \zv(t) = \int_{0}^{t}\sv(\tthetav_a(\tau))\drm \tau + \kv_I^{-1} \dv.
\]

The region of convergence can be arbitrarily enlarged by increasing \(\kv_P\) and expanding the non-saturation region of the integral, making the class of controllers guarantee semi-global stability~\cite{isidori1999nonlinearII}. When \(\kv_I = \zerov_{m \times m}\), the saturation is instrumental only to show the result but does not appear in the controller, thus reducing the lower bounds on \(\kv_P\). This reflects the well-known fact that an integral action has a tendency to reduce stability. Moreover, in such cases, the region of convergence depends solely on \(\kv_P\), and increasing \(\kv_P\) enlarges the set \(\Omega_0\).

Four possible expressions for \(\uv_m\) are presented in Table~\ref{tab:model based regulator terms} and yield eight controllers, depending on whether the integral action is employed. In all cases, feedback from the unactuated variables is used to eliminate the dependence of \(\norm{\uv_m(\thetav, \ \thetav_{a_d}) - \gv_a(\thetav) - \kv_a(\thetav)}\) on $\thetav_u$. This dependence can be a limiting factor, as obtaining measurements of the unactuated variables is challenging and currently requires exteroceptive sensors. Furthermore, not all DoF can be measured since they are theoretically infinite. Whether controllers relying on feedforward terms in \(\thetav_{u_d}\) yield stronger results than local asymptotic stability (see Section~\ref{sec:regulators with local asymptotic stability}) remains an open question.

\subsection{Regulators with global stability properties}\label{sec:shape_regulation:global controllers}
We finally present control laws guaranteeing global convergence of $\thetav_{a}$ to $\thetav_{a_{d}}$. Enlarging the convergence region comes with the requirement of imposing further restrictions on the controller structure or the class of robots under consideration.

\begin{theorem}\label{theorem:shape_regulation:regulator 2}
    Consider a regulator of the form~\eqref{eq:regulator class}, where the model-based term satisfies the additional condition
    \begin{enumerate}[label=(\roman*)]
        \setcounter{enumi}{3}
        \item $\norm{\uv_{m}} \leq k_{u_{5}}$,
    \end{enumerate}
    for some constant $k_{u_{5}} \in \R^{+}$.
    For all $\thetav_{a_{d}} \in \R^{m}$ and a sufficiently large $\bm{K}_P$, the trajectories of the closed-loop system are bounded and converge asymptotically to the equilibrium state $(\thetav, \,\, \dthetav) = (\thetav_{a_{d}}, \,\, \zerov_{n})$, where $\thetav_{u_{d}}$ is a possible solution to
    \begin{equation}\label{eq:shape_regulation:PD 1 elastically decoupled:unactuated equilibrium}
        \gv_{u}(\thetav_{a_{d}}, \thetav_{u}) + \kv_{u}(\thetav_{a_{d}}, \thetav_{u}) = \zerov_{n-m}.
    \end{equation}
\end{theorem}

\begin{proof}
The proof of this theorem is provided in Appendix \ref{Sec:Appendix 3}.    
\end{proof}

Similar considerations as those drawn for the controller in Section~\ref{sec:semi-global stability} apply. The remainder of this section presents global stability results for elastically decoupled and dominated robots.

\subsection*{Elastically decoupled soft robots}
Theorem~\ref{theorem:shape_regulation:regulator 2} can be applied to regulate the shape of elastically decoupled robots.

\begin{corollary}\label{theorem:shape_regulation:PD 1 elastically decoupled}
    For an elastically decoupled soft robot (Class~\ref{class:elastically decoupled}), a controller~\eqref{eq:regulator class} with model-based term
    \begin{equation}\label{eq:shape_regulation:PD 1 elastically decoupled}
        \uv_{m} = \gv_{a}(\thetav) + \kv(\thetav_{a_{d}}),
    \end{equation}
    or
    \begin{equation}\label{eq:shape_regulation:PD 2 elastically decoupled}
        \uv_{m} = \gv_{a}(\thetav_{a_{d}}, \thetav_{u}) + \kv(\thetav_{a_{d}}),
    \end{equation}
    guarantees that, for all $\thetav_{a_{d}} \in \R^{m}$ and a sufficiently large $\bm{K}_P$, the trajectories of the closed-loop system are bounded and converge to an equilibrium configuration where $\thetav_{a} = \thetav_{a_{d}}$.
\end{corollary}

\begin{proof}
    The model-based term $\uv_{m}$, as defined in~\eqref{eq:shape_regulation:PD 1 elastically decoupled} and~\eqref{eq:shape_regulation:PD 2 elastically decoupled}, is bounded in norm. This is due to the fact that the compensation for the elastic forces, $\kv_{a}(\thetav_{a_{d}})$, is constant, and $\gv$ remains bounded from Property~\ref{property:gravity}. Thus, the conditions of Theorem~\ref{theorem:shape_regulation:regulator 2} are satisfied.
\end{proof}

\begin{figure}[t]
    \centering
    \includegraphics[width=0.9\columnwidth]{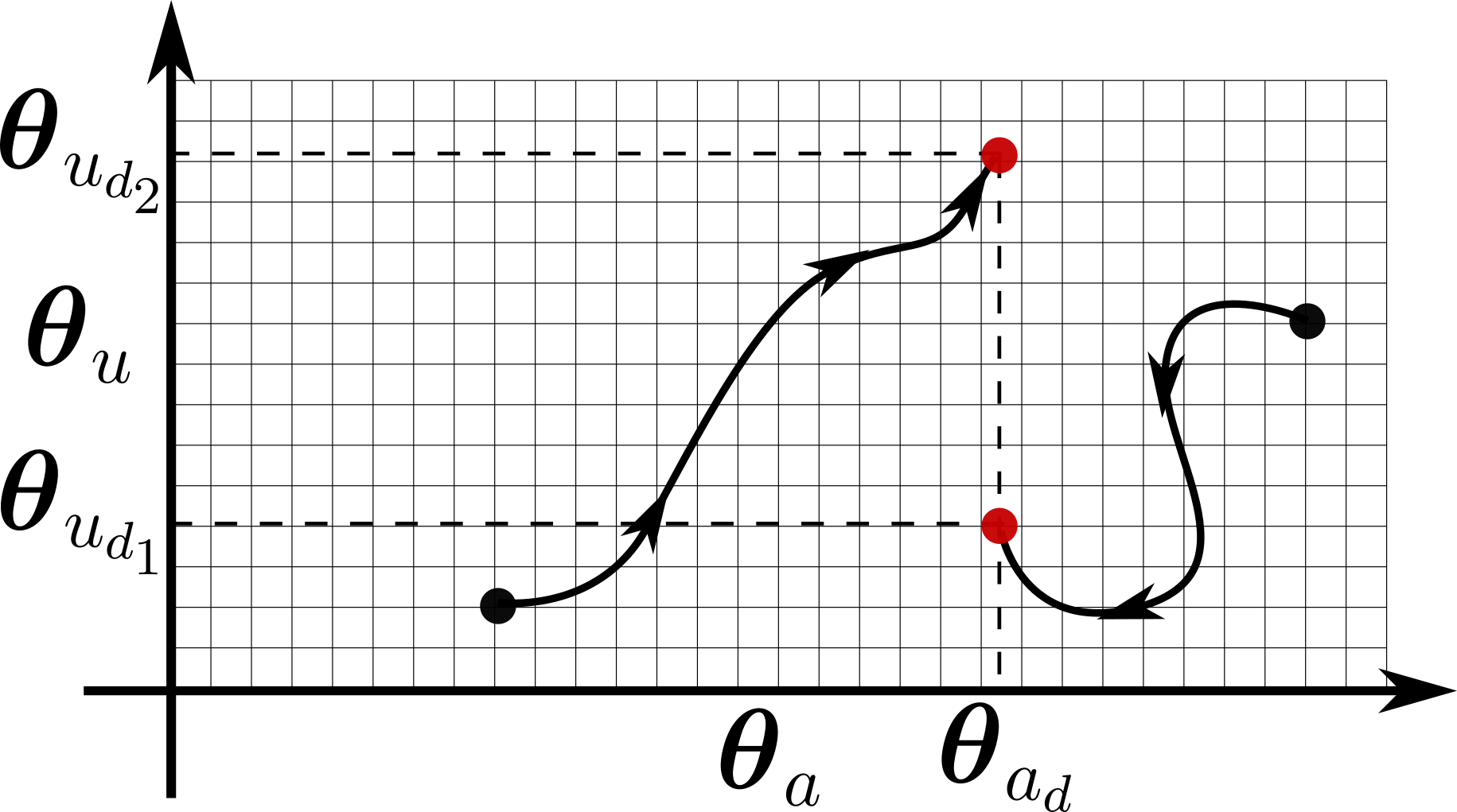}
    \caption{Possible unactuated equilibria $\thetav_{u_{d_{1}}}$ and $\thetav_{u_{d_{2}}}$ compatible with the desired $\thetav_{a_{d}}$ value of the {actuation coordinates}.}
    \label{shape_regulation:non uniqueness equilibria}
\end{figure}

\begin{table*}[b]
\centering
\renewcommand{\arraystretch}{1.35}
\begin{tabular}{
|>{\centering\arraybackslash}p{2.0cm}
|>{\centering\arraybackslash}p{2.0cm}
|>{\centering\arraybackslash}p{1.5cm}
|>{\centering\arraybackslash}p{2.0cm}
|>{\centering\arraybackslash}p{2.4cm}
|>{\centering\arraybackslash}p{2.7cm}
|>{\centering\arraybackslash}p{2.0cm}|
}
\hline
\textbf{FAMILY} &
\textbf{MEASURES} &
\textbf{MODEL} &
\textbf{INTEGRAL} &
\textbf{CLASS} &
\textbf{GUARANTEE} &
\textbf{THEOREM} \\
\hline
PD &
$\bm{\theta}_a$ &
none &
--- &
--- &
bounded / practical &
--- \\
\hline
PD + feedforward &
$\bm{\theta}_a$ &
$\bm{F}_{c_d}$ &
--- &
under class &
local / global$^{*}$ &
T1 \\
\hline
PD / PID + cancellation &
full $\bm{\theta}$ &
full potential forces &
optional &
under class &
semi-global$^{*}$ &
T2--3 \\
\hline

P-satI-D &
$\bm{\theta}_a$ &
none &
\begin{tabular}[c]{@{}c@{}}
saturated\\
integral
\end{tabular} &
\begin{tabular}[c]{@{}c@{}}
elastically\\
dominated
\end{tabular} &
global$^{*}$ &
T4 \\
\hline
\end{tabular}
\caption{Summary of controller families and their guarantees. Select the row that matches the available measurements and model knowledge, then read across to identify the applicable controller structure and its corresponding performance guarantee. $^\ast$ Guarantee holds under the stated class assumption; (Sec. \ref{sec:shape regulators}). }
\label{tab:controller-families}
\end{table*}

The control law obtained with~\eqref{eq:shape_regulation:PD 2 elastically decoupled}, compared to the one with~\eqref{eq:shape_regulation:PD 1 elastically decoupled}, relies on feedback from the unactuated variables only. We can expect that $\uv_{m}$ should depend on feedback from $\thetav_{u}$ to ensure a global result, because, in general, the equilibrium equations for the unactuated variables might not have a unique solution for a given $\thetav_{a_{d}}$. If two unactuated equilibria are associated with $\thetav_{a_{d}}$, there are two possible closed-loop equilibria compatible with $\thetav_{a_{d}}$. This situation is illustrated graphically in Fig.~\ref{shape_regulation:non uniqueness equilibria}. Therefore, to guarantee convergence independently of $\thetav_{u}$, one has to cancel the effect of $\thetav_{u}$ through feedback.

\subsection*{Elastically dominated soft robots}
Here, we present two simpler controllers for elastically dominated soft robots that allow solving Problem~\ref{problem:control problem} globally, thereby demonstrating that this class of robots is easier to control. Remarkably, both controllers rely on feedback from $\thetav_{a}$ only.

\begin{lemma}
    Consider the closed-loop system~\eqref{eq:dynamics collocated splitted}--\eqref{eq:PD with feedforward}. If the robot is elastically dominated (Class~\ref{class:elastically dominated}) and $\bm{K}_P$ is sufficiently large, then $\thetav_{d}$ is globally asymptotically stable.
\end{lemma}
\begin{proof}
    The proof follows from Lemma~\ref{lemma:PD:uniqueness of equilibria} and Lemma~\ref{lemma:PD with feedfoward stability}.
\end{proof}

The {actuation coordinates} can also be controlled starting from an arbitrary initial condition without relying on system knowledge through a P-satI-D
\begin{equation}\small\label{eq:dominant_elasticity:PID}
    \uv = \bm{K}_P\tthetav_{a} - \bm{K}_D\dthetav_{a} + \int_{0}^{t}\sv(\tthetav_{a}(\tau))\drm \tau.
\end{equation}

This controller only requires measuring $\bm{\theta}_a$, making it easy to implement without additional sensors or observers to estimate $\bm{\theta}_u$.

\begin{theorem}\label{theorem:shape_regulation:elastically dominated:PID}
    Suppose the robot is elastically dominated (Class~\ref{class:elastically dominated}). If $\bm{K}_P$ is large enough, the trajectories of the closed-loop system \eqref{eq:dynamics collocated splitted}--\eqref{eq:dominant_elasticity:PID} are bounded and converge asymptotically to $(\thetav, \dthetav) = (\thetav_{d}, \zerov_{n})$.
\end{theorem}

\begin{proof}
The proof of this theorem is provided in Appendix \ref{Sec:Appendix 3}.   
\end{proof}

Table \ref{tab:controller-families} provides a clear overview of the available controllers and their associated guarantees. Users can identify the row that best matches the available measurements and level of model knowledge, and then read across the table to determine the appropriate controller structure and its corresponding performance guarantee. 

\section{Experimental Results}\label{sec:experiments}
This section presents an experimental validation of the majority of the discussed regulators, as well as three additional controllers for which formal stability proofs are not yet available. These additional control laws are noteworthy because they operate without requiring feedback from the unactuated DoF. Table~\ref{tab:regulators} summarizes the controllers tested. In addition to validating stability, to the best of our knowledge, the experiments provide the first systematic performance comparison across different discretizations of the same soft robot.
\begin{table*}[t]
\centering
\caption{Regulators considered in the experiments. For each controller, we report if some form of stability has been proven.}
\begin{tabular}{cc}
\hline
Controller & Provably stable \\ \hline
$\uv_{1} = \bm{K}_P\tthetav_{a} - \bm{K}_D\dthetav_{a}$ & yes \\
$\uv_{2} = \bm{K}_P\tthetav_{a} - \bm{K}_D\dthetav_{a} + \gv_{a}(\thetav_{d}) + \kv_{a}(\thetav_{d})$ & yes \\
$\uv_{3} = \bm{K}_P\tthetav_{a} - \bm{K}_D\dthetav_{a} + \gv_{a}(\thetav) + \kv_{a}(\thetav_{d})$ & yes \\
$\uv_{4} = \bm{K}_P\tthetav_{a} - \bm{K}_D\dthetav_{a} + \gv_{a}(\thetav) + \kv_{a}(\thetav)$ & yes \\
$\uv_{5} = \bm{K}_P\tthetav_{a} - \bm{K}_D\dthetav_{a} + \gv_{a}(\thetav_{a}, \thetav_{u_{d}}) + \kv_{a}(\thetav_{d})$ & no \\
$\uv_{6} = \bm{K}_P\tthetav_{a} - \bm{K}_D\dthetav_{a} + \gv_{a}(\thetav_{a_{d}}, \thetav_{u}) + \kv_{a}(\thetav_{d})$ & yes \\
$\uv_{7} = \displaystyle \bm{K}_P\tthetav_{a} - \bm{K}_D\dthetav_{a} + \bm{K}_I\int_{0}^{t}\tthetav_{a}(\tau)\drm \tau$ & no \\
$\uv_{8} = \displaystyle \bm{K}_P\tthetav_{a} - \bm{K}_D\dthetav_{a} + \bm{K}_I\int_{0}^{t}\sv(\tthetav_{a}(\tau))\drm \tau$ & yes \\
$\uv_{9} = \displaystyle \bm{K}_P\tthetav_{a} - \bm{K}_D\dthetav_{a} + \bm{K}_I\int_{0}^{t}\sv(\tthetav_{a}(\tau))\drm \tau + \gv_{a}(\thetav_{d}) + \kv_{a}(\thetav_{d})$ & yes \\
$\uv_{10} = \displaystyle \bm{K}_P\tthetav_{a} - \bm{K}_D\dthetav_{a} + \bm{K}_I\int_{0}^{t}\sv(\tthetav_{a}(\tau))\drm \tau + \gv_{a}(\thetav) + \kv_{a}(\thetav_{d})$ & yes \\
$\uv_{11} = \displaystyle \bm{K}_P\tthetav_{a} - \bm{K}_D\dthetav_{a} + \bm{K}_I\int_{0}^{t}\sv(\tthetav_{a}(\tau))\drm \tau + \gv_{a}(\thetav) + \kv_{a}(\thetav)$ & yes \\
$\uv_{12} = \displaystyle \bm{K}_P\tthetav_{a} - \bm{K}_D\dthetav_{a} + \bm{K}_I\int_{0}^{t}\sv(\tthetav_{a}(\tau))\drm \tau + \gv_{a}(\thetav_{a}, \thetav_{u_{d}}) + \kv_{a}(\thetav_{d})$ & no \\
$\uv_{13} = \displaystyle \bm{K}_P\tthetav_{a} - \bm{K}_D\dthetav_{a} + \bm{K}_I\int_{0}^{t}\sv(\tthetav_{a}(\tau))\drm \tau + \gv_{a}(\thetav_{a_{d}}, \thetav_{u}) + \kv_{a}(\thetav_{d})$ & yes \\
\hline
\end{tabular}
\label{tab:regulators}
\end{table*}
\subsection{Setup}
\begin{figure}[h]
    \centering
    \includegraphics[width=0.98\columnwidth]{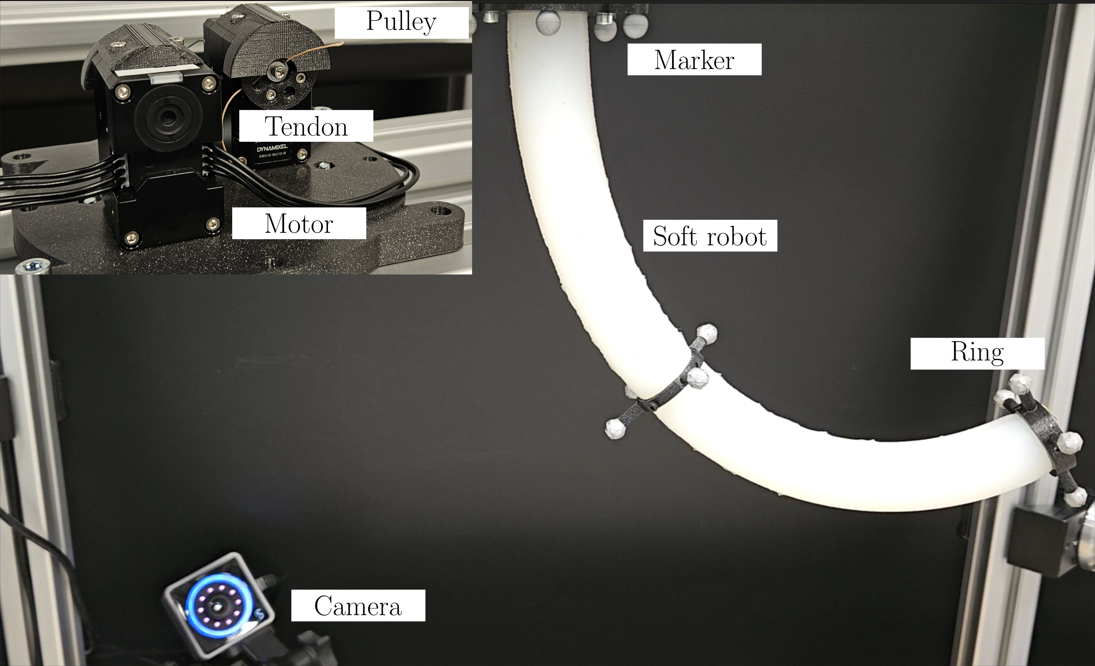}
    \caption{The experimental setup consists of a slender continuum soft arm of Eco-Flex 20 silicone actuated by a pair of tendons tensioned by Dynamixel XM430-W210-T motors. Eight Optitrack Prime\textsuperscript{x} 13 cameras measure the position and orientation of three rings of markers placed along the robot body.}
    \label{experimental setup}
\end{figure}
With reference to Fig.~\ref{experimental setup}, the robot consists of a slender body fabricated from Eco-Flex 20 silicone. The arm is $\SI{0.3}{[\meter]}$ long with a radius of $\SI{0.025}{[\meter]}$, and it is actuated by two tendons spaced $\SI{180}{\degree}$ apart, running from the base to the tip at a constant distance of $\SI{0.02}{[\meter]}$ from the backbone. As a result, the robot motion is restricted to a planar workspace. One end of each tendon is fixed to a Dynamixel XM430-W210-T motor via a pulley of radius $\SI{0.015}{[\meter]}$, while the other end is anchored to the arm tip through a thin 3D-printed disk. Silicone hoses with an external radius of $\SI{0.005}{[\meter]}$ are embedded in the body to allow the tendons to slide with reduced friction. The robot body is cast using a 3D-printed mold.

The robot configuration is estimated using eight Optitrack Prime\textsuperscript{x} 13 cameras, which capture the positions of three rigid rings with reflective markers located at the base, middle, and tip of the robot. A Dell workstation is used to process data from the motion capture system and to compute the control action in Simulink, which runs with a real-time kernel operating at $80~[\si{\hertz}]$. Although this control frequency is suboptimal for a configuration space controller, it is limited by the operating frequencies of the actuation and measurement systems.

\subsection{Kinematic model and state estimation}
\begin{figure*}[t]
    \centering
     \includegraphics[width=1\textwidth,  trim={0cm, 0cm, 0.0cm, 0}, clip]{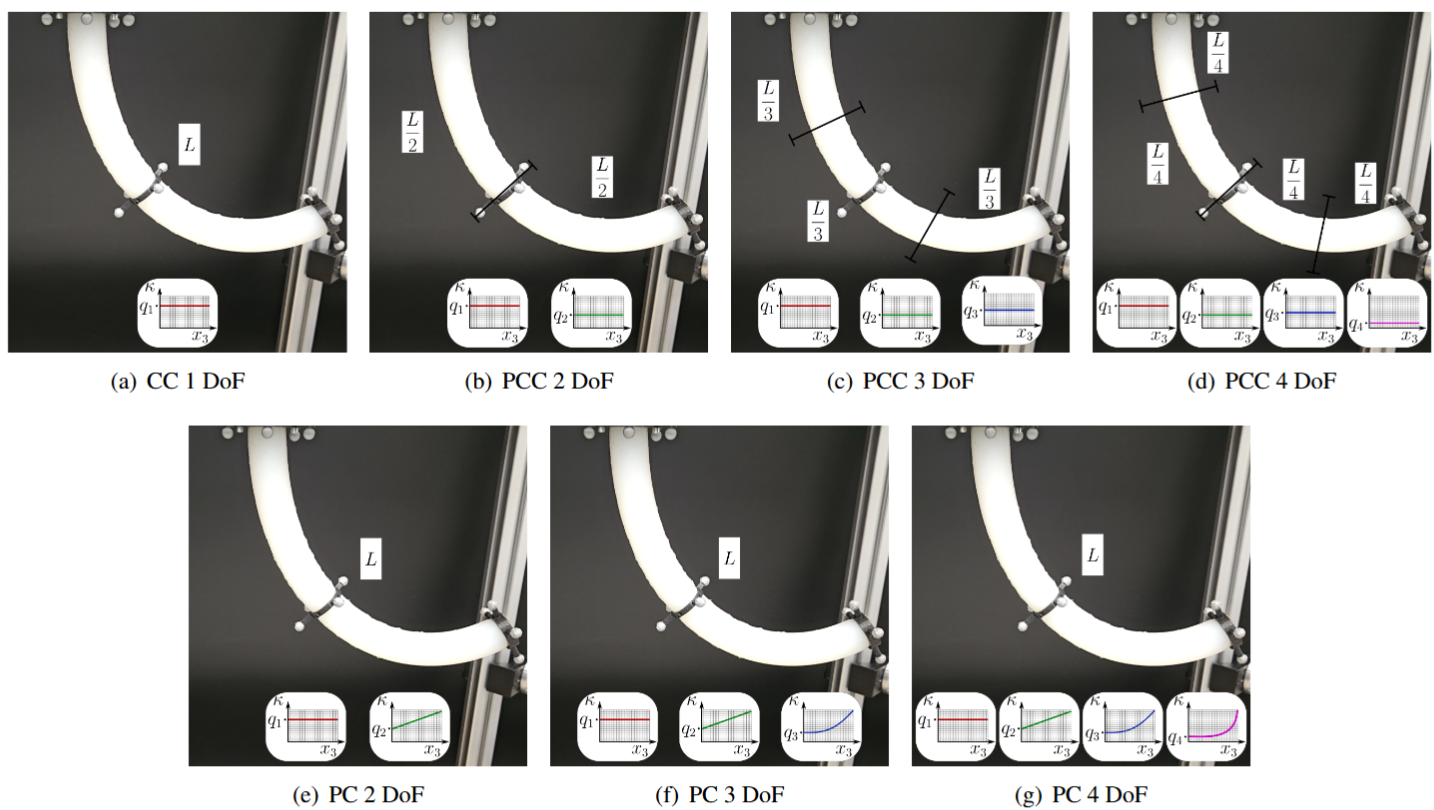}
    \caption{Robot kinematic models considered in the experiments.}
    \label{experiment kinematic models}
\end{figure*}

The discretization of the robot is performed across two conceptual layers. The first layer pertains to the number of bodies into which the arm is discretized, while the second concerns the DoF assigned to each body. The robot is discretized into up to four bodies, each modeled as a Cosserat rod~\cite{boyer2020dynamics} with bending strain. Elongation is neglected to reduces the number of rings required to reconstruct the configuration. This way, the information that would have been used to estimate elongation is instead utilized to measure an additional curvature mode.

We consider seven possible models for curvature, described either under the Constant Curvature (CC) or Polynomial Curvature (PC) framework, as illustrated in Fig.~\ref{experiment kinematic models}. Under the CC assumption, we analyze four models, ranging from one to four bodies. Instead, the PC discretization is done always by considering a single body and by using polynomial basis to augment the number of DoF. Table~\ref{tab:experiment:curvature} presents the functional expressions of the curvature for all models.

\begin{table}[]
\centering
\caption{Curvature models used in the experiments. The variable $s$ denotes the curvilinear abscissa.}
\footnotesize
\begin{tabular}{cc}
\hline
\multicolumn{1}{c}{Model} & Curvature                \\ \hline
\vspace{1pt}
CC 1 DoF                 & $\small\kappa(s, q) = \frac{1}{0.3}\,q_{1}$
\\[1pt]
PCC 2 DoF                & $\small\kappa(s, \qv) = \frac{1}{0.15} \left\{ \begin{array}{l}
     q_{1}; \quad s \in [0, 0.15)\\
     q_{2}; \quad s \in [0.15, 0.3]
\end{array}\right.$
\\[1pt]
PCC 3 DoF                & $\small\kappa(s, \qv) = \frac{1}{0.1} \left\{ \begin{array}{l}
     q_{1}; \quad s \in [0, 0.1)\\
     q_{2}; \quad s \in [0.1, 0.2)\\
     q_{3}; \quad s \in [0.2, 0.3]
\end{array}\right.$
\\[1pt]
PCC 4 DoF                & $\small\kappa(s, \qv) = \frac{1}{0.075} \left\{ \begin{array}{l}
     q_{1}; \quad s \in [0, 0.075)\\
     q_{2}; \quad s \in [0.075, 0.15)\\
     q_{3}; \quad s \in [0.15, 0.225]\\
     q_{4}; \quad s \in [0.225, 0.3]\\
\end{array}\right.$
\\[1pt]
PC 2 DoF                 & $\small\kappa(s, \qv) = \frac{1}{0.3}\,(q_{1} + \frac{s}{0.3} q_{2} )$
\\[1pt]
PC 3 DoF                 & $\small\kappa(s, \qv) = \frac{1}{0.3}\,( q_{1} + \frac{s}{0.3} q_{2} + \rb{\frac{s}{0.3}}^{2} q_{3} )$
\\[1pt]
PC 4 DoF                 & \!\!\!$\small \kappa(s, \qv) = \frac{1}{0.3}\,( q_{1} + \frac{s}{0.3} q_{2} + \rb{\frac{s}{0.3}}^{2} q_{3} + \rb{\frac{s}{0.3}}^{3} q_{4} )$
\\[1pt]
\hline
\end{tabular}
\label{tab:experiment:curvature}
\end{table}

Each tracked ring along the robot body provides the one-dimensional orientation in the plane of motion, and the Cartesian position, which is two-dimensional. This in principle allows to consider a kinematic model with up to six DoF. However, we limit the number of DoF to a maximum of four because, in this way, the inverse kinematics problem becomes over-constrained and because convergence is faster, which is of paramount importance since state estimation is to be done in real-time. The configuration vector $\qv$ is computed by numerically solving the inverse kinematics with a Newton-Raphson method in body coordinates~\cite{lynch2017modern}. Then, $\dqv$ is estimated from $\qv$ using backward finite differences and smoothed to remove noise with a Savitzky–Golay filter.

\subsection{Dynamic model}
We compute the dynamics using the algorithm described in~\cite{pustina2025recursive}, which is also utilized to implement the model-based terms required by the regulators. The mass density of Eco-Flex 20 silicone is $\rho = 1080~[\si{\kilogram \per \cubic \meter}]$. Elastic forces are modeled using Hooke law, with a Young modulus of $3.20[\si{\mega \pascal}]$ and a Poisson ratio of $0.45$, both of which were identified experimentally.. Consequently, the generalized elastic force is linear in $\qv$ and takes the form $\kv(\qv) = \kv \qv$.

The input to the system is modeled as a distributed torque acting on the robot cross-sections. Specifically, let the controlled variable be the torque $u \in \R$ along the arm cross-sections. The generalized actuation force can be derived using the principle of virtual work, balancing the work on the robot with that of the actuation
\begin{equation}\label{experiment:power balance}
    \delta \qv^{T} \tauv = \int_{0}^{0.3} \delta \kappa(s, \qv) u \drm s.
\end{equation}
Recalling the dependence of $\kappa$ on $\qv$ (see Table \ref{tab:experiment:curvature}), we have
\begin{equation*}
    \int_{0}^{0.3} \delta \kappa(s, \qv) u \drm s = \int_{0}^{0.3} \diff{ \kappa(s, \qv) }{\qv} \delta \qv u \drm s.
\end{equation*}
Substituting this expression into~\eqref{experiment:power balance} exploiting that virtual displacements are arbitrary, gives
\begin{equation*}
    \nuv = \underbrace{\int_{0}^{0.3} \left( \diff{ \kappa(s, \qv) }{ \qv } \right)^{T} \drm s}_{\Am(\qv)} u.
\end{equation*}
For both the PCC and PC models, the resulting actuation matrix is independent on $\qv$, implying that the dynamics admits collocated form with actuation coordinate $\theta_{a} = \Am^{T}\qv$. Furthermore, from~\eqref{experiment:power balance}, it can also be verified that $\theta_{a}$ corresponds to the orientation of the robot end-effector, which is the variable directly influenced by the input torque $u$.

The command $u$ produced by the control actions is converted into a tendon tension, which is translated into a desired motor torque and finally into a current--the motor variable that can be directly controlled in the setup. When the controller outputs a negative torque, the motor on the left is used, while a positive torque engages the motor on the right. It is worth observing that the tendon tension could alternatively be modeled directly as the system input. However, this approach would require accounting for the unidirectional nature of tendons, which can only apply pulling forces. The block scheme of the closed-loop system is illustrated in Fig.~\ref{experiments:closed-loop system} in Appendix \ref{sec:block scheme}.

\subsection{Controllers comparison}
We evaluate the controllers in Table~\ref{tab:regulators} with respect to three different variables.
\begin{enumerate}[label=(\roman*)]
    \item Variations in the proportional gain.
    \item Presence or absence of a payload of $0.09~[\si{\kilogram}]$ attached to the robot tip.
    \item Discretization type (see Table~\ref{tab:experiment:curvature}). 
\end{enumerate}

Indicative gain values are obtained through a preliminary identification experiment. See Appendix \ref{sec:identification experiment} for further details on the controller-tuning process. Once the tuning is complete, each controller is evaluated by commanding six step references, always starting from the straight configuration, i.e., $\qv = \zerov_{n}$. Subsequently, a sequence of six consecutive step commands is applied, where each reference starts from the final configuration reached by the previous step rather than from the straight configuration. Each step lasts $5~[\si{\second}]$. The reference values are drawn from uniform distributions over the intervals $[-\pi,\pi]~[\si{\radian}]$ and $\left[-\frac{\pi}{2},\frac{\pi}{2}\right]~[\si{\radian}]$ when the robot operates without and with the payload, respectively.

For regulators comparison, we consider performance metrics related to tracking error and time.

\begin{enumerate}[label=(\roman*)]
    \item Root mean square error (RMSE)
            \begin{equation*}
                \sqrt{\frac{1}{N}\sum_{i=1}^{N} \left( \theta_{a_{i}} - \theta_{a_{d}} \right)^{2}},
            \end{equation*}
            where $N$ is the number of measured samples, $\theta_{a_{i}}$ is the actuation coordinate at the $i$-th sample, and $\theta_{a_{d}}$ is the desired step value.
    \item Absolute value of the steady state error
            \begin{equation*}
                \lim_{t \rightarrow \infty} \abs{ \theta_{a}(t) - \theta_{a_{d}} }.
            \end{equation*}
    \item Steady-state error percentage
            \begin{equation*}
                \lim_{t \rightarrow \infty} \frac{\abs{ \theta_{a}(t) - \theta_{a_{d}} }}{\abs{\theta_{a_{d}}}} \times 100 \%.
            \end{equation*}
    \item Percentage of overshoot, defined as
            \begin{equation*}
                \abs{\max_{\theta_{a}} \frac{\theta_{a}}{\theta_{a_{d}}}} \times 100 \%.
            \end{equation*}
    \item Transient time, measured with respect to the steady state value reached by $\theta_{a}$.
    \item Settling time, measured with respect to the steady state value reached by $\theta_{a}$.
\end{enumerate}

\subsection*{PD regulators} 

The RMSE, steady-state error, corresponding percentage errors, overshoot, settling time, and transient time for the PD controllers were analyzed for all curvature models listed in Table~\ref{tab:experiment:curvature}. For the complete set of results, the reader is referred to Fig.~\ref{PD:comparison:error metrics} and Fig.~\ref{PD:comparison:time metrics} in Appendix \ref{sec:Appendix_Results}. For each controller and gain value, the box plots show the median, lower and upper quartiles, and minimum and maximum values, while the colored marker indicates the mean value.

Several insights can be drawn from these plots. First, the PD regulator ($u_{1}$) is the most sensitive to the proportional gain. Its performance, in terms of the error metrics, is worse than that of all model-based controllers, regardless of the proportional gain. The values of the error metrics for the model-based PDs are typically below the mean performance of the standard PD controller, which is represented by a black dashed line. However, when $k_{P}$ becomes sufficiently large ($k_{P} \geq 0.5~[\si{\newton \meter}]$), the error approaches that of the other regulators. Most model-based controllers exhibit similar performance when $k_{P} \geq 0.3~[\si{\newton \meter}]$, with the error falling below $0.1~[\si{\radian}]$. The PD controller with potential cancellation ($u_{4}$) performs similarly to the standard PD ($u_{1}$) when $k_{P} = 0~[\si{\newton \meter}]$ because the absence of a feedforward term prevents the robot from moving toward the desired configuration. However, as soon as $k_{P} > 0~[\si{\newton \meter}]$, its performance improve significantly. Across all error metrics, the PD with feedforward ($u_{2}$) is the best-performing. As $k_{P}$ increases, the mean and variance of all control laws decrease, making their performance comparable. Another trend is that the type and order of discretization also affect the results, particularly for $u_{1}$. Overall, the PC model with four DoF demonstrates superior performance. All model-based regulators show improvement in performance from the CC model with one DoF to higher-order discretizations.

\begin{figure*}[!b]
    \centering
    \subfigure[{PD $(u_{1})$}]{
        \includegraphics[width=0.3\textwidth, clip]{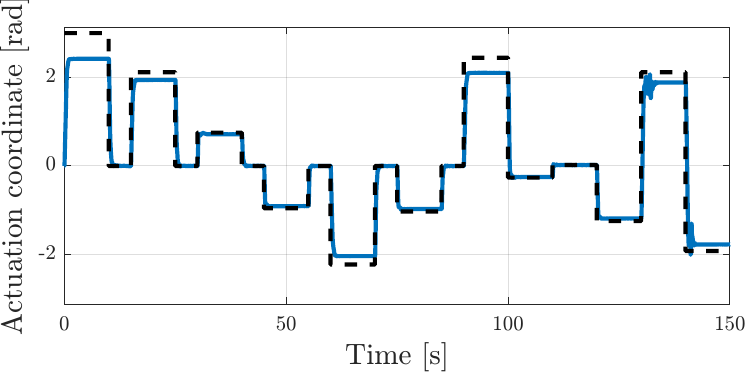}
    }
    \subfigure[{$u_{2}$}]{
        \includegraphics[width=0.3\textwidth, clip]{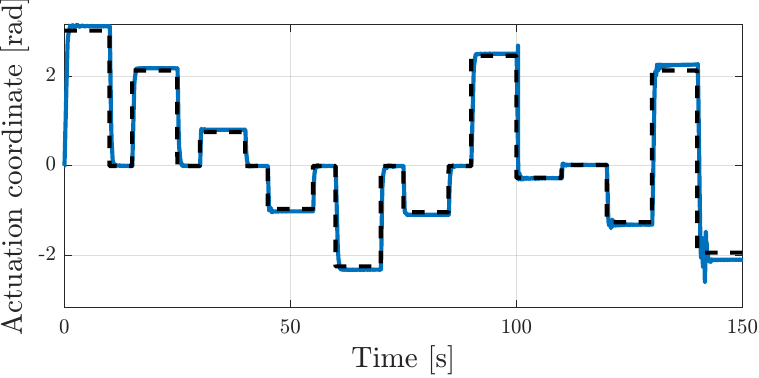}
    }
    \subfigure[{$u_{3}$}]{
        \includegraphics[width=0.3\textwidth, clip]{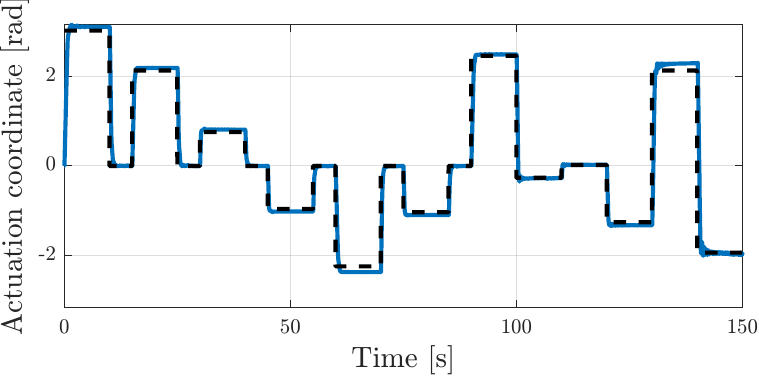}
    }
    \subfigure[{$u_{4}$}]{
        \includegraphics[width=0.3\textwidth, clip]{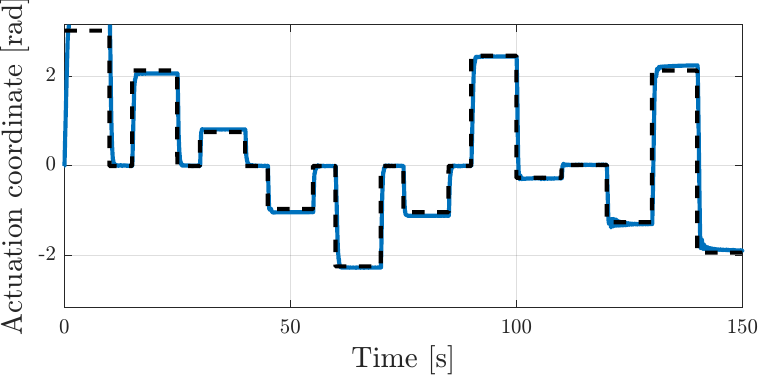}
    }
    \subfigure[{$u_{5}$}]{
        \includegraphics[width=0.3\textwidth, clip]{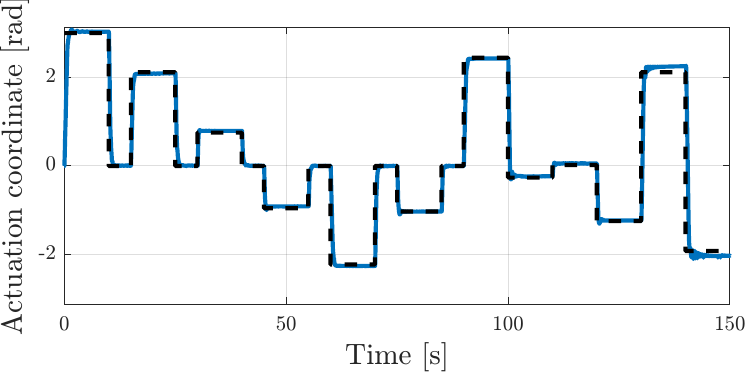}
    }
    \subfigure[{$u_{6}$}]{
        \includegraphics[width=0.3\textwidth, clip]{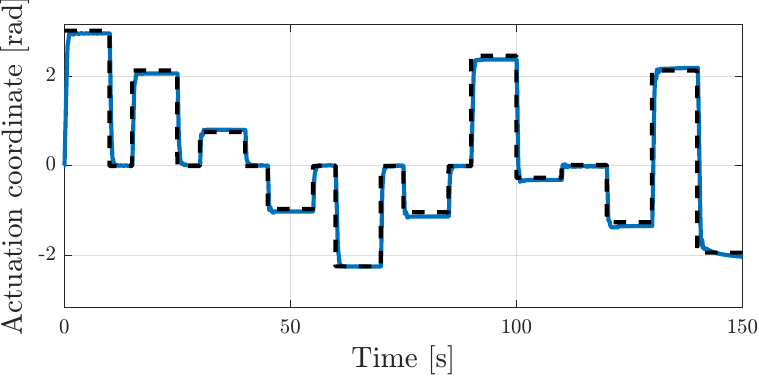}
    }
    \caption{Time evolution of the actuation coordinate and the corresponding reference (black dashed lines) for the PD regulators under the PC discretization with $4$ DoF, a proportional gain $k_{P} = 0.5~[\si{\newton \meter}]$ and no mass at the tip.}
    \label{PD:time evolutions}
\end{figure*}

 The overshoot values are all comparable. However, differences arise in the settling and transient times, particularly when a PC model is employed. The poorest performance is exhibited by the PD with compensation on the unactuated variables ($u_{5}$). One possible explanation for this behavior is that the potential forces are sensitive to the unactuated DoF. In fact, these time metrics are smaller when a fully-actuated approximation is used and increase as the number of DoF increases. When $k_{P}$ is raised, both the settling and transient times lengthen because a larger gain results in a greater control action for the same error magnitude, leading to more oscillatory behavior.

Finally, Fig.~\ref{PD:time evolutions} illustrates the time evolutions of the controllers under the PC discretization with 4 DoF with $k_{P} = 0.5~[\si{\newton \meter}]$ and no mass at the tip. The most notable differences, as expected, are observed between the PD ($u_1$) and the other regulators.

\subsection*{PID regulators}

The same metrics were now analyzed for the PID
controllers. For the complete set of results, the reader is referred to  Figs.~\ref{PID:comparison:error metrics} and \ref{PID:comparison:time metrics} in Appendix \ref{sec:Appendix_Results}.


 It is immediately apparent that the integral term results in closed-loops with similar performance in terms of error metrics, even when the proportional gain is small. The regulator exhibiting worst results is $u_{10}$. Figure ~\ref{PID:comparison:time metrics} illustrates the overshoot, settling time and transient time. In this case, the outcome of the regulators varies significantly. The PID ($u_{7}$) and P-satI-D ($u_{8}$) perform better in terms of overshoot, which is consistent with the fact that neither incorporate a model-based term. Model-based terms, especially when used as feedforward, can lead to larger overshoots. This occurs because, at the start of the regulation interval, the integral accumulates a large error, which is quickly reduced by the model-based term. Regarding settling and transient times, the P-satI-D ($u_{8}$) performs better than the others. The addition of $u_{m}$ does not improve the time metrics. This behavior can be explained by the dominance of the integral over the model-based components. 

Compensating for the unactuated DoF ($u_{12}$) tends to worsen the performance. When the plots of a controller exceed $5~[\si{\second}]$ or are absent, it indicates that the controller is unable to settle within the control interval. In other words, the closed-loop system continues to oscillate around the setpoint. This suggests that compensating for the unactuated DoF may have a negative impact on stability. Furthermore, this effect seems more pronounced for the PC model and systems with a higher number of DoF.

Additionally, we believe that these oscillations are partly due to the high friction in the motors, which was an issue only detected after the experimental platform was built. In fact, it is well known that integral terms can lead to oscillatory behavior when high friction is present in the actuators.

Finally, Fig.~\ref{PsatID:time evolution} shows the time evolution of these controllers under the PC discretization with 4 DoF with $k_{P} = 0.5~[\si{\newton \meter}]$ and no mass at the tip. The behaviors observed in these evolutions are consistent with the previous plots. By zooming in, one can observe the effect of friction, which often causes the output to take on a piecewise constant shape at steady state.

\begin{figure*}[!t]
    \centering
    \subfigure[{PID $(u_{7})$}]{
        \includegraphics[width=0.3\textwidth, clip]{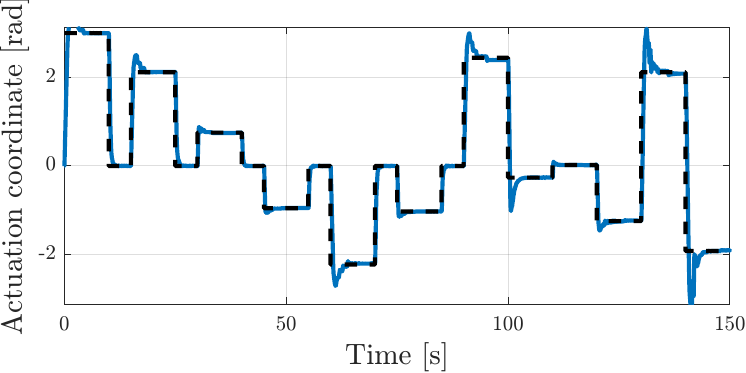}
    }
    \subfigure[{$u_{8}$}]{
        \includegraphics[width=0.3\textwidth, clip]{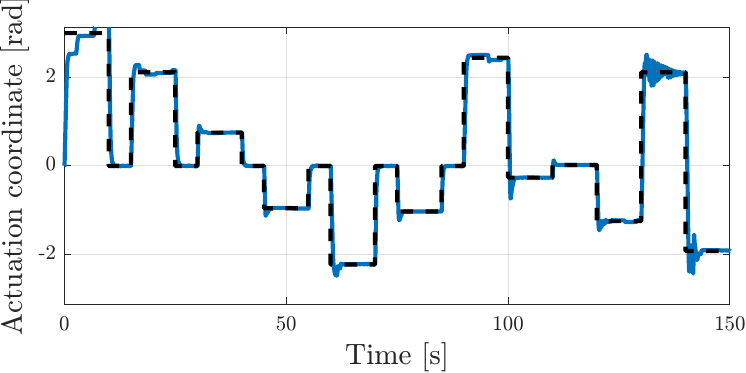}
    }
    \subfigure[{$u_{9}$}]{
        \includegraphics[width=0.3\textwidth, clip]{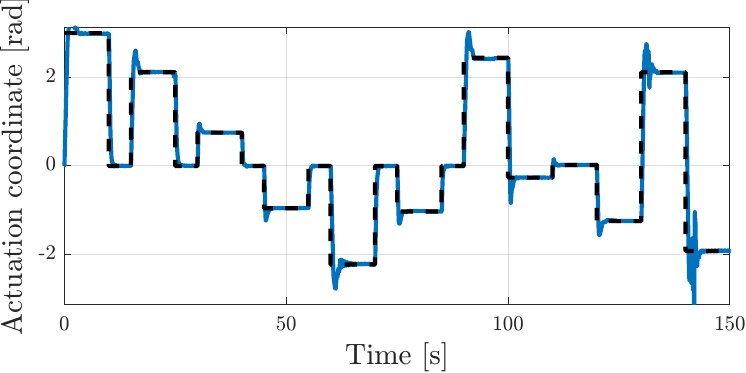}
    }
    \subfigure[{$u_{10}$}]{
        \includegraphics[width=0.3\textwidth, clip]{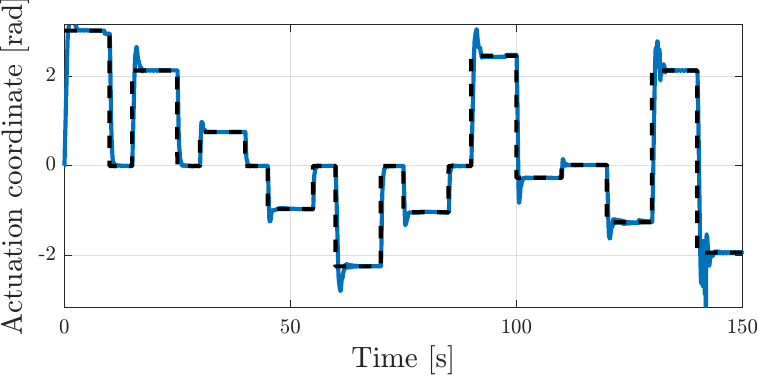}
    }
    \subfigure[{$u_{11}$}]{
        \includegraphics[width=0.3\textwidth, clip]{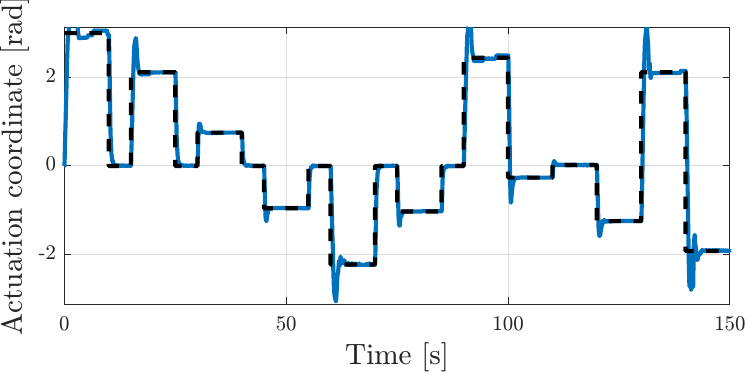}
    }
    \subfigure[{$u_{12}$}]{
        \includegraphics[width=0.3\textwidth, clip]{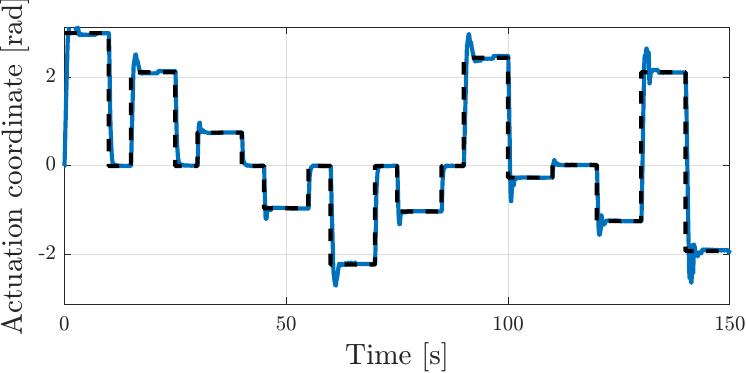}
    }
    \subfigure[{$u_{13}$}]{
        \includegraphics[width=0.3\textwidth, clip]{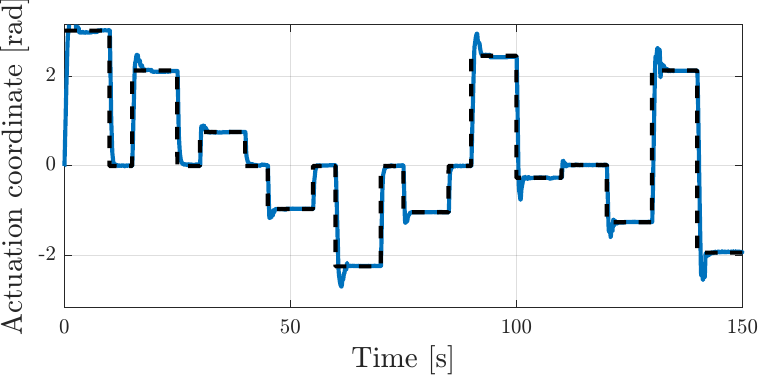}
    }
    \caption{Time evolution of the actuation coordinate and the corresponding reference (black dashed lines) for the PID/P-satI-D regulators under the PC discretization with $4$ DoF, a proportional gain $k_{P} = 0.5~[\si{\newton \meter}]$ and no mass at the tip.}
    \label{PsatID:time evolution}
\end{figure*}

\subsection{PD vs PID}

We now compare the performance of the PD and PID controllers using the previously presented metrics and their mean values. Figure ~\ref{PD:comparison:time metrics_Final} summarizes the controller performance across all discretization schemes by averaging the metric values over all discretizations. More detailed analysis of these results are referred to Figs.~\ref{experiment:overall controllers comparison error} and~\ref{experiment:overall controllers comparison time} in Appendix \ref{sec:Appendix_Results} showing all the mean values across all the controllers. Note that the color map has been rescaled between the PD and PID controllers.

By inspecting Fig.~\ref{PD:comparison:time metrics_Final}, it can be observed that the PD controllers exhibit more consistent behavior in both error- and time-based metrics as $k_{P}$ increases. The error scale differs significantly, with the integral in the PID controllers providing approximately an order of magnitude improvement. This is because it compensates for model mismatches and unmodeled dynamics.

However, when examining overshoot and time metrics, the benefits of not having an integrator become evident. The PDs demonstrate smaller overshoot, and shorter settling and transient times. The PIDs have higher tendency toward instability, as predicted by theory. Instability is more pronounced at larger $k_{P}$ values, despite that increasing proportional gain should attenuate it. This is likely because the increase in $k_{P}$ amplifies measurement noise and leads to actuator saturation--factors that are not accounted for in the analyses.

Figures ~\ref{experiment:overall controllers comparison error payload} and~\ref{experiment:overall controllers comparison time payload} present the performance metrics across all controllers of each type with the presence of the payload. The results are consistent with those obtained without the payload, demonstrating the robustness of the proposed approach.

\begin{figure}[h]
    \centering
    \includegraphics[width=0.98\columnwidth]{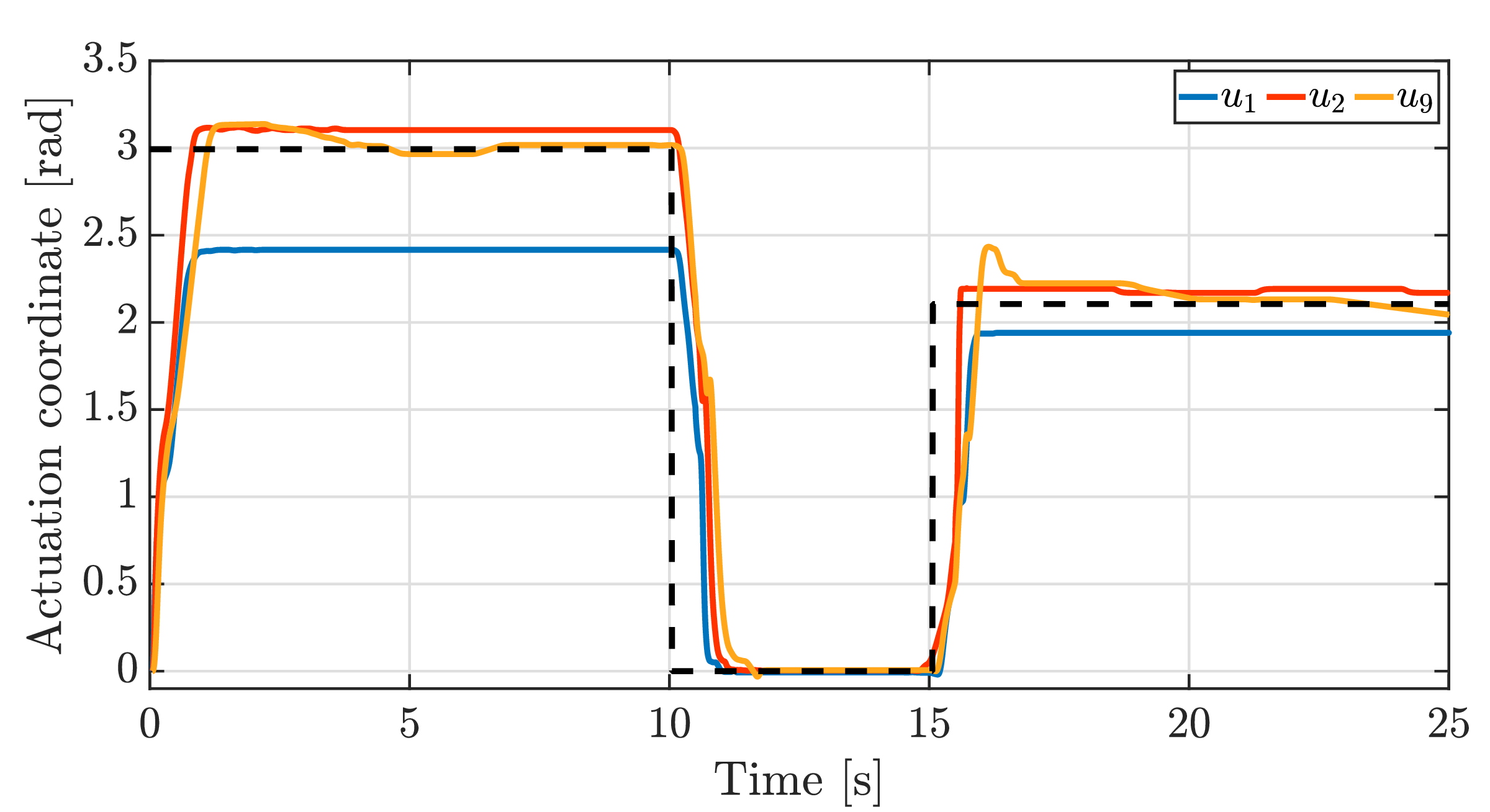}
    \caption{Time evolution of three controllers (PD ($u_1$), PD+ ($u_2$), and P-satI-D+ ($u_9$)) for two  different target shapes using  PC discretization with 4 DoF and a proportional gain $k_{P} = 0.5~[\si{\newton \meter}]$.}
    \label{experimental_pick}
\end{figure}

Finally, Fig. \ref{experimental_pick} compares the performance of three representative controllers described in this work—PD ($u_1$), PD+ ($u_2$), and P-satI-D+ ($u_9$)—for two target shapes, using PC discretization with 4 DoF and a proportional gain of $\bm{K}_P = 0.5~[\si{\newton \meter}]$.  The figure clearly shows  that the PD controller without model-based components does not provide sufficient accuracy. Adding the model-based feedforward term substantially reduces the tracking error. Moreover, combining integral action with the model-compensation term yields a response that is both fast and accurate, while reducing the steady-state error caused by model mismatch and other uncertainties. This improvement comes at the cost of a slightly slower response and a small overshoot, with the settling time increasing from $0.4~[\si{\second}]$ for the PD+ controller to $0.85~[\si{\second}]$ in this particular example.

\subsection{Discretizations comparison}
We move to qualitatively discuss the performance of the discretizations (Table~\ref{tab:experiment:curvature}) for shape representation. To the best of our knowledge, this is the first time that different discretizations have been applied to the same soft robotics platform and their performance compared.
For each model, the following metrics are considered.
\begin{enumerate}[label=(\roman*)]
    \item The average norm of the steady state Cartesian error at the markers, i.e.,
        \begin{equation*}
        \lim_{t \rightarrow \infty}\norm{\pv(t) - \hat{\pv}(\qv(t))},
        \end{equation*}
    where $\pv \in \R^{6}$ is the vector collecting the Cartesian positions of the markers measured by the motion capture system, and $\hat{\pv}(\qv)$ is the one computed through the direct kinematics.
    \item The average RMS Cartesian error at the markers, defined as
        \begin{equation*}
        \sqrt{\frac{1}{N}\sum_{i = 1}^{N} \norm{\pv_{i} - \hat{\pv}(\qv_{i})}^{2}}.
        \end{equation*}
    \item The average absolute value of the steady state angular error at the markers, computed similarly to the Cartesian error, using the vector of bending plane angles $\alphav$ instead of $\pv$.
    \item The average RMS angular error at the markers, computed similarly to the RMS Cartesian error, using the vector of bending plane angles $\alphav$ instead of $\pv$.
    \item The steady state task error.
    \item The RMS task error.
\end{enumerate}

Their values are computed by considering all controllers, masses, and values of $k_{P}$, allowing for an analysis of the discretizations performance independently of the specific controller used.

Figures ~\ref{model comparison:metrics} and~\ref{discretiations:strobo_plots} present the metrics and snapshots of the soft robot, respectively. Among the tested models, the fully actuated CC discretization with a single DoF (CC 1 DoF) exhibits the worst performance. Adding a second DoF significantly improves the model ability to describe the robot configuration.

Discretizations with 3 DoF show reduced performance compared to those with 2 DoF, particularly in orientation. This discrepancy may be attributed to the placement of the first marker, located at the midpoint of the robot, as marker positions likely influence state reconstruction and its accuracy. However, it is noteworthy that 3 DoF discretizations exhibit less variability in the position metrics compared to models with fewer DoF.

Models with 4 DoF better capture the robot shape, with the most noticeable improvements observed in orientation accuracy. Additionally, while task error exhibits greater variability in several underactuated models, its average value remains consistent across all discretizations. Among the tested configurations, the PC model with 4 DoF slightly outperforms the other underactuated descriptions in overall performance.

\section{Conclusions}\label{sec:conclusions}
This paper addresses the shape regulation problem for continuum soft robots. We approach this challenge by leveraging the concept of collocated form, which reformulates the robot dynamics such that only some degrees of freedom are directly influenced by the system inputs. 

We demonstrate that shape regulation via collocated control can be achieved in a unified framework using three families of PD- and PID-like regulators, each offering distinct stability properties. The majority of these controllers are validated on a newly designed soft robotic platform. In the experiments, we evaluate the closed-loop performance of the control strategies based on three criteria: variations in proportional gain, changes in payload carried by the robot, and the use of different reduced-order models. Additionally, we examine how the choice of dynamic model influences shape control.

Experimental results indicate that PD regulators exhibit superior transient performance and stability compared to PID regulators. However, the latter achieve tracking errors that are, on average, an order of magnitude smaller. Furthermore, the results show empirically that underactuated models generally provide a more accurate description of the robot shape.

Experiments reveal that increasing the discretization order can adversely affect the metrics related to shape. This fact may be attributed to sensors location, which likely impacts the accuracy of configuration reconstruction. 

\vspace{0.5cm}

\hspace{-0.4cm}{\LARGE\textbf{Appendix}}

\appendix

\section{Proof of theorem 1} 
\label{Sec:Appendix 1}

  First, note that $\thetav_{d}$ is an equilibrium and the linearization around $\thetav_{d}$ reads
    \begin{equation*}
    \begin{split}
        &\Mm(\thetav_{d})\Delta\ddthetav + \jac{\gv}{\thetav}(\thetav_{d})\Delta\thetav + \jac{\kv}{\thetav}(\thetav_{d})\\
        &+ \tilde{\Dm}(\thetav_{d})\Delta \dthetav + \begin{carray}{c}
            \bm{K}_P\Delta\thetav_{a}-\jac{\uv_{m}}{\thetav}(\thetav_{d})\Delta\thetav\\
            \zerov_{n-m}
        \end{carray} = \zerov_{n},
    \end{split}
    \end{equation*}
    with
    \begin{equation*}
        \tilde{\Dm}(\thetav_{d}) = \begin{carray}{cc}
                \Dm_{aa} + \bm{K}_D & \Dm_{au}\\
                \Dm_{ua} & \Dm_{uu}
        \end{carray}.
    \end{equation*}
    The configuration $\Delta\thetav = \zerov_{n}$ is asymptotically stable, in a local sense, if and only if
    \begin{equation*}
        \Qm = \begin{carray}{cc}
            \Qm_{\thetav_{a}\thetav_{a}} & \Qm_{\thetav_{a}\thetav_{u}}\\
            \mathrm{sym} & \Qm_{\thetav_{u}\thetav_{u}}
        \end{carray} > 0,
    \end{equation*}
    where
    \begin{equation*}
        \Qm_{\thetav_{a}\thetav_{a}} = \jac{\gv_{a}}{\thetav_{a}} + \jac{\kv_{a}}{\thetav_{a}} + \bm{K}_P - \jac{\uv_{m}}{\thetav_{a}},
    \end{equation*}
    \begin{equation*}
        \Qm_{\thetav_{a}\thetav_{u}} = \jac{\gv_{a}}{\thetav_{u}} + \jac{\kv_{a}}{\thetav_{u}} - \frac{1}{2}\jac{\uv_{m}}{\thetav_{u}},
    \end{equation*}
    and
    \begin{equation*}
        \Qm_{\thetav_{u}\thetav_{u}} = \jac{\gv_{u}}{\thetav_{u}} + \jac{\kv_{u}}{\thetav_{u}}.
    \end{equation*}
    Since $\Qm_{\thetav_{u}\thetav_{u}} > 0$ by assumption, $\Qm$ is positive definite if and only if
    \begin{equation*}
        \Qm_{\thetav_{a}\thetav_{a}} - \Qm_{\thetav_{a}\thetav_{u}}\Qm_{\thetav_{u}\thetav_{u}}^{-1}\Qm_{\thetav_{a}\thetav_{u}}^{T} > 0,
    \end{equation*}
    which can always be enforced by a sufficiently large $\bm{K}_P$.

\section{Proof of theorem 2} 
\label{Sec:Appendix 2}

We first provide some properties that will be needed to develop the proof of Theorem \ref{theorem:regulator} and Theorem \ref{theorem:shape_regulation:regulator 2}. 

\begin{property}
(\hspace{1sp}\cite{kelly1998global})\label{prop:saturation:1} Given $k_{1} = \alpha_{1}$, $k_{2} = \alpha_{1}\beta$, and $k_{3} = \sqrt{h}\beta$, it holds
$$\small
\norm{\sv(\xv)} \geq 
\begin{cases}
    k_{1}\norm{\xv};\,\,\, \norm{\xv} \leq \beta\\
    \,\,\,k_{2}\,\,\,\,\,\,\,;\,\,\, \norm{\xv} > \beta, 
\end{cases}
$$
and
$$\small
\norm{\sv(\xv)} \leq 
\begin{cases}
    \norm{\xv};\,\,\, \norm{\xv} \leq \beta\\
    \,\,\,k_{3}\,;\,\,\, \norm{\xv} > \beta.
\end{cases}
$$
\end{property}

\begin{property}(\hspace{1sp}\cite{kelly1998global})\label{property:saturation:scalar product 1}
$\small
\sv^{T}(\xv)\xv \geq 
    \begin{cases}
        k_{1} \norm{\xv}^{2} ;\,\,\, \norm{\xv} \leq \beta\\
        k_{2} \norm{\xv} ;\,\,\, \norm{\xv} > \beta.
    \end{cases}
$
\end{property}

\begin{property}(\hspace{1sp}\cite{pustina2023psaitd})\label{prop:saturation:3}
$\small \lambda_{\max}(\Qm) \sv^{T}(\xv)\xv \geq \sv^{T}(\xv) \Qm \xv \geq \lambda_{\min}(\Qm) \sv^{T}(\xv)\xv$.
\end{property}

\begin{property}(\hspace{1sp}\cite{pustina2023psaitd})\label{prop:saturation:4}
Given $k_{4} = \alpha_{2}$ and $k_{5} = \alpha_{3}$, one has 
$$\norm{\jac{\sv(\xv)}{\xv}} \leq k_{4}, \quad \norm{\IIm_{h} - \jac{\sv(\xv)}{\xv}} \leq k_{5}\norm{\sv(\xv)}.$$ 
\end{property}

\begin{property}\label{prop:saturation:5}
When $\norm{x} \leq \beta$,
$
\small k_{1}\xv^{T}\yv \leq \sv^{T}(\xv)\yv \leq \xv^{T}\yv.
$
\end{property}

   Let
    \begin{equation}\label{eq:V:regulator}
    \begin{split}
        V &= \nu_{1} \left( \frac{1}{2} \dthetav^{T}\Mm\dthetav + \mathcal{U}_{g} + \mathcal{U}_{e} + \frac{1}{2}\tthetav_{a}^{T}\bm{K}_P\tthetav_{a} + \tthetav_{a}^{T}\uv_{m}  \right.\\
        &\left. \phantom{\frac{1}{2}} + \tthetav_{a}^{T}\bm{K}_I\zv + \nu_{2} \right) + \frac{1}{2}\zv^{T}\bm{K}_I\zv - \sv^{T}(\tthetav_{a}) \Mm_{a}\dthetav + \nu_{3},
    \end{split}
    \end{equation}
    where $\nu_{1}, \nu_{2}$, and $\nu_{3} \in \R^{+}$ are constants to be selected. Moreover, $\zv$ is the state vector associated with the saturated integral action, defined as
    \begin{equation*}
        \zv(t) = \int_{0}^{t} \sv(\tthetav_{a}(\tau))\drm \tau.
    \end{equation*}

    First, we examine the conditions under which $V$ remains non-negative. To do so, it is helpful to express
    \begin{equation*}
        V = V_{1} + V_{2} + \nu_{1} \left( \mathcal{U}_{g} + V_{3} \right),
    \end{equation*}
    where
    \begin{equation}\label{eq:shape regulation:V1}
    \begin{split}
        V_{1} &= \nu_{1} \frac{1}{2} \dthetav^{T}\Mm\dthetav - \sv^{T}(\tthetav_{a})\Mm_{a}\dthetav + \nu_{3},
    \end{split}
    \end{equation}
    \begin{equation}\label{eq:shape regulation:V2}
    \begin{split}
        V_{2} &= \nu_{1} \left( \frac{1}{4} \tthetav_{a}^{T}\bm{K}_P\tthetav_{a} + \tthetav_{a}^{T}\bm{K}_I\zv \right) + \frac{1}{2}\zv^{T}\bm{K}_I\zv,
    \end{split}
    \end{equation}
    and
    \[
        V_{3} = \mathcal{U}_{e} + \frac{1}{4} \tthetav_{a}^{T}\bm{K}_P\tthetav_{a} + \tthetav_{a}^{T}\uv_{m} + \nu_{2}.
    \]

    Owing to the properties of the saturation function, we have
    \begin{equation*}
    \begin{split}
        V_{1} &\geq \frac{\nu_{1}}{2} \lambda_{m} \norm{\dthetav}^{2} - k_{2}\sigma_{\max}(\Mm_{a})\norm{\dthetav} + \nu_{3}.
    \end{split} 
    \end{equation*}
    This inequality is lower bounded as it is dominated by a quadratic function in $\norm{\dthetav}$. To ensure that $V_{1}$ remains non-negative, it is sufficient to impose that the discriminant
    \[
        \Delta = k_{2}^{2} \sigma_{\max}^{2}(\Mm_{a}) - 2 \nu_{1} \lambda_{m} \nu_{3} \leq 0,
    \]
    which holds when $\nu_{3} = \frac{k_{2}^{2} \sigma_{\max}^{2}(\Mm_{a})}{2 \nu_{1} \lambda_{m}}$. 

    The term $V_{2}$ can be expressed as a quadratic form
    \begin{equation*}
        V_{2} = \frac{1}{2} \begin{carray}{c} \tthetav_{a} \\ \zv \end{carray}^{T} \begin{carray}{cc} \frac{\nu_{1}}{2}\bm{K}_P & \nu_{1}\bm{K}_I \\ \mathrm{sym} & \bm{K}_I \end{carray} \begin{carray}{c} \tthetav_{a} \\ \zv \end{carray},
    \end{equation*}
    which is positive definite if and only if
    \[
        \bm{K}_P - 2\nu_{1}\bm{K}_I > 0.
    \]
    This condition can be satisfied by selecting $\bm{K}_P$ sufficiently large. 

    Now, considering $V_3$, we observe that it is lower bounded by
    \begin{equation}\label{eq:shape regulation:V3 ineq}
    \begin{split}
        V_3 &\geq \frac{1}{2} \begin{carray}{c} \norm{\tthetav_{a}} \\ \norm{\thetav} \end{carray}^{T}
        \begin{carray}{cc} \frac{1}{4}\lambda_{\min}(\bm{K}_P) & -k_{u_1} \\ \mathrm{sym} & \lambda_k \end{carray}
        \begin{carray}{c} \norm{\tthetav_{a}} \\ \norm{\thetav} \end{carray} \\
        &\quad + \frac{1}{8}\lambda_{\min}(\bm{K}_P) \norm{\tthetav_{a}}^2 - k_{u_2} \norm{\thetav_{a_d}} \norm{\tthetav_{a}} + \nu_2.
    \end{split}
    \end{equation}
    The quadratic term in $\norm{\tthetav_{a}}$ and $\norm{\thetav}$ is non-negative when
    \begin{equation}
        \lambda_{\min}(\bm{K}_P) \geq \frac{2 k_{u_1}^2}{\lambda_k}.
    \end{equation}
    Furthermore, the quadratic term in $\norm{\tthetav_{a}}$ can also be made non-negative by ensuring that its discriminant is non-positive, i.e.,
    \[
        \Delta = k_{u_2}^2 \norm{\thetav_{a_d}}^2 - \frac{\lambda_{\min}(\bm{K}_P) \nu_2}{2} \leq 0.
    \]
    This holds when
    \[
        \nu_2 = \frac{2 k_{u_2}^2 \norm{\thetav_{a_d}}^2}{\lambda_{\min}(\bm{K}_P)}.
    \]

    Thus, $V$ is non-negative everywhere provided that $\bm{K}_P$ is large. Additionally, $V$ is radially unbounded because it is lower bounded by quadratic functions of the state.

    Next, we analyze the first-order time derivative of $V$. After some calculations and simplifications, gives
    \begin{equation}
    \begin{split}
        \dot{V} &= \nu_1 \left( -\dthetav^T \Tilde{\Dm} \dthetav + \tthetav_{a}^T \jac{\uv_{m}}{\thetav} \dthetav + \sv^T(\tthetav_{a}) \bm{K}_I \tthetav_{a} \right) \\
        &\quad + \dthetav^T \Mm_{a}^T \jac{\sv(\tthetav_{a})}{\tthetav_{a}} \dthetav_{a} + \sv^T(\tthetav_{a}) \left( (\Cm^T)_a \dthetav + \tilde{\Dm}_a \dthetav \right) \\
        &\quad - \sv^T(\tthetav_{a}) \left( \uv_{m} - \gv_{a} - \kv_{a} \right) - \sv^T(\tthetav_{a}) \bm{K}_P \tthetav_{a},
    \end{split}
    \end{equation}
    with
    \[
        \tilde{\Dm} = \begin{carray}{cc} \Dm_{aa} + \bm{K}_D & \Dm_{au} \\ \Dm_{ua} & \Dm_{uu} \end{carray}.
    \]
    The reader can recognize that $\dot{V}$ resembles a quadratic function of $\tthetav_{a}$ and $\dthetav$. This is indeed true when $\norm{\tthetav_{a}} \leq \beta$. However, for $\norm{\tthetav_{a}} > \beta$, the saturation of $\sv(\tthetav_{a})$ ensures that $\dot{V}$ is upper bounded only by a function of $\tthetav_{a}$, and no conclusion can be drawn about its sign. Thus, we assume $\norm{\tthetav_{a}} \leq \beta$ from this point onward.

    Invoking Property~\ref{property:coriolis matrix}, it holds
    \[
        \sv^T(\tthetav_{a}) \Cm^{T}_{a} \dthetav \leq k_3 \lambda_C \dthetav^T \dthetav.
    \]
    From Property~\ref{property:saturation:scalar product 1} and Property~\ref{prop:saturation:5}, we also have
    \[
        -\sv^T(\tthetav_{a}) \left( \bm{K}_P - \nu_1 \bm{K}_I \right) \tthetav_{a} \leq -\lambda_{\min} \left( \bm{K}_P - \nu_1 \bm{K}_I \right) \tthetav_{a}^T \tthetav_{a},
    \]
    \[
        \sv^T(\tthetav_{a}) \tilde{\Dm}_a \dthetav \leq \tthetav_{a}^T \tilde{\Dm}_a \dthetav,
    \]
    and from the second property of $\uv_{m}$
    \[
        -\sv^T(\tthetav_{a}) \left( \uv_{m} - \gv_{a} - \kv_{a} \right) \leq k_{u_3} \tthetav_{a}^T \tthetav_{a}.
    \]
    Therefore, $\dot{V}$ is upper bounded by
    \[
        \dot{V} \leq - \begin{carray}{c} \tthetav_{a} \\ \dthetav \end{carray}^T
        \begin{carray}{cc} \Qm_{\tthetav_{a} \tthetav_{a}} & \Qm_{\tthetav_{a} \dthetav} \\ \mathrm{sym} & \Qm_{\dthetav \dthetav} \end{carray}
        \begin{carray}{c} \tthetav_{a} \\ \dthetav \end{carray},
    \]
    with
    \[
        \Qm_{\tthetav_{a} \tthetav_{a}} = \left[ \lambda_{\min} \left( \bm{K}_P - \nu_1 \bm{K}_I \right) - k_{u_3} \right] \IIm_{m},
    \]
    \[
        \Qm_{\tthetav_{a} \dthetav} = -\frac{1}{2} \left( \nu_1 \jac{\uv_{m}}{\thetav} + \tilde{\Dm}_a \right),
    \]
    and
    \[
    \begin{split}    
        \Qm_{\dthetav \dthetav} &= \nu_1 \tilde{\Dm} - k_3 \lambda_C \IIm_n \\
        &\quad- \mathrm{sym}\begin{carray}{c} 
        \jac{\sv(\tthetav_{a})}{\tthetav_{a}}^T \Mm_{a} \\ \zerov_{(n-m) \times n} \end{carray}.
    \end{split}
    \]
    This quadratic form is negative definite if
    \[
        \Qm_{\dthetav \dthetav} > 0
    \]
    and
    \[
        \Qm_{\tthetav_{a} \tthetav_{a}} - \Qm_{\tthetav_{a} \dthetav} \Qm_{\dthetav \dthetav}^{-1} \Qm_{\tthetav_{a} \dthetav}^T > 0.
    \]
    The first condition holds if
    \[
        \nu_1 > \frac{k_3 \lambda_C + \sigma_{\max} \left( \jac{\sv(\tthetav_{a})}{\tthetav_{a}}^T \Mm_{a} \right)}{\lambda_{\min}(\tilde{\Dm})}.
    \]
    The second condition can be enforced by appropriately selecting $\bm{K}_P$, for example
    \[
        \lambda_{\min}(\bm{K}_P) > \nu_1 \lambda_{\max}(\bm{K}_I) + k_{u_3} + \frac{\lambda_{\max} \left( \Qm_{\tthetav_{a} \dthetav} \Qm_{\tthetav_{a} \dthetav}^T \right)}{\lambda_{\min}(\Qm_{\dthetav \dthetav})}.
    \]

    The set
    \[
        \Omega_0 = \left\{ (\thetav, \dthetav, \zv) \, \middle| \, V(\thetav, \dthetav, \zv) \leq c \right\},
    \]
    with $\displaystyle c = \nu_1 \left( \frac{\beta^2}{8} \lambda_{\min}(\bm{K}_P) - k_{u_2} \norm{\thetav_{a_d}} \beta + \nu_2 \right)$, is compact since $V$ is non-negative and radially unbounded~\cite{khalil2002nonlinear}. Furthermore, $\dot{V}$ is non-positive within $\Omega_0$ because $\norm{\tthetav_{a}} \leq \beta$ in $\Omega_{0}$. Assume the contrary, i.e., there exists some $\thetav_{a} \in \Omega_0$ such that $\norm{\tthetav_{a}} > \beta$. Then, recalling the bounds on $V$, we find
    \[
        V \geq \nu_1 \left( \frac{1}{8} \lambda_{\min}(\bm{K}_P) \norm{\tthetav_{a}}^2 - k_{u_2} \norm{\thetav_{a_d}} \norm{\tthetav_{a}} + \nu_2 \right) > c,
    \]
    which is a contradiction.

    All conditions for applying LaSalle Theorem hold. Thus, the trajectories of the closed-loop system starting in $\Omega_0$ converge to the largest invariant set within $\Omega_0$ where $\dot{V} = 0$
    \[
        \left\{ (\thetav, \dthetav) \, \middle| \, \tthetav_{a} = \thetav_{a_d}, \dthetav = \zerov_n \right\}.
    \]

\section{Proof of theorem 3} 
\label{Sec:Appendix 3}

    The proof relies on a Lyapunov-like function similar to the one used in Theorem~\ref{theorem:regulator}.
    Specifically, consider
    \begin{equation*}
        \begin{split}
        V &= \nu_{1} \left( \frac{1}{2} \dthetav^{T}\Mm\dthetav + \mathcal{U}_{g} + \mathcal{U}_{e} + \frac{1}{2}\tthetav_{a}^{T}\bm{K}_P\tthetav_{a} + \sv^{T}(\tthetav_{a})\uv_{m} \right.\\
        &\left. \phantom{\frac{1}{2}} + \tthetav_{a}^{T}\bm{K}_I\zv + \nu_{2} \right) + \frac{1}{2}\zv^{T}\bm{K}_I\zv - \sv^{T}(\tthetav_{a}) \Mm_{\thetav_{a}}\dthetav + \nu_{3}.
        \end{split}
    \end{equation*}
    Rearranging the terms, $V$ can be expressed as
    \begin{equation*}
        V = V_{1} + V_{2} + \nu_{1} \left( \mathcal{U}_{g} + V_{3} \right),
    \end{equation*}
    where $V_{1}$ and $V_{2}$ are defined as in~\eqref{eq:shape regulation:V1} and~\eqref{eq:shape regulation:V2}, respectively, and
    \begin{equation*}
        V_{3} = \mathcal{U}_{e} + \frac{1}{4}\tthetav_{a}^{T}\bm{K}_P\tthetav_{a} + \sv^{T}(\tthetav_{a})\uv_{m}(\thetav, \thetav_{a_{d}}) + \nu_{2}.
    \end{equation*}
    
    Since $V_{1}$ and $V_{2}$ are non-negative (see proof of Theorem~\ref{theorem:regulator}), $V$ is non-negative if and only if $V_{3}$ is non-negative. Using Property~\ref{property:gravity}, we derive
    \begin{equation*}
    \begin{split}
        V_{3} &\geq \lambda_{k}\norm{\thetav}^{2} + \frac{1}{4}\lambda_{\min}(\bm{K}_P)\norm{\tthetav_{a}}^{2}\\
        &\quad - \norm{\sv(\tthetav_{a})}\left( k_{u_{1}}\norm{\thetav} + k_{u_{2}}\norm{\thetav_{a_{d}}} \right) + \nu_{2}.
    \end{split}
    \end{equation*}
    For $\norm{\tthetav} \leq \beta$, we have $\norm{\sv(\tthetav_{a})} \geq \norm{\tthetav_{a}}$, yielding
    \begin{equation*}
        V_{3} \geq \lambda_{k}\norm{\thetav}^{2} + \frac{1}{4}\lambda_{\min}(\bm{K}_P)\norm{\tthetav_{a}}^{2} - \norm{\tthetav_{a}}\left( k_{u_{1}}\norm{\thetav} + k_{u_{2}}\norm{\thetav_{a_{d}}} \right) + \nu_{2}.
    \end{equation*}
    Referring to~\eqref{eq:shape regulation:V3 ineq}, $V_{3}$ is non-negative if
    \begin{equation*}
        \nu_{2} \geq \frac{2 k_{u_{2}}^{2} \norm{\thetav_{a_{d}}}^{2}}{\lambda_{\min}(\bm{K}_P)}, \quad \lambda_{\min}(\bm{K}_P) \geq \frac{2 k_{u_{1}}^{2}}{\lambda_{k}}.
    \end{equation*}
    
    For $\norm{\tthetav} > \beta$, $\norm{\sv(\tthetav_{a})} \geq k_{3}$, which results in
    \begin{equation*}
        V_{3} \geq \lambda_{k}\norm{\thetav}^{2} - k_{3}k_{u_{1}}\norm{\thetav} - k_{3}k_{u_{2}}\norm{\thetav_{a_{d}}} + \nu_{2},
    \end{equation*}
    and is non-negative if
    \begin{equation*}
        \nu_{2} \geq k_{3}k_{u_{2}} \norm{\thetav_{a_{d}}} + \frac{k_{3}^{2}k_{u_{2}}^{2}}{4 \lambda_{k}}.
    \end{equation*}

    Thus, by selecting
    \begin{equation*}
        \nu_{2} = \max\left( \frac{2 k_{u_{2}}^{2} \norm{\thetav_{a_{d}}}^{2}}{\lambda_{\min}(\bm{K}_P)}, \, k_{3}k_{u_{2}} \norm{\thetav_{a_{d}}} + \frac{k_{3}^{2}k_{u_{2}}^{2}}{4 \lambda_{k}} \right),
    \end{equation*}
    and sufficiently large $\bm{K}_P$, we can ensure that $V_{3} \geq 0$ everywhere.

    Now, consider the first-order time derivative of $V$. After simplification, we obtain
    \begin{equation*}
    \begin{split}
        \dot{V} &= \nu_{1} \left[ -\dthetav^{T} \Tilde{\Dm} \dthetav + \dthetav_{a}^{T}\left( \IIm_{m} - \jac{\sv(\tthetav_{a})}{\tthetav_{a}} \right)\uv_{m} \right.\\ 
        &\left. + \sv(\tthetav_{a})^{T}\jac{\uv_{m}}{\thetav}\dthetav 
        \vphantom{\dthetav_{a}^{T}\left( \IIm_{m} - \jac{\sv(\tthetav_{a})}{\tthetav_{a}} \right)\uv_{m}} + \sv^{T}(\tthetav_{a})\bm{K}_I\tthetav_{a} \right] + \dthetav^{T}\Mm_{\thetav_{a}}^{T} \jac{\sv(\tthetav_{a})}{\tthetav_{a}}\dthetav_{a} \\
        &\quad + \sv^{T}(\tthetav_{a})\left( (\Cm_{\thetav}^{T})_{a} \dthetav + \tilde{\Dm}_{a}\dthetav \right) -\sv^{T}(\tthetav_{a})\left( \uv_{m} - \gv_{a} - \kv_{a} \right)\\
        &\quad- \sv^{T}(\tthetav_{a})\bm{K}_P\tthetav_{a},
    \end{split}
    \end{equation*}
    where
    \begin{equation*}
        \tilde{\Dm} = \begin{carray}{cc}
            \Dm_{aa} + \bm{K}_D & \Dm_{au}\\
            \Dm_{ua} & \Dm_{uu}
        \end{carray}.
    \end{equation*}

    Using the properties of the saturation function and the model-based term $\uv_{m}$, when $\norm{\tthetav_{a}} \leq \beta$, $\dV$ is upper bounded by
    \begin{equation*}
        \dV \leq -\begin{carray}{c}
            \norm{\tthetav_{a}}\\
            \norm{\dthetav}
        \end{carray}^{T}\begin{carray}{cc}
            Q_{\tthetav_{a}\tthetav_{a}} & Q_{\tthetav_{a}\dthetav}\\
            \mathrm{sym} & Q_{\dthetav\dthetav}
        \end{carray}\begin{carray}{c}
            \norm{\tthetav_{a}}\\
            \norm{\dthetav}
        \end{carray},
    \end{equation*}
    being
    \begin{equation*}
        Q_{\tthetav_{a}\tthetav_{a}} = \lambda_{\min}(\bm{K}_P - \nu_{1}\bm{K}_I) - k_{u_{3}},
    \end{equation*}
    \begin{equation*}
        Q_{\tthetav_{a}\dthetav} = -\frac{1}{2}\left[ \nu_{1}\left( k_{u_{4}} + k_{5}k_{u_{5}} \right) + \sigma_{\max}(\tilde{\Dm}_{a}) \right],
    \end{equation*}
    and
    \begin{equation*}
        Q_{\dthetav\dthetav} = \nu_{1}\lambda_{\min}(\tilde{\Dm}) - \left( k_{3}\lambda_{C} + k_{4}\right)\lambda_{M}.
    \end{equation*}

    This quadratic form is negative definite for $Q_{\dthetav\dthetav} > 0$ and
    \begin{equation*}
        Q_{\tthetav_{a}\tthetav_{a}} > \frac{Q_{\tthetav_{a}\dthetav}^{2}}{Q_{\dthetav\dthetav}}.
    \end{equation*}
    The first condition holds if
    \begin{equation*}
        \nu_{1} > \frac{k_{3}\lambda_{C} + k_{4}\lambda_{M}}{\lambda_{\min}(\tilde{\Dm})},
    \end{equation*}
    and the second can be enforced by appropriately choosing $\bm{K}_P$, with a possible choice
    \begin{equation*}
        \lambda_{\min}(\bm{K}_P) > \nu_{1}\lambda_{\max}(\bm{K}_I) + k_{u_{3}} + \frac{Q_{\tthetav_{a}\dthetav}^{2}}{Q_{\dthetav\dthetav}}.
    \end{equation*}

    In the case where $\norm{\tthetav_{a}} > \beta$, $\dV$ can be bounded as
    \begin{equation*}
    \begin{split}
        \dV &\leq - Q_{\dthetav\dthetav}\norm{\dthetav}^{2} + k_{3}Q_{\tthetav_{a}\dthetav}\norm{\dthetav}\\
        &\quad- \beta \left[ k_{2}\lambda_{\min}\left( \bm{K}_P - \nu_{1}\bm{K}_I \right) - k_{3}k_{u_{3}} \right].
    \end{split}
    \end{equation*}

    The quadratic function in $\norm{\dthetav}$ is negative definite if its discriminant is negative, i.e.,
    \begin{equation*}
        k_{3}^{2}Q_{\tthetav_{a}\dthetav}^{2}-4Q_{\dthetav\dthetav}\beta\left[ k_{2}\lambda_{\min}\left( \bm{K}_P - \nu_{1}\bm{K}_I \right) - k_{3}k_{u_{3}} \right] < 0.
    \end{equation*}
    This holds if
    \begin{equation*}
        \lambda_{\min}(\bm{K}_P) > \nu_{1}\lambda_{\max}(\bm{K}_I) + \frac{1}{k_{2}}\left( k_{3}k_{u_{3}} + \frac{k_{3}^{2}Q_{\tthetav_{a}\dthetav}^{2}}{4 \beta Q_{\dthetav\dthetav}} \right).
    \end{equation*}

    Thus, by choosing $\bm{K}_P$ sufficiently large, $\dV$ remains negative everywhere. The conclusion follows again from LaSalle theorem and noting that the set where $\dV = 0$ is
    \begin{equation*}
        \left\{ (\thetav, \dthetav) \mid \thetav_{a} = \thetav_{a_{d}}, \, \dthetav = \zerov_{n} \right\}.
    \end{equation*}

\section{Proof of Theorem 4} 
\label{Sec:Appendix 4}

    Let $\thetav_{u_{d}}$ be the solution to the unactuated statics for the desired $\thetav_{a_{d}}$, namely
    \begin{equation}\label{eq:elastically dominated:PsatID unactuated eq}
        \gv_{u}(\thetav_{a_{d}}, \thetav_{u_{d}}) + \kv_{u}(\thetav_{a_{d}}, \thetav_{u_{d}}) = \zerov_{n-m},
    \end{equation}
    and define $\tthetav = \thetav_{d} - \thetav$ and
    \begin{equation*}\label{eq:elastically dominated:PsatID z}
        \zv = \int_{0}^{t}\sv(\tthetav_{a}(\tau))\drm\tau - \bm{K}_I^{-1}\rb{ \gv(\thetav_{d}) + \kv_{a}(\thetav_{d}) }.
    \end{equation*}
    Using $\zv$, the control action reads
    \begin{equation}\label{eq:elastically dominated:PsatID z form}
        \uv = \bm{K}_P\tthetav_{a} - \bm{K}_D\dthetav_{a} + \bm{K}_I\zv + \gv(\thetav_{d}) + \kv_{a}(\thetav_{d}).
    \end{equation}

    Consider the Lyapunov-like function
    \begin{equation*}
    \begin{split}
        V &= \nu_{1}\left[ \frac{1}{2}\dthetav^{T}\Mm(\thetav)\dthetav + \mathcal{U} + \frac{1}{2}\tthetav_{a}^{T}\bm{K}_P\tthetav_{a} + \tthetav_{a}^{T} (\gv_{a}(\thetav_{d}) \right.\\
        &\left. \vphantom{\frac{1}{2}} + \kv_{a}(\thetav_{d})) + \tthetav_{a}\bm{K}_I\zv + \nu_{2} \right] + \frac{1}{2}\zv^{T}\bm{K}_I\zv - \sv(\tthetav)\Mm(\thetav)\dthetav + \nu_{3},
    \end{split}
    \end{equation*}
    where $\nu_{1}$, $\nu_{2}$, and $\nu_{3}$ are constants to be specified later. As in previous proofs, it is convenient to express $V$ as the sum of three terms: $V = V_{1} + V_{2} + \nu_{1}(\mathcal{U} + V_{3})$, where
    \begin{equation*}
        V_{1} = \frac{1}{2}\nu_{1}\dthetav^{T}\Mm\dthetav - \sv^{T}(\tthetav) \Mm \dthetav + \nu_{3},
    \end{equation*}
    \begin{equation*}
        V_{2} = \nu_{1}\rb{ \frac{1}{4} \tthetav_{a}^{T}\bm{K}_P\tthetav_{a} + \tthetav_{a}^{T}\bm{K}_I\zv } + \frac{1}{2}\zv^{T}\bm{K}_I\zv,
    \end{equation*}
    and
    \begin{equation*}
        V_{3} = \frac{1}{4}\tthetav_{a}^{T}\bm{K}_P\tthetav_{a} + \tthetav_{a}^{T}\rb{ \gv_{a}(\thetav_{d}) + \kv_{a}(\thetav_{d}) } + \nu_{2}.
    \end{equation*}

    The function $V_{1}$ can be lower bounded as
    \begin{equation*}
        V_{1} \geq \frac{1}{2}\nu_{1}\lambda_{m}\norm{\dthetav}^{2} - \lambda_{M}k_{3}\norm{\dthetav} + \nu_{3},
    \end{equation*}
    and is non-negative when $\nu_{3} = \displaystyle \frac{ \lambda_{M}^{2}k_{3}^{2} }{2 \nu_{1}\lambda_{m}}$.

    For $V_{2}$, we have
    \begin{equation*}
        V_{2} = \frac{1}{2}\begin{carray}{c}
            \tthetav_{a}\\ \zv
        \end{carray}^{T}\begin{carray}{cc}
            \frac{1}{2}\nu_{1}\bm{K}_P & \nu_{1}\bm{K}_I\\
            \mathrm{sym} & \bm{K}_I
        \end{carray}\begin{carray}{c}
            \tthetav_{a}\\ \zv
        \end{carray}.
    \end{equation*}
    By applying Schur criterion, this quadratic function is positive definite when $\bm{K}_I > 0$ and $\bm{K}_P > 2 \nu_{1} \bm{K}_I$.

    Finally, for $V_{3}$, we have
    \begin{equation*}
        V_{3} \geq \frac{1}{4}\lambda_{\min}(\bm{K}_P)\norm{\tthetav_{a}}^{2} - \norm{\gv_{a}(\thetav_{d}) + \kv_{a}(\thetav_{d})}\norm{\tthetav_{a}} + \nu_{2}.
    \end{equation*}
    Setting $\nu_{2} = \displaystyle \frac{\norm{ \gv_{a}(\thetav_{d}) + \kv_{a}(\thetav_{d}) }^{2}}{ \lambda_{\min}(\bm{K}_P)}$ ensures that $V_{3}$ is non-negative.

    Since $V$ is the sum of non-negative functions and is radially unbounded (as it is lower bounded by positive quadratic functions of the closed-loop state), it is non-negative.

    Moving on to the time derivative, it holds
    \begin{equation*}
    \begin{split}
        \dV &= \nu_{1} \left[ - \dthetav_{T}\Dm\dthetav + \dthetav_{a}^{T}\uv - \dthetav_{a}^{T}\rb{ \gv(\thetav_{d}) + \kv_{a}(\thetav_{d}) } - \dthetav_{a}^{T}\bm{K}_I\zv\right.\\
        &\left.+ \tthetav_{a}^{T}\bm{K}_I\sv(\tthetav_{a}) \right] + \zv^{T}\bm{K}_I\sv(\tthetav_{a}) + \dthetav^{T}\jac{{}^{T}\sv(\tthetav)}{\tthetav}\Mm \dthetav\\
        &- \sv^{T}(\tthetav)\rb{\Cm^{T}\dthetav + \Dm\dthetav} + \sv^{T}(\tthetav)\rb{ \gv + \kv } - \sv^{T}(\tthetav_{a})\uv.
    \end{split}
    \end{equation*}
    Substituting~\eqref{eq:elastically dominated:PsatID z form} and using the identity~\eqref{eq:elastically dominated:PsatID unactuated eq}, some simplifications yield
    \begin{equation*}
        \begin{split}
            \dV &= \nu_{1}\rb{ - \dthetav^{T}\tilde{\Dm}\dthetav + \tthetav_{a}^{T}\bm{K}_P\tthetav_{a} } + \dthetav^{T}\jac{{}^{T}\sv(\tthetav)}{\tthetav}\Mm \dthetav\\ 
            &- \sv^{T}(\tthetav)\rb{ \Cm^{T} + \tilde{\Dm} \dthetav } + \sv^{T}\rb{ \gv(\thetav) - \gv(\thetav_{d}) + \kv(\thetav) - \kv(\thetav_{d}) }\\
            &- \sv^{T}(\tthetav_{a})\bm{K}_P\tthetav_{a},
        \end{split}
    \end{equation*}
    where
    \begin{equation*}
        \tilde{\Dm} = \begin{carray}{cc}
            \Dm_{aa} + \bm{K}_D & \Dm_{au}\\
            \Dm_{ua} & \Dm_{uu}
        \end{carray}.
    \end{equation*}

    From the Mean Value Theorem for vector-valued functions, one has
    \begin{equation*}
        \gv(\thetav) - \gv(\thetav_{d}) = -\rb{\int_{0}^{1} \jac{\gv(\xv)}{\xv}_{\xv = \thetav - s \tthetav}\drm s}\tthetav
    \end{equation*}
    and
    \begin{equation*}
        \kv(\thetav) - \kv(\thetav_{d}) = -\rb{\int_{0}^{1} \jac{\kv(\xv)}{\xv}_{\xv = \thetav - s \tthetav}\drm s}\tthetav.
    \end{equation*}

\begin{figure*}
    \centering
    \subfigure[Identified step response]{
        \includegraphics[width=0.45\linewidth]{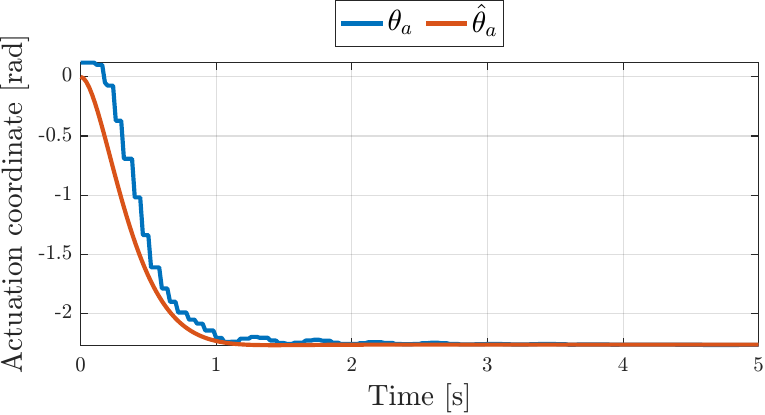}
        \label{identification experiment:step response}
    }
    \subfigure[Estimated closed-loop dynamics]{
        \includegraphics[width=0.45\linewidth]{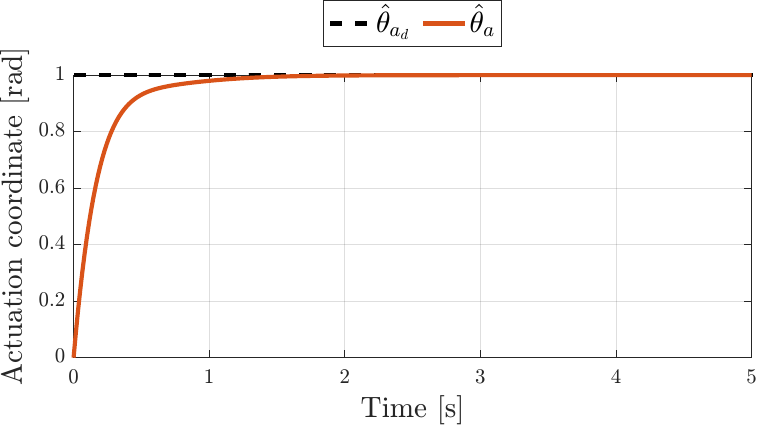}
        \label{identification experiment:closed loop}
    }
    \caption{(a) Measured and identified step response of the open-loop system, and (b) estimated closed-loop response of the linear plant $P(s)$ under a PID regulator with $k_{P} = 0.294~[\si{\newton \meter}]$, $k_{D} = 0.0313~[\si{\newton \meter \second}]$, and $k_{I} = 0.689~[\si{\newton \meter \per \second}]$. }
    \label{identification experiment}
\end{figure*}

    This leads to the following ``quasi'' quadratic form
    \begin{equation*}
    \begin{split}
        &\sv^{T}\rb{ \gv(\thetav) - \gv(\thetav_{d}) + \kv(\thetav) - \kv(\thetav_{d}) }\\
        &\quad\quad - \sv^{T}(\tthetav_{a})\rb{\bm{K}_P - \nu_{1}\bm{K}_I}\tthetav_{a} = - \sv^{T}(\tthetav) \tilde{\Km} \tthetav,
    \end{split}
    \end{equation*}
    where
    \begin{equation*}
        \tilde{\Km} = \small \int_{0}^{1} \begin{carray}{cc}
            \bm{K}_P + \jac{\gv_{a}}{\xv_{a}} + \jac{\kv_{a}}{\xv_{a}} &  \jac{\gv_{a}}{\xv_{u}} + \jac{\kv_{a}}{\xv_{u}}\\
            \jac{\gv_{u}}{\xv_{a}} + \jac{\kv_{u}}{\xv_{a}} & \jac{\gv_{u}}{\xv_{u}} + \jac{\kv_{u}}{\xv_{u}}
        \end{carray}_{\xv = \thetav - s \tthetav} \hspace{-1cm} \drm s.
    \end{equation*}

    A sufficient condition for the above matrix to be positive definite is that its integrand is positive definite. Since the robot is elastically dominated, $\jac{\gv_{u}}{\xv_{u}} + \jac{\kv_{u}}{\xv_{u}} > 0$ always holds, and the integrand is positive definite when
    \begin{equation*}
    \begin{split}
        &\bm{K}_P - \nu_{1}\bm{K}_I + \jac{\gv_{a}}{\xv_{a}} + \jac{\kv_{a}}{\xv_{a}} >\\
        &\quad\rb{\jac{\gv_{a}}{\xv_{u}} + \jac{\kv_{a}}{\xv_{u}}}\rb{\jac{\gv_{u}}{\xv_{u}} + \jac{\kv_{u}}{\xv_{u}}}
        \rb{\jac{\gv_{a}}{\xv_{u}} + \jac{\kv_{a}}{\xv_{u}}}^{T},
    \end{split}
    \end{equation*}
    which can be enforced through $\bm{K}_P$.

    After the above considerations, $\dV$ simplifies to
    \begin{equation*}
        \begin{split}
            \dV &= \nu_{1}\rb{ - \dthetav^{T}\tilde{\Dm}\dthetav + \tthetav_{a}^{T}\bm{K}_P\tthetav_{a} } + \dthetav^{T}\jac{{}^{T}\sv(\tthetav)}{\tthetav}\Mm \dthetav\\ 
            &- \sv^{T}(\tthetav)\rb{ \Cm^{T} + \tilde{\Dm} \dthetav } + \sv^{T}\tilde{\Km}\tthetav.
        \end{split} 
    \end{equation*}

    By exploiting the properties of the saturation function and the dynamic model, $\dV$ can be upper bounded as
    \begin{equation*}
    \begin{split}
    \dV &\leq - \left(\nu_{1} \lambda_{\min}(\tilde{\Dm}) - (k_{4}\lambda_{M} + k_{3}\lambda_{C})\right) \norm{\dthetav}^{2}\\
    &+ \lambda_{\max}(\tilde{\Dm}) \norm{\sv(\tthetav)}\norm{\dthetav} - \lambda_{\min}(\tilde{\Km})\sv^{T}(\tthetav)\tthetav.
    \end{split}
    \end{equation*}

    For $\norm{\tthetav} \leq \beta$, we have
    \begin{equation*}
    \begin{split}
        \dot{V} \leq - \begin{carray}{c}
            \norm{\tthetav}\\
            \norm{\dthetav}
        \end{carray}^{T}\left( \begin{array}{cc}
            k_{1}Q_{\tthetav \tthetav} & Q_{\tthetav \dthetav} \\[2pt]
            \mathrm{sym} & Q_{\dthetav \dthetav}
        \end{array} \right)\begin{carray}{c}
            \norm{\tthetav}\\
            \norm{\dthetav}
        \end{carray},
    \end{split}
    \end{equation*}
    where
    \begin{gather*}\small
        Q_{\tthetav\tthetav} = \lambda_{\min}(\tilde{\Km})\\
        Q_{\tthetav \dthetav} = -\frac{1}{2} \lambda_{\max}(\tilde{\Dm})\\
        Q_{\dthetav \dthetav} = \gamma \lambda_{\min}(\tilde{\Dm}) - (k_{4}\lambda_{M} + k_{3}\lambda_{C}).
    \end{gather*}

    According to Sylvester criterion, $\Qm > 0$ if and only if
    \begin{align}
        \label{eq:dominant_elasticity:PID:lamda_min_K_positive}
        \lambda_{\min}(\tilde{\Km}) &> 0,\\
        \label{eq:dominant_elasticity:PID:det_Q_positive}
        \det{\Qm} &> 0.
    \end{align}

    Eq.~\eqref{eq:dominant_elasticity:PID:lamda_min_K_positive} holds when $\bm{K}_P$ is sufficiently large, while~\eqref{eq:dominant_elasticity:PID:det_Q_positive} is guaranteed if
    \begin{equation*}
        \nu_{1} > \frac{1}{\lambda_{\min}(\tilde{\Dm})}  \left( \frac{1}{k_{1}} \frac{\lambda^{2}_{\max}(\tilde{\Dm})}{4 \lambda_{\min}(\tilde{\Km})} + k_{4}\lambda_{M} + k_{3}\lambda_{C} \right).
    \end{equation*}

    When $\norm{\tthetav} > \beta$, $\dot{V}$ is bounded by
    \begin{equation*}
        \dot{V} < -Q_{\dthetav \dthetav}\norm{\dthetav}^{2} + k_{3}Q_{\tthetav\dthetav} \norm{\dthetav} - k_{2}\beta Q_{\tthetav\tthetav},
    \end{equation*}
    and is negative if
    \begin{equation*}
        Q_{\dthetav\dthetav} > \frac{k_{3}^{2}}{k_{2}\beta} \frac{Q_{\tthetav\dthetav}^{2}}{Q_{\tthetav\tthetav}},
    \end{equation*}
    or equivalently
    \begin{equation*}
        \nu_{1} > \frac{1}{\lambda_{\min}(\tilde{\Dm})}  \left( \frac{k_{3}^{2}}{k_{2}\beta} \frac{\lambda^{2}_{\max}(\tilde{\Dm})}{4 \lambda_{\min}(\tilde{\Km})} + k_{4}\lambda_{M} + k_{3}\lambda_{C} \right).
    \end{equation*}

    By choosing $\nu_{1}$ large enough, $\dV$ is negative definite. The conclusion follows once again from LaSalle Principle.

\section{Block scheme of the closed-loop system}\label{sec:block scheme}

The block scheme of the closed-loop system is illustrated in Fig.~\ref{experiments:closed-loop system}.

\begin{figure}[h]
    \centering
    \includegraphics[width=1\columnwidth]{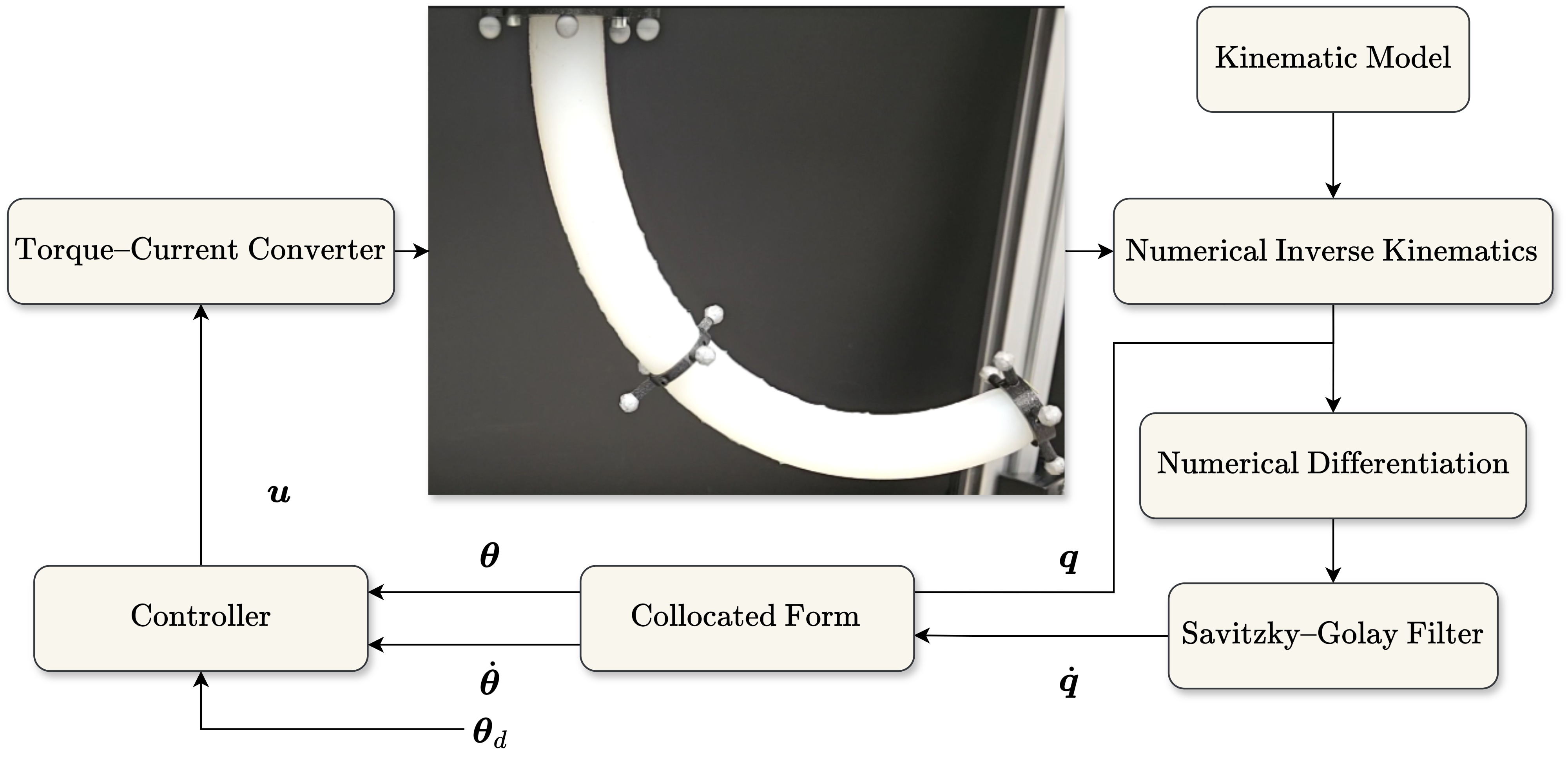}
    \caption{Block scheme of the closed-loop system.}
    \label{experiments:closed-loop system}
\end{figure}

\section{Identification Experiment for Gain Tuning of the controllers}\label{sec:identification experiment}

    Indicative values for the gains are obtained by performing a preliminary identification experiment, where the robot dynamics is approximated as a second-order linear time-invariant system with input $u$ and output $\theta_{a}$, which can be directly measured using the motion capture system. It is worth remarking that we use the direct measure of $\theta_{a}$ only for identification. All other experiments use the estimate of $\bm{\theta}$ obtained from the curvature variables $\qv$. The linear plant is represented in the Laplace domain as
\begin{equation*}
    P(s) = \frac{\theta_{a}}{u} = \frac{A_{P}}{\left(T_{w}s\right)^{2} + 2T_{w} \zeta s + 1}.
\end{equation*}
and identified by applying a constant step input of $u = -0.3~[\si{\newton \meter}]$, resulting in the following parameters: $A_{P} = 7.5432$, $T_{w} = 0.20737$, and $\zeta = 0.88971$. Figure ~\ref{identification experiment:step response} shows the measured step response and the identified response for the same input. 

Using $P(s)$, we tune the gains of a PID controller such that the dynamics achieves a crossover frequency of $6~[\si{\radian \per \second}]$ and a phase margin of $\frac{\pi}{2}~[\si{\radian}]$. This results in the estimated closed-loop response of Fig.~\ref{identification experiment:closed loop}, with control gains $k_{P} = 0.294~[\si{\newton \meter}]$, $k_{D} = 0.0313~[\si{\newton \meter \second}]$, and $k_{I} = 0.689~[\si{\newton \meter \per \second}]$. Based on these initial estimates, for the experiments, we choose $k_{D} = 0.039~[\si{\newton \meter \second}]$ and $k_{I} = 0.689~[\si{\newton \meter \per \second}]$, while $k_{P}$ takes ten possible values from the set $\{0,\,\, 0.1,\,\, \dots,\,\, 0.9 \}~[\si{\newton \meter}]$.

\section{Comprehensive Results of the Experiments}\label{sec:Appendix_Results}

This appendix provides the complete experimental results, complementing the main findings presented in the previous sections. Figures \ref{PD:comparison:error metrics}-\ref{discretiations:strobo_plots} offer a detailed view of the obtained results across the different experimental settings, allowing for a more comprehensive comparison and analysis.


\begin{figure*}
    \centering
    \includegraphics[width=1\textwidth,  trim={0cm, 0cm, 0.0cm, 0}, clip]{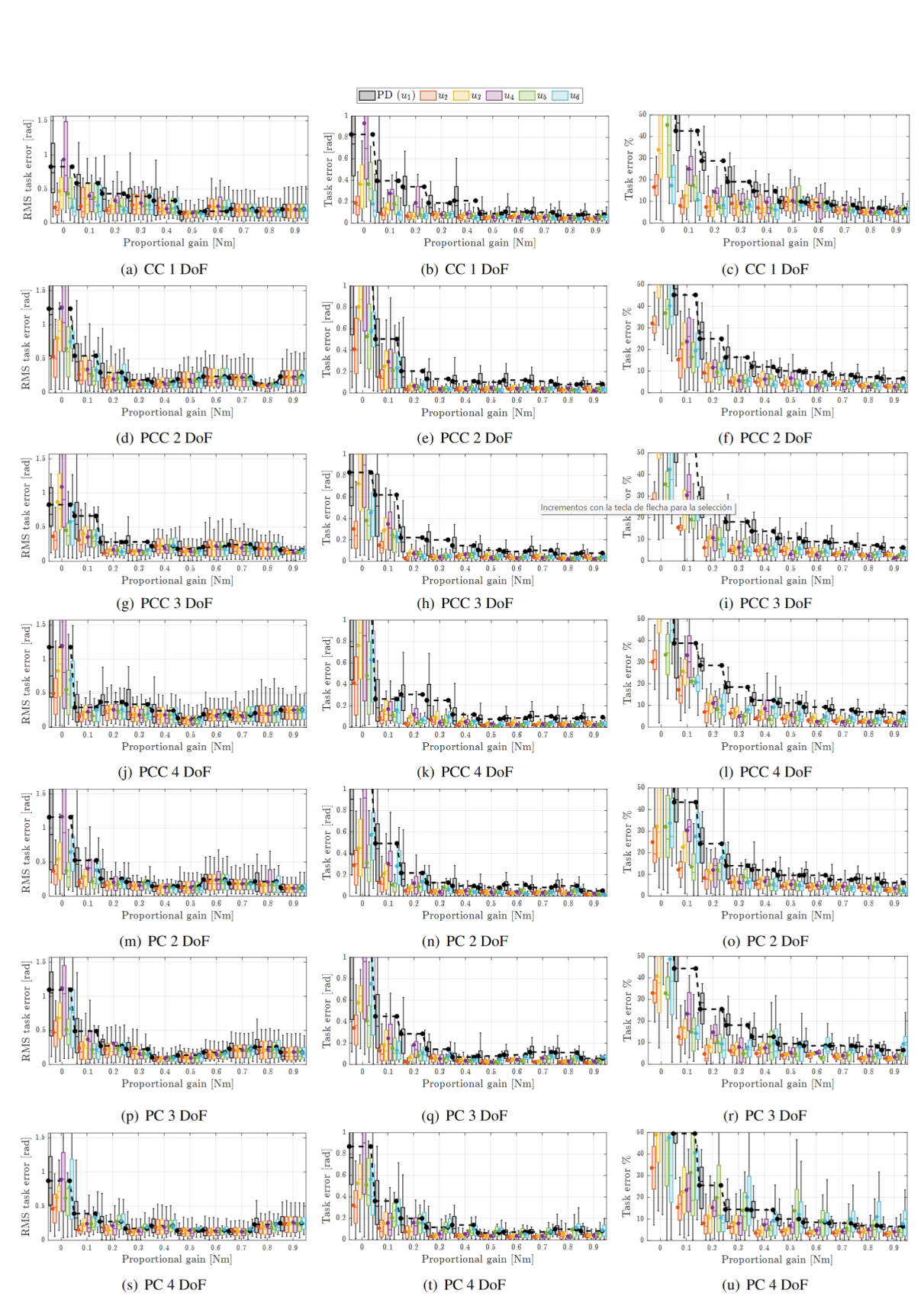}
    \caption{RMSE, absolute value of the steady state error, and its percentage for the PD regulators. The solid lines represent the mean value for each metric of the PD regulator. For each controller, gain, and metric, the plots show the median, lower and upper quartiles, as well as the minimum and maximum values, and with a colored dot the mean value.}
    \label{PD:comparison:error metrics}
\end{figure*}

\begin{figure*}
    \centering
     \includegraphics[width=1\textwidth,  trim={0cm, 0cm, 0.0cm, 0}, clip]{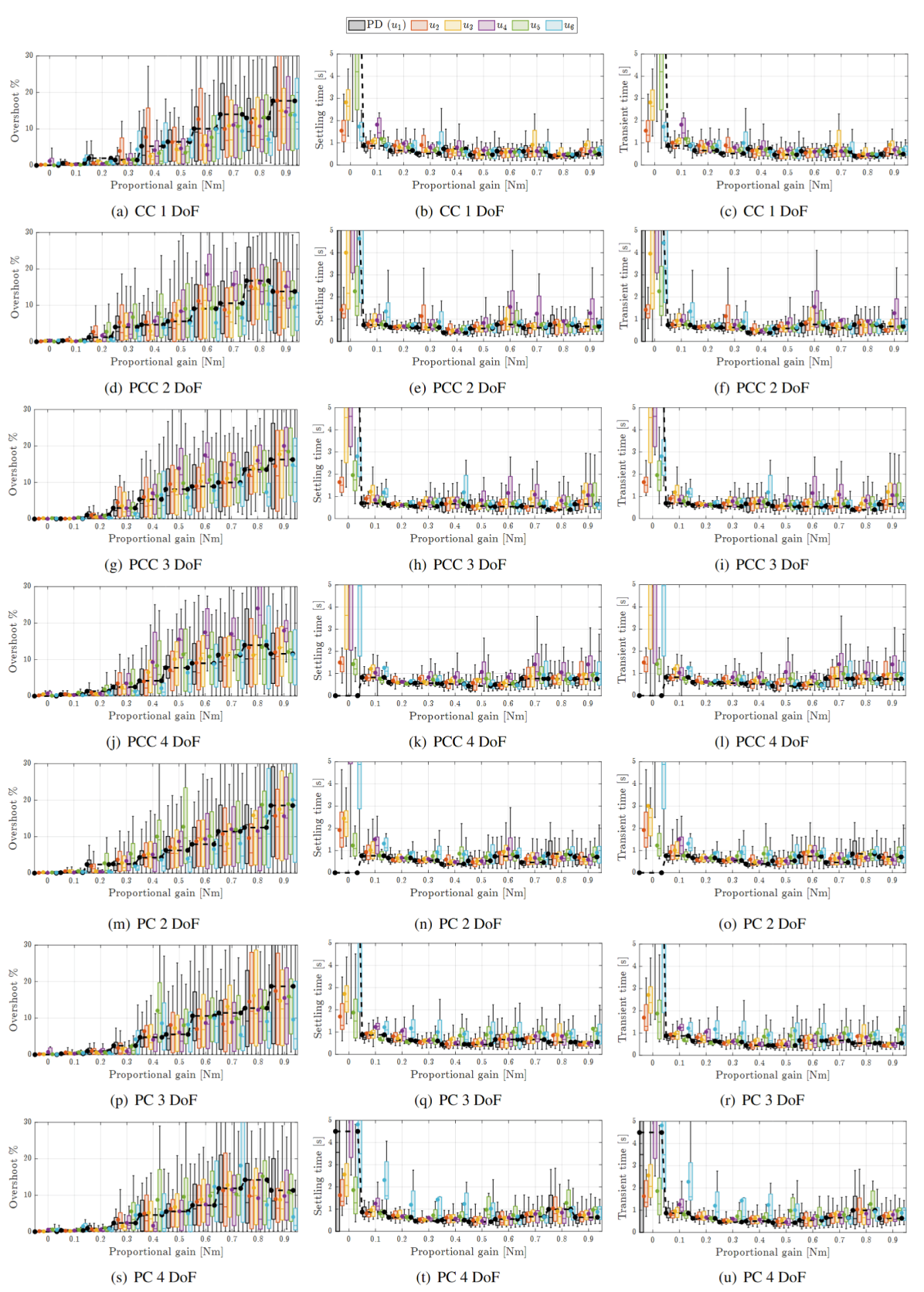}
    \caption{Overshoot, settling time and transient time for the PD regulators. The solid lines represent the mean value for each metric of the PD regulator. For each controller, gain, and metric, the plots show the median, lower and upper quartiles, as well as the minimum and maximum values, and with a colored dot the mean value.}
    \label{PD:comparison:time metrics}
\end{figure*}

\begin{figure*}
    \centering
    \includegraphics[width=1\textwidth,  trim={0cm, 0cm, 0.0cm, 0}, clip]{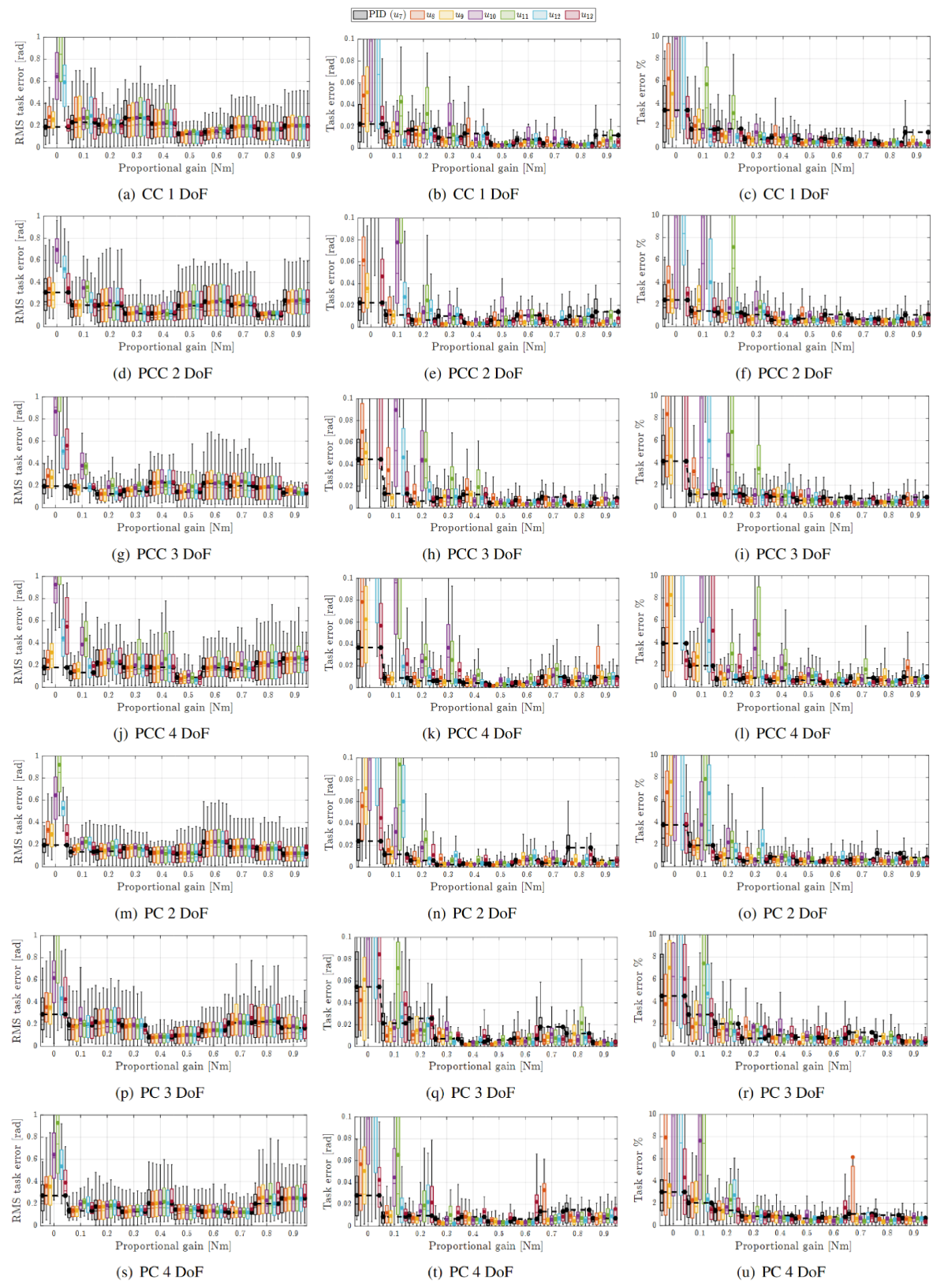}
    \caption{RMSE, absolute value of the steady state error and its percentage for the PID regulators. The solid lines represent the mean value for each metric of the PID regulator. For each controller, gain, and metric, the plots show the median, lower and upper quartiles, as well as the minimum and maximum values, and with a colored dot the mean value. }
    \label{PID:comparison:error metrics}
\end{figure*}
\begin{figure*}
    \centering
    \includegraphics[width=1\textwidth,  trim={0cm, 0cm, 0.0cm, 0}, clip]{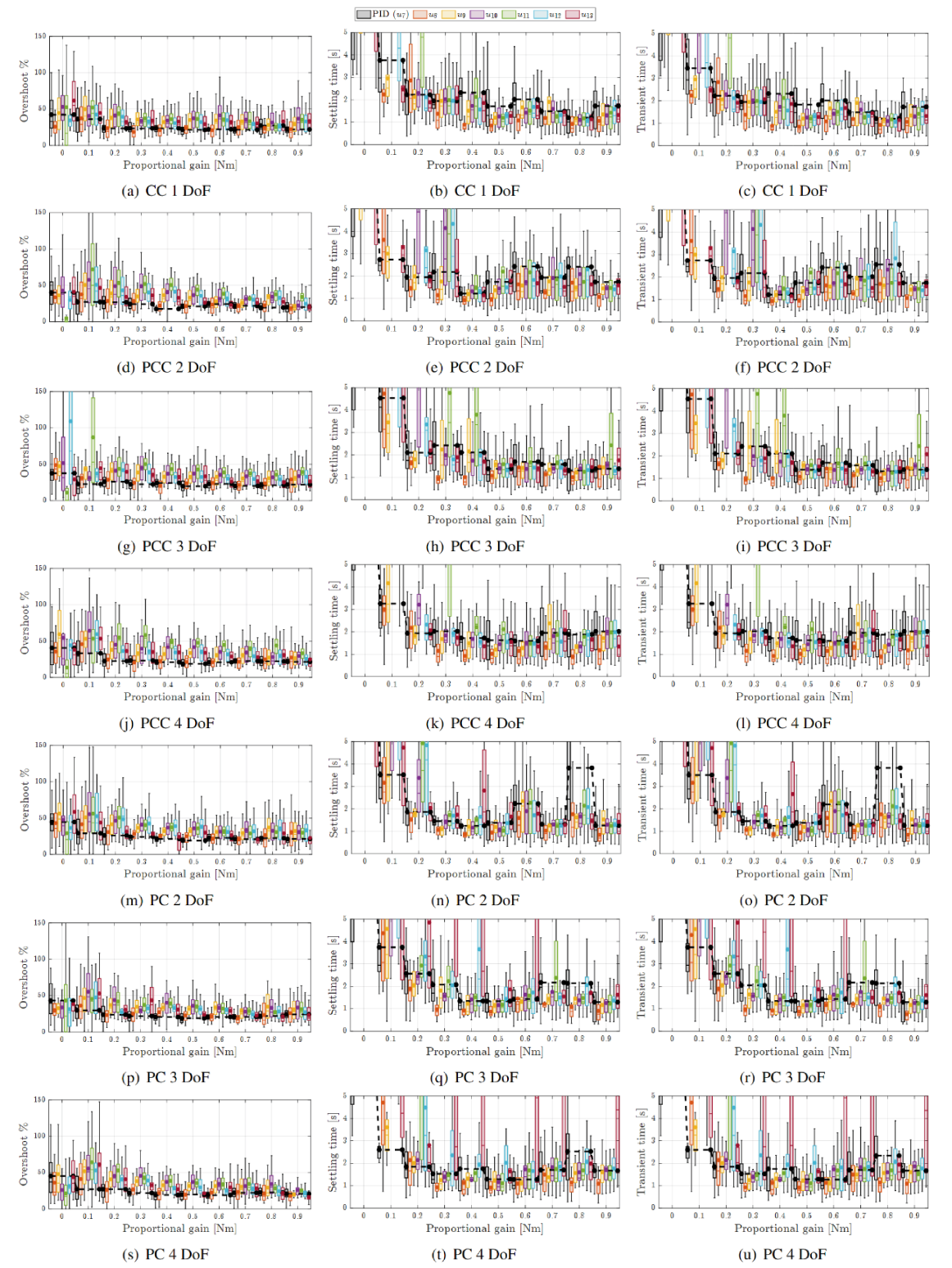}
    \caption{Overshoot, settling time and transient time for the PID regulators. The solid lines represent the mean value for each metric of the PID regulator. For each controller, gain, and metric, the plots show the median, lower and upper quartiles, as well as the minimum and maximum values, and with a colored dot the mean value.}
    \label{PID:comparison:time metrics}
\end{figure*}

\begin{figure*}
    \centering
    \includegraphics[width=1\textwidth,  trim={0cm, 0cm, 0.0cm, 0}, clip]{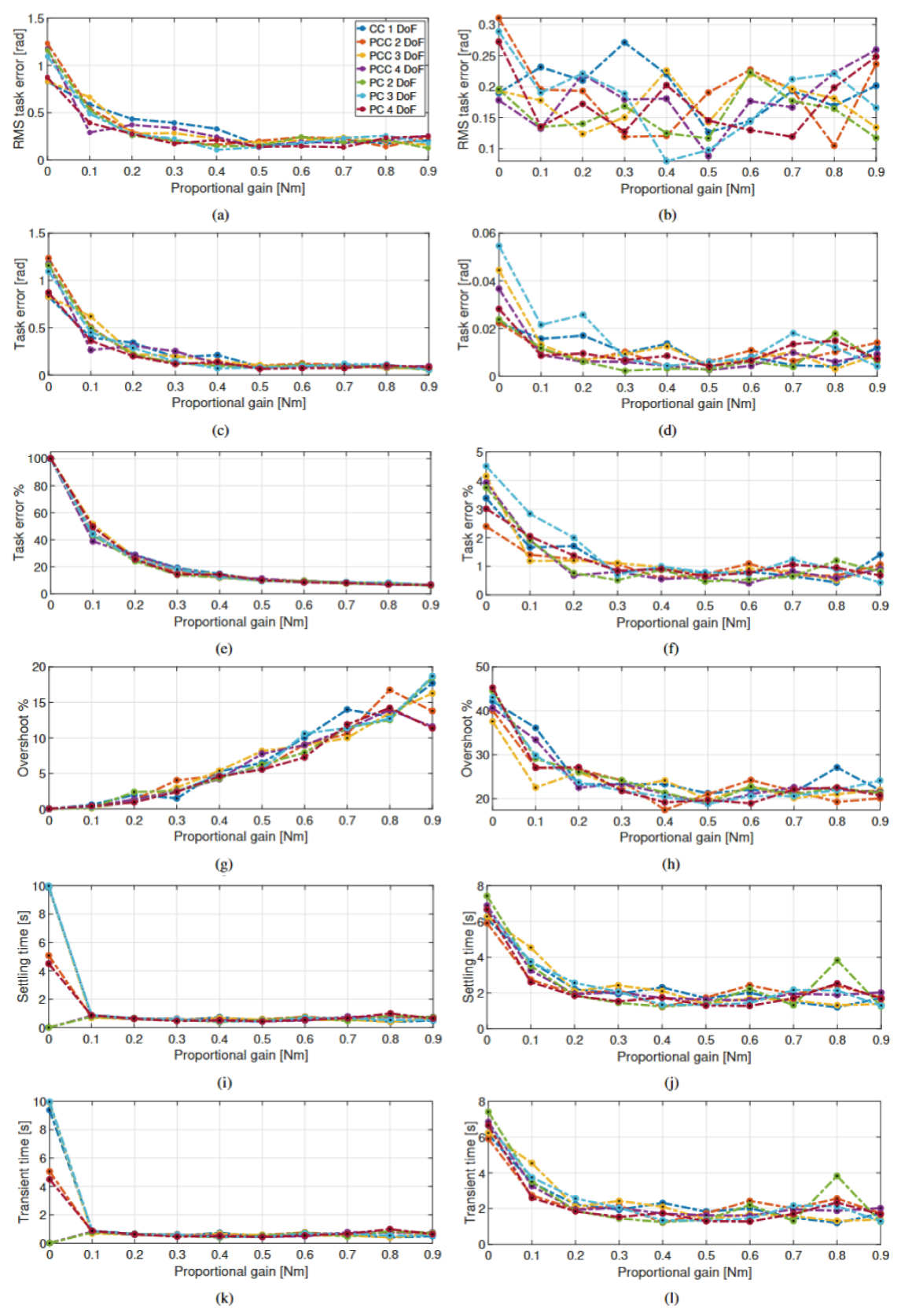}
    \caption{Mean RMSE, absolute and percentage steady-state errors, overshoot, settling time, and transient time for the PD regulator (left side of the figure) and the PID regulator (right side of the figure). The lines represent the metric values averaged across all model discretizations.}
    \label{PD:comparison:time metrics_Final}
\end{figure*}

\begin{figure*}
    \centering
    \includegraphics[width=1\textwidth,  trim={0cm, 0cm, 0.0cm, 0}, clip]{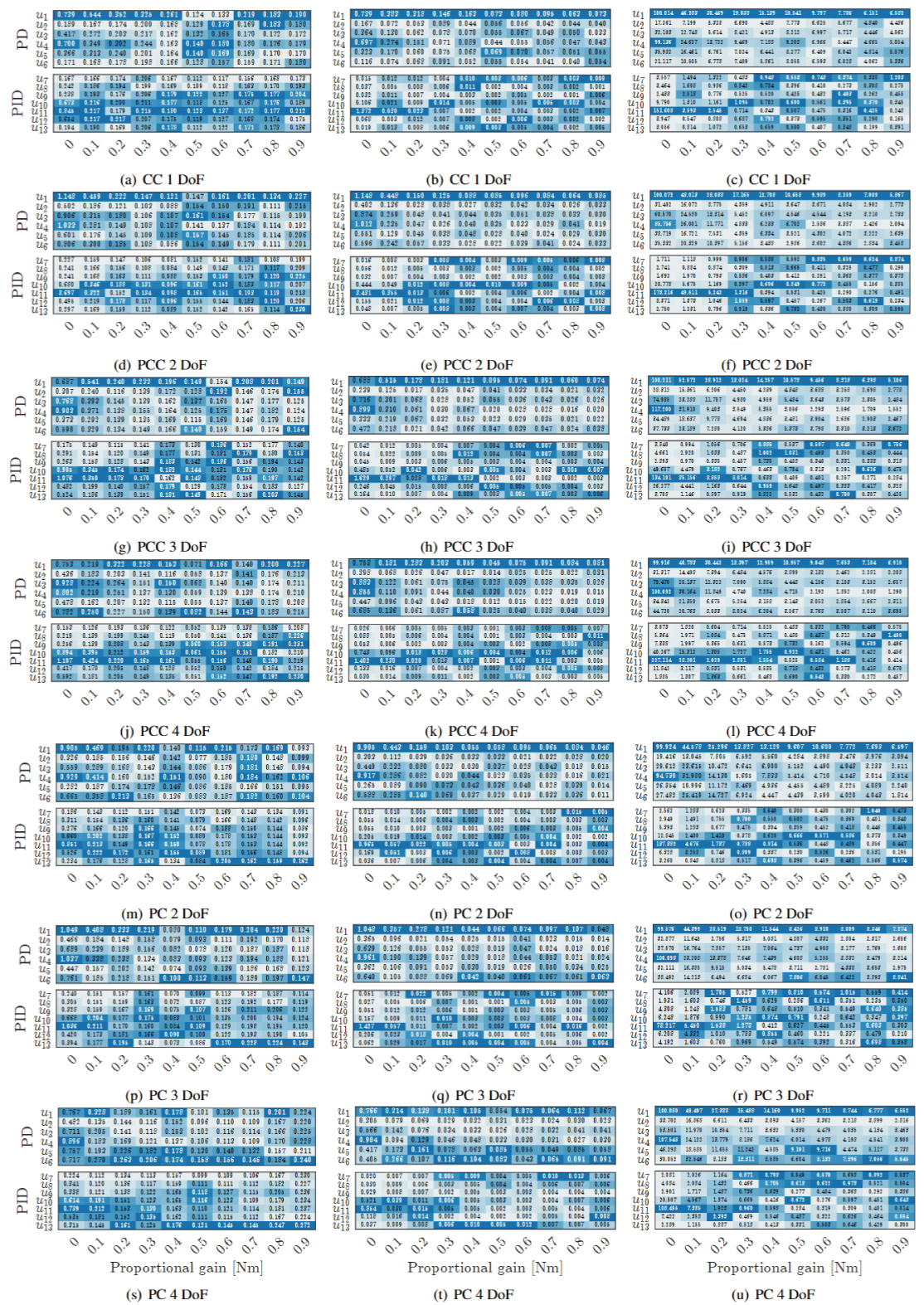}
    \caption{Mean of the RMSE [left], absolute value of the steady state error [center] and its percentage [right]
for the analyzed regulators with varying proportional gain (in [Nm]).
 }
    \label{experiment:overall controllers comparison error}
\end{figure*}
\begin{figure*}
    \centering
    \includegraphics[width=1\textwidth,  trim={0cm, 0cm, 0.0cm, 0}, clip]{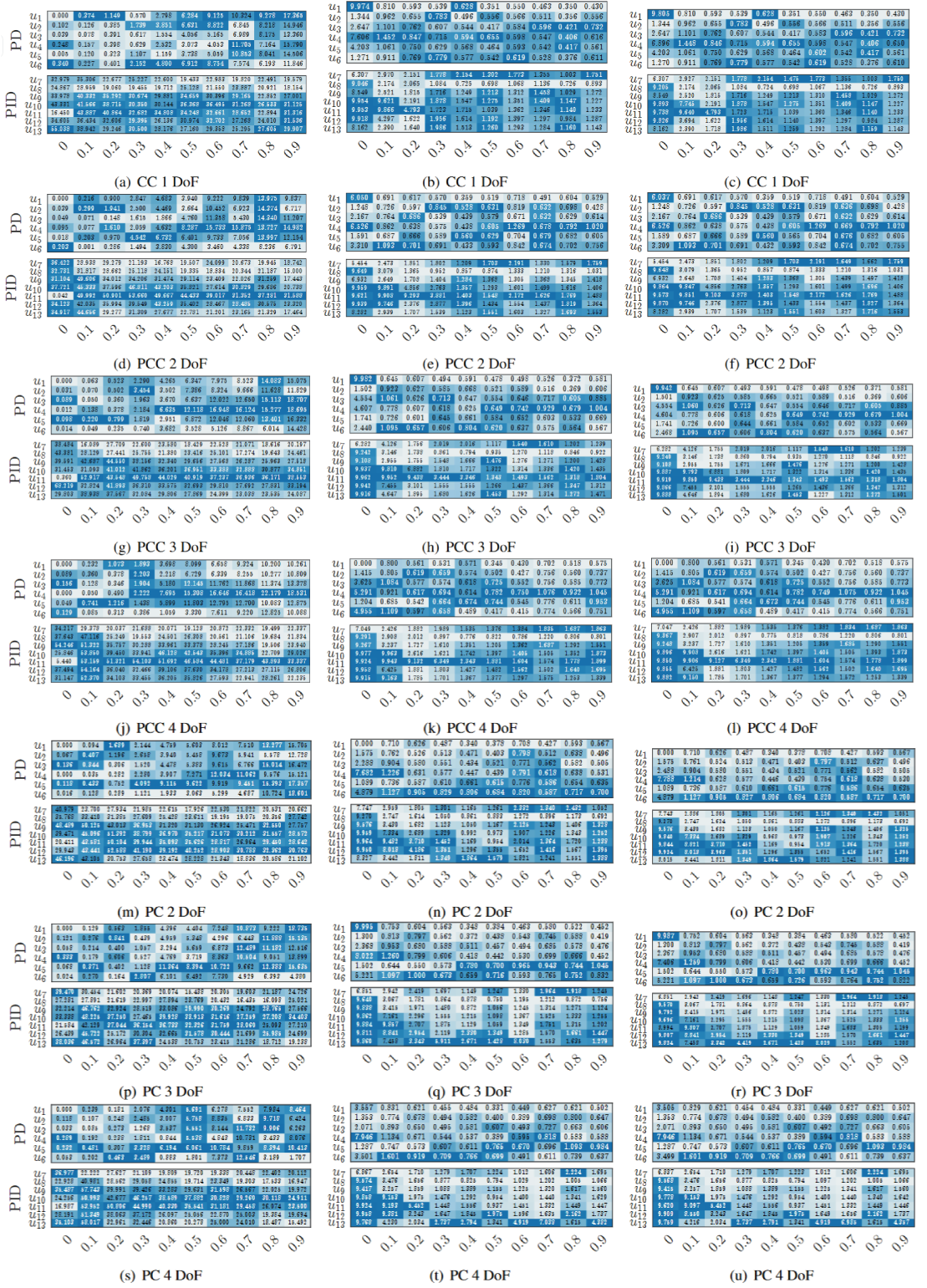}
    \caption{Mean value of the overshoot [left], settling time [center] and transient time [right]
for the analyzed regulators with varying proportional gain (in [Nm]). }
    \label{experiment:overall controllers comparison time}
\end{figure*}

\begin{figure*}
    \centering
    \includegraphics[width=1\textwidth,  trim={0cm, 0cm, 0.0cm, 0}, clip]{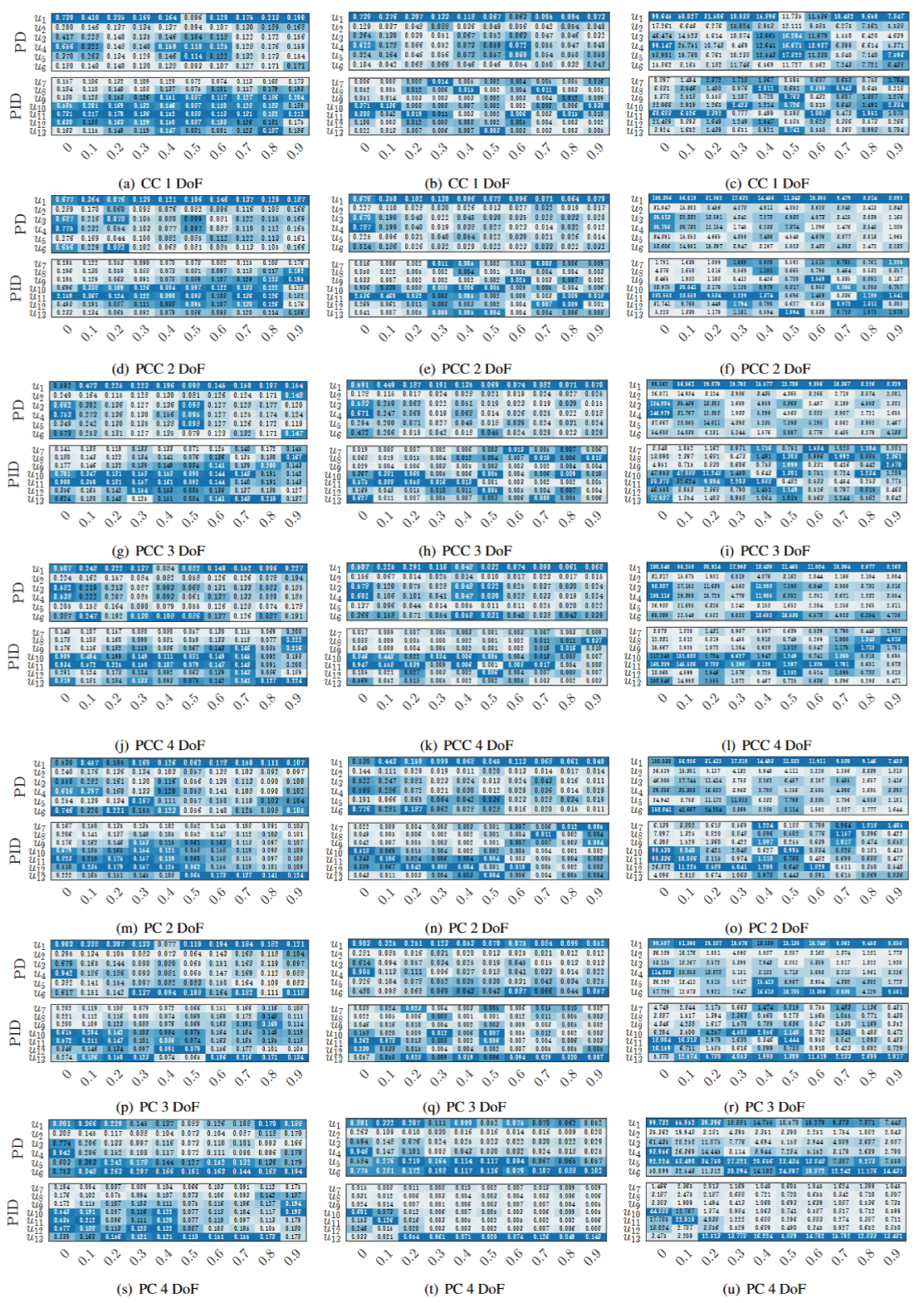}
    \caption{Mean of the RMSE [left], absolute value of the steady state error [center] and its percentage [right]
respectively) [center] and its percentage [right] with varying proportional gain (in [Nm]) in the presence of a payload of $0.09~[\si{\kilogram}]$ attached to the robot tip.}
    \label{experiment:overall controllers comparison error payload}
\end{figure*}
\begin{figure*}
    \centering
    \includegraphics[width=1\textwidth,  trim={0cm, 0cm, 0.0cm, 0}, clip]{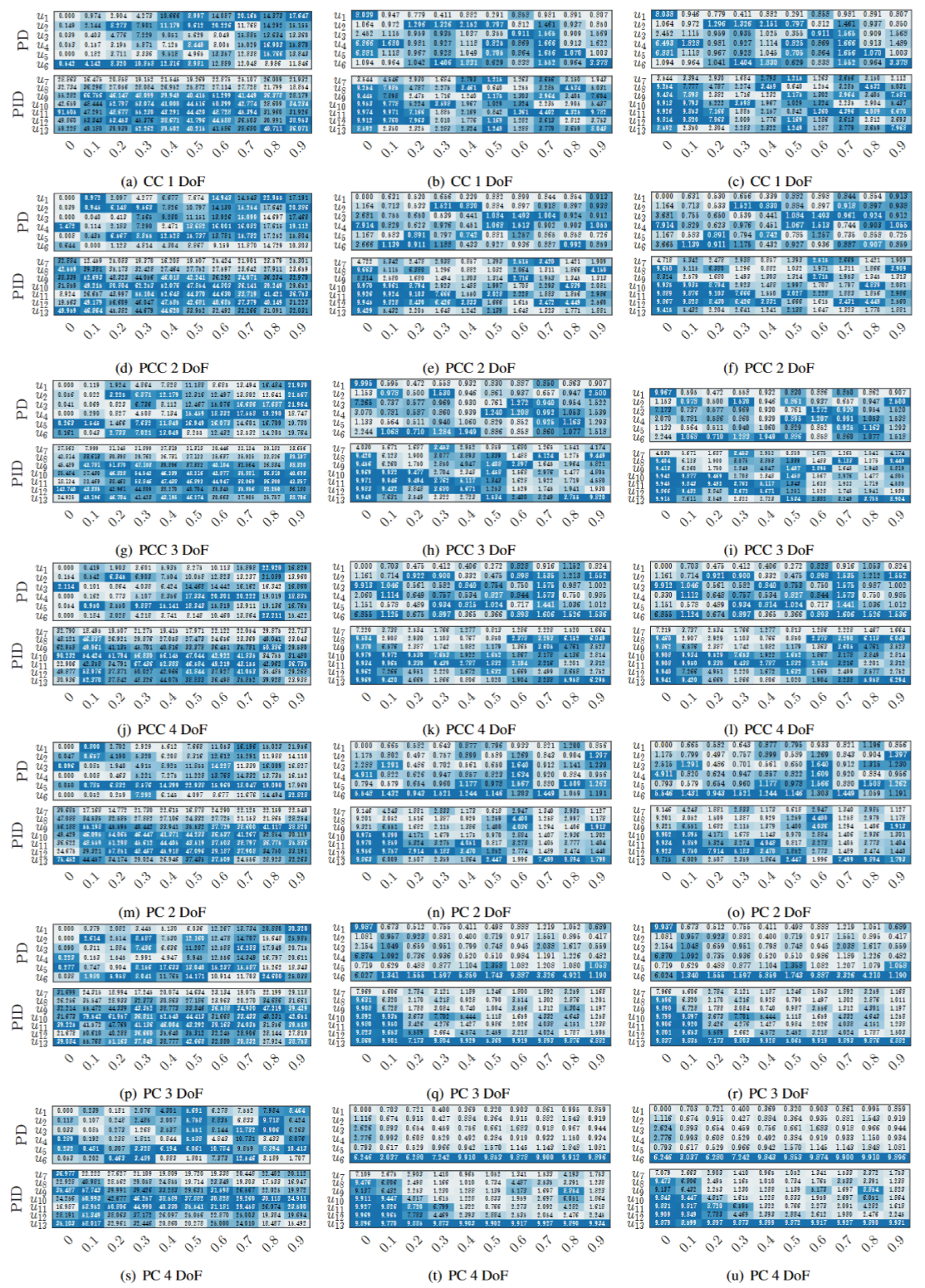}
    \caption{Mean value of the overshoot [left], settling time [center] and transient time [right] for the analyzed regulators with varying proportional gain (in [Nm]) in the presence of a payload of $0.09~[\si{\kilogram}]$ attached to the robot tip.}
    \label{experiment:overall controllers comparison time payload}
\end{figure*}

\begin{figure*}[t]
    \centering
    \subfigure[Steady state Cartesian error]{
        \includegraphics[width=0.3\linewidth]{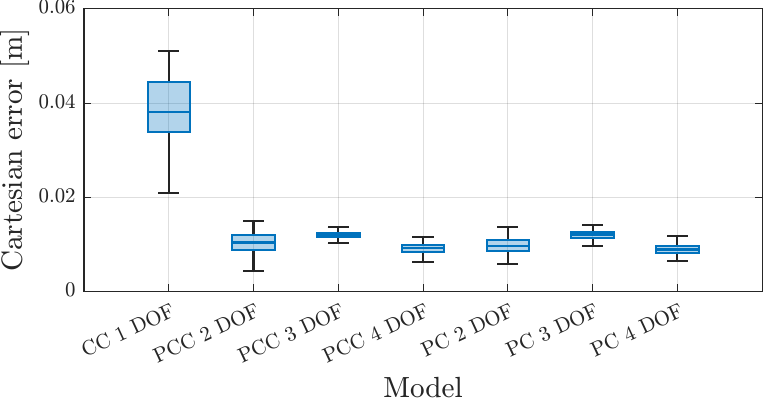}
    }
    \subfigure[Steady state angular error]{
        \includegraphics[width=0.3\linewidth]{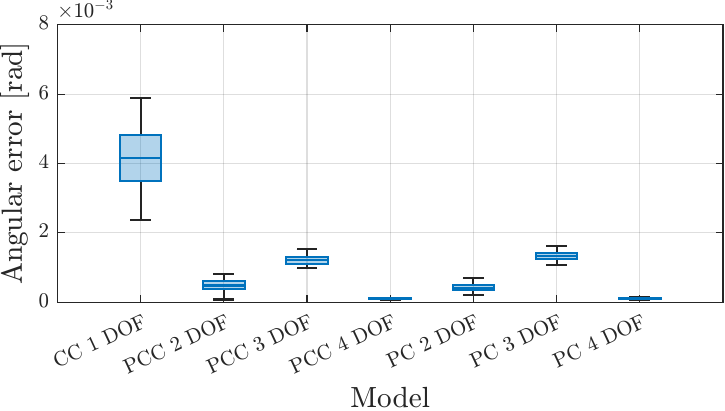}
    }
    \subfigure[Steady state task error]{
        \includegraphics[width=0.3\linewidth]{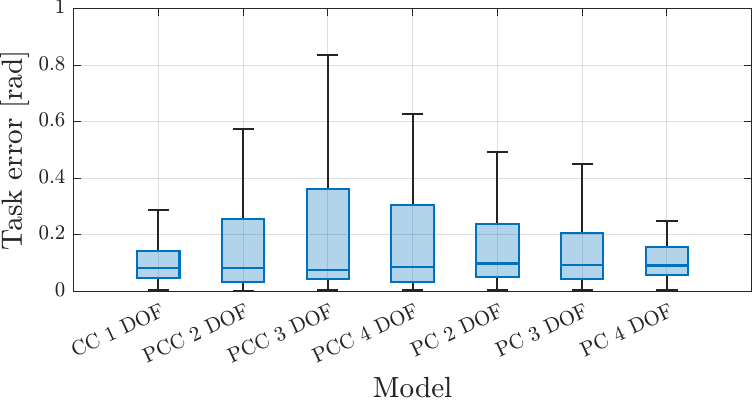}
    }
    \subfigure[RMS Cartesian error]{
        \includegraphics[width=0.3\linewidth]{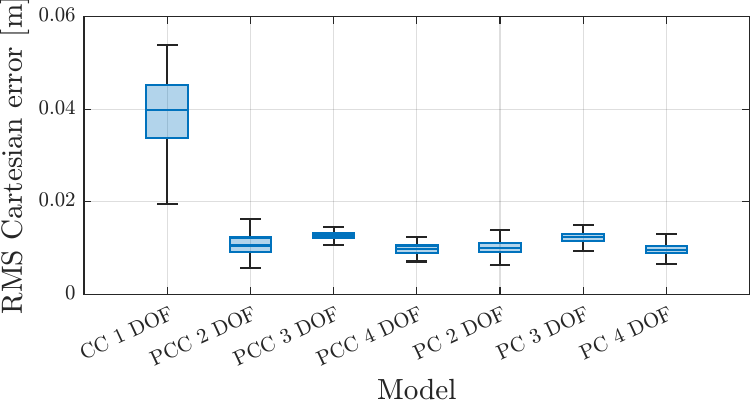}
    }
    \subfigure[RMS angular error]{
        \includegraphics[width=0.3\linewidth]{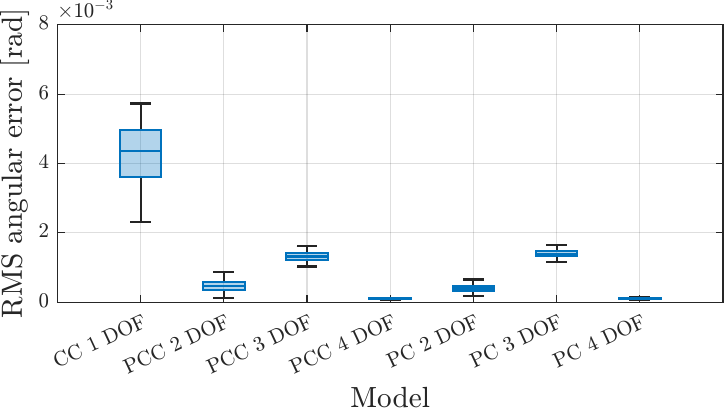}
    }
    \subfigure[RMS task error]{
        \includegraphics[width=0.3\linewidth]{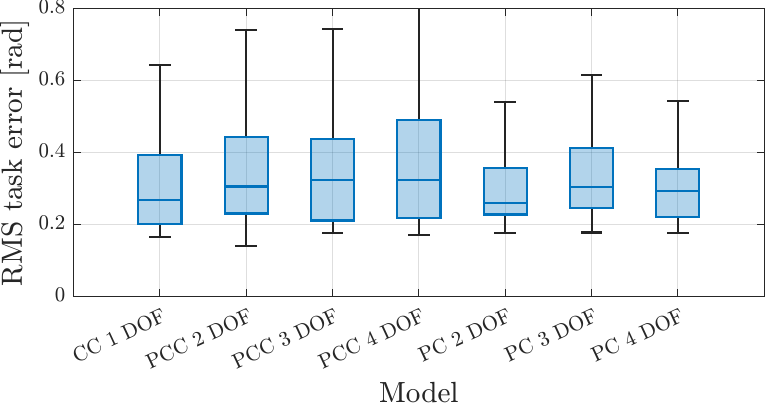}
    }
    \caption{Steady state and RMS values of the Cartesian, angular, and task errors for the considered discretizations.}
    \label{model comparison:metrics}
\end{figure*}

\begin{figure*}[t]
    \centering
    \includegraphics[width=1\textwidth,  trim={0cm, 0cm, 0.0cm, 0}, clip]{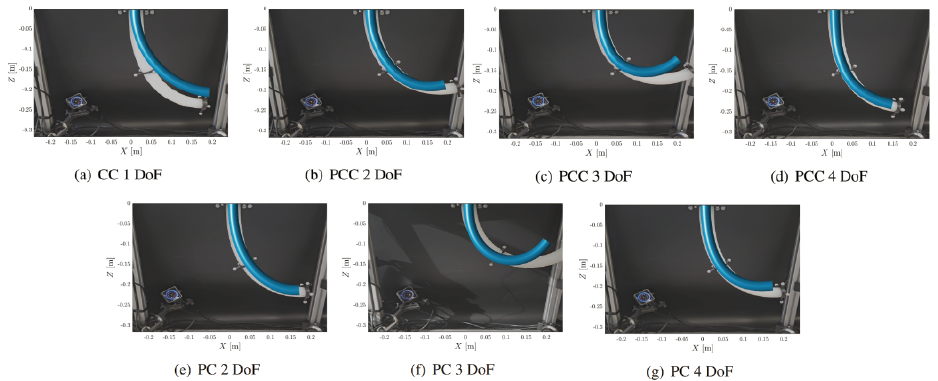}
    \caption{Snapshots of the robot and corresponding dynamic model for the closed-loop system under the P-satI-D controller with $k_{P} = 0.5~[\si{\newton \meter}]$.}
    \label{discretiations:strobo_plots}
\end{figure*}


\section*{Acknowledgements}
The work was supported by the European Union (ERC, RIPLEY, 101165078).

\clearpage
\bibliographystyle{agsm}
\bibliography{bibliography}

\end{document}